\documentclass{article}

\PassOptionsToPackage{nonamebreak}{natbib}
\usepackage[preprint]{neurips_2025}

\usepackage[utf8]{inputenc}
\usepackage[T1]{fontenc}
\usepackage[breaklinks=true]{hyperref}
\usepackage{url}
\usepackage{booktabs}
\usepackage{amsfonts}
\usepackage{mathrsfs}
\usepackage{placeins}
\usepackage{nicefrac}
\usepackage{microtype}
\usepackage[table]{xcolor}
\usepackage{graphicx}
\usepackage{subcaption}
\usepackage{tikz}
\usetikzlibrary{positioning,arrows.meta,calc,backgrounds}
\usepackage{adjustbox}
\usepackage{algorithm}
\usepackage[noend]{algpseudocode}
\usepackage{tabularx}
\usepackage{array,longtable}
\usepackage{amsmath, amssymb, amsthm, mathtools}
\usepackage{aliascnt}
\usepackage{bm}
\usepackage{titlesec}

\newcommand{\safeincludegraphics}[2][]{\includegraphics[#1]{#2}}

\titlespacing*{\section}{0pt}{2.5ex plus 1ex minus .2ex}{1.5ex plus .2ex}
\titlespacing*{\subsection}{0pt}{2ex plus .8ex minus .2ex}{1ex plus .2ex}

\theoremstyle{plain}
\newtheorem{theorem}{Theorem}[section]

\newaliascnt{lemma}{theorem}
\newtheorem{lemma}[lemma]{Lemma}
\aliascntresetthe{lemma}

\newaliascnt{proposition}{theorem}
\newtheorem{proposition}[proposition]{Proposition}
\aliascntresetthe{proposition}

\newaliascnt{corollary}{theorem}
\newtheorem{corollary}[corollary]{Corollary}
\aliascntresetthe{corollary}

\theoremstyle{definition}
\newaliascnt{definition}{theorem}
\newtheorem{definition}[definition]{Definition}
\aliascntresetthe{definition}

\newaliascnt{assumption}{theorem}

\aliascntresetthe{assumption}

\theoremstyle{remark}
\newaliascnt{remark}{theorem}
\newtheorem{remark}[remark]{Remark}
\aliascntresetthe{remark}

\newcommand{\op}{\operatorname{op}}
\newcommand{\tr}{\operatorname{tr}}
\newcommand{\range}{\operatorname{range}}
\newcommand{\rank}{\operatorname{rank}}
\newcommand{\diag}{\operatorname{diag}}
\newcommand{\psinv}[1]{#1^{\dagger/2}}

\newcommand{\Kss}{\bm{K}_{\!SS}}
\newcommand{\Kqs}{\bm{K}_{\!QS}}
\newcommand{\Ksq}{\bm{K}_{\!SQ}}
\newcommand{\Rsg}{R_{\!S,\gamma}}

\title{A Theory of Generalization in Deep Learning}

\author{%
  Elon Litman\thanks{Corresponding author: \texttt{elonlit@stanford.edu}} \And Gabe Guo
}

\begin{document}

\maketitle
\vspace{-1.2em}
\begin{center}
  Stanford University
\end{center}
\vspace{1.2em}

\begin{abstract}
We present a non-asymptotic theory of generalization in deep learning where the empirical neural tangent kernel partitions the output space. In directions corresponding to signal, error dissipates rapidly; in the vast orthogonal dimensions corresponding to noise, the kernel's near-zero eigenvalues trap residual error in a test-invisible reservoir. Within the signal channel, minibatch SGD ensures that coherent population signal accumulates via fast linear drift, while idiosyncratic memorization is suppressed into a slow, diffusive random walk. We prove generalization survives even when the kernel evolves $\mathcal{O}(1)$ in operator norm, the full feature-learning regime. This theory naturally explains disparate phenomena in deep learning theory, such as benign overfitting, double descent, implicit bias, and grokking. Lastly, we derive an exact population-risk objective from a single training run with no validation data, for any architecture, loss, or optimizer, and prove that it measures precisely the noise in the signal channel. This objective reduces in practice to an SNR preconditioner on top of Adam, adding one state vector at no extra cost; it accelerates grokking by $5 \times$, suppresses memorization in PINNs and implicit neural representations, and improves DPO fine-tuning under noisy preferences while staying $3 \times$ closer to the reference policy.
\end{abstract}

\section{Introduction}

A neural network with more parameters than training examples can memorize arbitrary labels, including pure noise \citep{zhang2017understanding}, yet the same training procedure usually generalizes on real data. Classical capacity bounds are vacuous at practical scale, and frozen-kernel theory \citep{jacot2018neural} describes the lazy regime, whereas modern architectures train in the full feature-learning regime. We develop a theory of generalization that handles full feature learning, and we derive a practical method that trains directly on population risk.

We work in output space along the realized trajectory. The empirical tangent kernel $\Kss=\bm J_S\bm J_S^\top$ selects which output directions training can move, and integrating its evolution gives the cumulative dissipation $\mathcal W_S$, whose range is the \emph{signal channel} and whose kernel is the \emph{reservoir}. The test-train kernel $\Kqs=\bm J_Q\bm J_S^\top$ shares the factor $\bm J_S^\top$, so every reservoir direction is invisible to every test set; SGD attenuates surviving label noise because its centered minibatch fluctuation diffuses while drift accumulates; and on the signal-channel side the training and test displacements both factor through $\mathcal W_S^{1/2}$, so under squared loss test motion is determined exactly by training motion along the realized path, even when the kernel drifts by $\mathcal{O}(1)$ in operator norm. \autoref{fig:four_cells} summarizes the resulting decomposition of test error and locates the classical phenomena: grokking, double descent, implicit bias, and benign overfitting, within it.

Exchangeability turns the same operators into population risk. The test transfer $\mathsf G$ instantiated with each training point as a one-point test set against the remaining batch is an unbiased rate of population-risk decrease, and on a one-step window it pulls back through the per-example Jacobians to $\operatorname{tr}(M\bm A_B)$ with $\bm A_B=\bar{\bm g}_B\bar{\bm g}_B^\top-\tfrac{1}{b-1}\bm\Sigma_B$. Maximizing this through the optimizer's metric updates parameter $k$ only when $\mu_k^2$ exceeds $\sigma_k^2/(b-1)$, a one-line change to Adam with one extra state vector.

\begin{figure}[!t]
\centering
\begin{adjustbox}{center}
\definecolor{cellgood}{RGB}{232,238,246}
\definecolor{cellbad}{RGB}{246,234,232}
\definecolor{subduedtext}{RGB}{100,100,108}
\definecolor{accentgood}{RGB}{50,90,145}
\definecolor{accentbad}{RGB}{155,75,72}
\definecolor{phenomenon}{RGB}{95,70,140}
\begin{tikzpicture}[x=1cm, y=1cm,
  phenarrow/.style={line width=1.0pt, color=phenomenon, line cap=round},
  phenlabel/.style={font=\small\itshape, color=phenomenon, inner sep=2pt},
  badge/.style={font=\scriptsize\itshape, color=phenomenon, inner sep=2pt,
                fill=white, fill opacity=0.78, text opacity=1, rounded corners=2pt}]

  \fill[cellgood] (1.65, -0.95) rectangle (7.15, -3.05);
  \fill[cellbad]  (7.15, -0.95) rectangle (12.65, -3.05);
  \fill[cellbad]  (1.65, -3.05) rectangle (7.15, -5.15);
  \fill[cellgood] (7.15, -3.05) rectangle (12.65, -5.15);

  \node[align=center, font=\small\bfseries, inner sep=0pt] at (4.40, -0.475)
    {Signal Channel\\[1pt]
     {\mdseries\footnotesize\itshape\color{subduedtext}$\range(\mathcal W_S)$}};
  \node[align=center, font=\small\bfseries, inner sep=0pt] at (9.90, -0.475)
    {Reservoir\\[1pt]
     {\mdseries\footnotesize\itshape\color{subduedtext}$\ker(\mathcal W_S)$}};

  \node[align=center, font=\small\bfseries, inner sep=0pt] at (0.825, -2.00) {Signal};
  \node[align=center, font=\small\bfseries, inner sep=0pt] at (0.825, -4.10) {Noise};

  \node[align=center, font=\small, text width=5.20cm, inner sep=0pt] at (4.40, -2.00)
    {{\color{accentgood}\textbf{Transfers to Test}}\\[5pt]
     $\tfrac{1}{n}\bm A_\circ\bm D\bigl(f^\star(S)-\bm U_S(0)\bigr)$};
  \node[align=center, font=\small, text width=5.20cm, inner sep=0pt] at (9.90, -2.00)
    {{\color{accentbad}\textbf{Residual Bias}}\\[5pt]
     bounded by $\|\bm R_\perp\|_{\op}$};

  \node[align=center, font=\small, text width=5.20cm, inner sep=0pt] at (4.40, -4.10)
    {{\color{accentbad}\textbf{Only Surviving Variance Term}}\\[5pt]
     $\tfrac{1}{n}\bm G\bm P_{\mathrm{sig}}\bm\varepsilon$};
  \node[align=center, font=\small, text width=5.20cm, inner sep=0pt] at (9.90, -4.10)
    {{\color{accentgood}\textbf{Trapped, Invisible to Test}}\\[5pt]
     $\tfrac{1}{n}\bm G\bm P_{\mathrm{res}}\bm\varepsilon=\bm 0$};


  \draw[phenarrow, -{Stealth[length=3mm,width=3mm]}]
    (10.7, 0.30) to[bend right=8] (4.1, 0.30);
  \node[phenlabel, anchor=south] at (7.15, 0.65)
    {\textbf{Grokking}};

  \draw[phenarrow, {Stealth[length=3mm,width=3mm]}-{Stealth[length=3mm,width=3mm]}]
    (4.1, -5.47) to[bend right=8] (10.7, -5.47);
  \node[phenlabel, anchor=north] at (7.15, -5.82)
    {\textbf{Double Descent}};

  \node[badge, anchor=north west]
    at (1.78, -0.99) {$\downarrow$\,\textbf{Implicit Bias}};

  \node[badge, anchor=south east]
    at (12.52, -5.11) {$\bigstar$\,\textbf{Benign Overfitting}};

\end{tikzpicture}
\end{adjustbox}
\caption{\textbf{Four-cell decomposition of test error.} Each cell is one contribution to $\bm U_Q(T)-f^\star(Q)$ from \eqref{eq:thesis_decomposition}. The two blue cells generalize correctly: clean signal transfers through $\bm A_\circ\bm D$, and any label noise the optimizer placed in the reservoir is killed unconditionally by $\bm G\bm P_{\mathrm{res}}=\bm 0$ (\autoref{thm:reservoir_invisibility}). The two red cells are the failure modes: signal in the reservoir feeds the residual bias (\autoref{thm:sobolev_remainder}), and noise in the signal channel is the only variance term that survives, suppressed by SGD drift-diffusion separation (\autoref{thm:minibatch_coherence}) and the population-risk gate of \autoref{sec:pop_risk_training}. Classical phenomena map onto these cells (purple annotations): \emph{Grokking} ($\leftarrow$, top) is signal migrating from the reservoir into the signal channel as the kernel evolves; \emph{Double Descent} ($\leftrightarrow$, bottom) is noise moving between channels as model capacity sweeps across interpolation (\autoref{rem:double_descent}); \emph{Implicit Bias} ($\downarrow$, top-left) is the spectral schedule of $\mathcal W_S(t)$ filling the signal channel from the largest eNTK eigenvalue down (\autoref{cor:grad_flow_filter_limit}); \emph{Benign Overfitting} ($\bigstar$, bottom-right) is noise sitting in the reservoir at interpolation (\autoref{cor:benign_overfitting}). Frozen-kernel filter and unified bias-variance: \autoref{sec:classical_phenomena}, \autoref{thm:unified_bias_variance}.}
\label{fig:four_cells}
\end{figure}
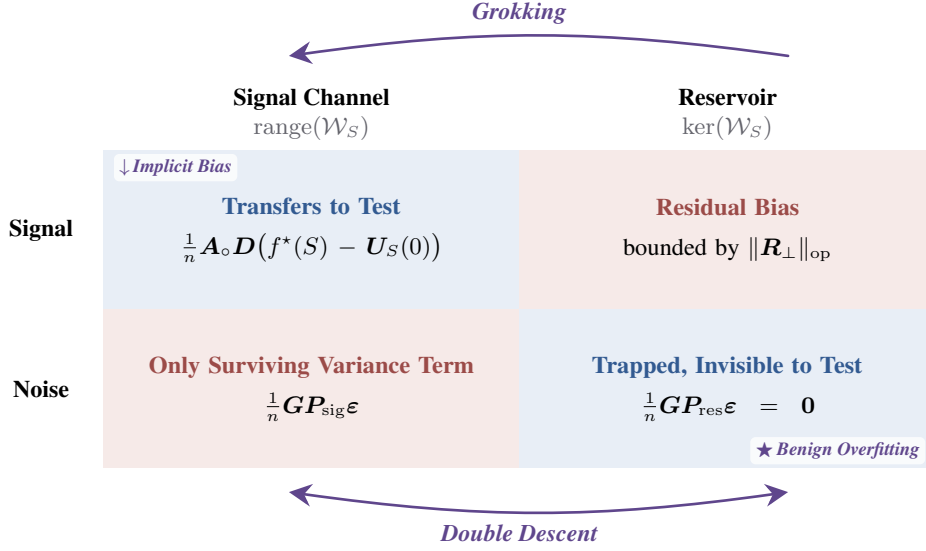

\section{Related Work}\label{sec:related}

The question of why overparameterized neural networks generalize has attracted sustained attention from both the statistical and optimization communities.

\paragraph{Worst-case bounds and algorithmic stability.}
Uniform-convergence bounds, whether expressed in terms of VC dimension \citep{vapnik1971uniform}, covering numbers \citep{dudley1967sizes}, Rademacher complexity \citep{bartlett2002rademacher}, weight norms \citep{bartlett1998sample,neyshabur2015norm}, or spectral complexity \citep{bartlett2017spectrally}, are vacuous at practical scale \citep{zhang2017understanding}, and \citet{nagarajan2019uniform} argued that uniform convergence is insufficient for deep learning. Algorithmic stability \citep{bousquet2002stability,hardt2016train} bounds the sensitivity to single-point perturbations and requires global Lipschitz constants, which are unavailable for nonconvex losses. PAC-Bayes bounds \citep{mcallester1999pac,dziugaite2017computing} incorporate a data-dependent posterior; the KL penalty becomes meaningful at practical scale once the prior is optimized along the training trajectory, reintroducing path dependence. Our theory provides the appropriate localization: global Lipschitz constants are replaced by path-dependent output-space quantities that capture the actual landscape encountered during training.

\paragraph{Kernel theories and benign overfitting.}
The neural tangent kernel (NTK) \citep{jacot2018neural,du2019gradient} shows that sufficiently wide networks evolve as kernel methods with a frozen tangent kernel, and generalization follows from classical kernel bounds \citep{arora2019fine,lee2019wide}. The benign-overfitting literature \citep{belkin2019reconciling,nakkiran2020deep,bartlett2020benign,tsigler2023benign,hastie2022surprises} established that interpolation can be statistically harmless under appropriate spectral decay, primarily for linear or kernel models. Our theory unifies both regimes: the frozen-kernel limit is a special case, full feature learning is handled with an evolving kernel, and benign overfitting is explained mechanistically as noise trapped in test-invisible directions.

\paragraph{Influence functions and leave-one-out.}
Classical influence functions \citep{cook1982residuals} approximate the effect of removing a training point via a single Newton step at the empirical risk minimizer. \citet{koh2017understanding} brought the tool to deep learning for data attribution, retaining the single-step linearization at the trained weights. The population-risk objective of \autoref{sec:population_risk} reads the leave-one-out displacement directly off the test-transfer operator $\mathsf G$ at the current step, with each training point treated as a one-point test set against the remaining batch. Where classical influence linearizes at the trained weights and folds the trajectory into a Hessian, the operator form gives a one-step kernel-block expression that the optimizer can compute from the gradients it already sees, and the same operator at the full window $T$ recovers the expected generalization gap as an average of self-influences (\autoref{thm:self_influence}).

\section{Output Space Dynamics: Signal Channel and Reservoir}\label{sec:second_law}

We work in output space, since that is where train and test predictions actually live. Let $\mathcal Z$ be a measurable instance space, let $S=(z_1,\dots,z_n)\in\mathcal Z^n$ be the training set, and let $F\colon\mathbb R^d\times\mathcal Z\to \mathbb R^p$ be $C^2$ in the parameters $\bm w$ for every instance $z$, supporting residual networks \citep{he2016deep} and Transformers \citep{vaswani2017attention}. Stack all training outputs into a single vector, assemble their parameter Jacobian, and form the kernel that governs which output directions training can move:
\begin{align}
\bm U_S(\bm w)
&\triangleq
\bigl(F(\bm w,z_1);\dots;F(\bm w,z_n)\bigr)\in\mathbb R^{np}, \label{def:stacked_output}\\
\bm J_S(\bm w)
&\triangleq
D_{\bm w}\bm U_S(\bm w)\in\mathbb R^{np\times d}, \label{def:stacked_jacobian}\\
\Kss(\bm w)
&\triangleq
\bm J_S(\bm w)\bm J_S(\bm w)^\top\succeq 0. \label{def:tangent_kernel}
\end{align}

Take $\Phi_S\colon\mathbb R^{np}\to\mathbb R$ convex and $C^2$ (squared loss has $\Phi_S(\bm u)=\tfrac{1}{2n}\|\bm u-\bm y\|_2^2$), let $L_S=\Phi_S\circ \bm U_S$, and write the output gradient and its Hessian as $\bm g(t)\triangleq\nabla_{\bm u}\Phi_S(\bm u(t))$ and $\bm B(t)\triangleq\nabla^2\Phi_S(\bm u(t))$. Under gradient flow $\partial_t \bm w=-\bm J_S^\top \bm g$, the chain rule yields the coupled output, output-gradient, and dissipation dynamics
\begin{align}\label{eq:gradient_flow}
\partial_t \bm u(t)
&=
-\Kss(t)\bm g(t),
\\
\partial_t \bm g(t)
&=
-\bm B(t)\Kss(t)\bm g(t),
\\
\tfrac{d}{dt}\Phi_S(\bm u(t))
&=
-\bm g(t)^\top \Kss(t) \bm g(t)
=
-\|\bm J_S^\top \bm g\|_2^2.
\end{align}
\phantomsection\label{thm:output_dynamics}\label{eq:output_dynamics_dynamics}\label{eq:output_dynamics_dissipation}
The output gradient therefore propagates as $\bm g(t)=\mathcal P_g(t,0)\bm g(0)$, where the propagator $\mathcal P_g(\cdot,s)$ solves the linear ODE
\begin{align}\label{eq:force_propagator}
\partial_t \mathcal P_g(t,s)
=
-\bm B(t)\Kss(t)\mathcal P_g(t,s),
\qquad
\mathcal P_g(s,s)=\bm I.
\end{align}
The eigenvectors of $\Kss(t)$ rotate during training, so the cumulative effect over a window $[s,T]$ requires integrating along the trajectory.

\begin{definition}[Cumulative Dissipation, Signal Channel, and Reservoir]\label{def:signal_channel}
Fix $0\le s\le T$. The \emph{cumulative dissipation Gramian} and its spectral projectors (derivation from output dynamics in \autoref{sec:operator_derivation}) are
\begin{align}
\mathcal W_S(s,T)
&\triangleq
\int_s^T \mathcal P_g(\tau,s)^\top \Kss(\tau)\mathcal P_g(\tau,s)d\tau,
\\
\bm P_{>\varepsilon}(s,T)
&\triangleq
\mathbf 1_{(\varepsilon,\infty)}\bigl(\mathcal W_S(s,T)\bigr),
\\
\bm P_{\le \varepsilon}(s,T)
&\triangleq
\mathbf 1_{[0,\varepsilon]}\bigl(\mathcal W_S(s,T)\bigr).
\end{align}
The \emph{signal channel} is $\range(\mathcal W_S(s,T))$, the directions where training dissipated loss; the \emph{reservoir} is $\ker\mathcal W_S(s,T)$, the directions where training dissipated none.
\end{definition}

For a test set $Q$ and $\bm W\triangleq\mathcal W_S(s,T)$, the test transfer operator is
\begin{align}\label{def:test_transfer}
\mathsf G_Q(T,s)
&\triangleq
\int_s^T \Kqs(\tau)\mathcal P_g(\tau,s)d\tau.
\end{align}
The chain rule gives $\bm U_Q(T)-\bm U_Q(s)=-\mathsf G_Q(T,s)\bm g(s)$, so $\mathsf G_Q$ propagates the output gradient to test displacement.

\begin{proposition}[Reservoir test-invisibility]\label{thm:reservoir_invisibility}
The test transfer operator vanishes on the reservoir,
\begin{align}\label{eq:reservoir_invisibility}
\ker \bm W \subseteq \ker \mathsf G_Q,
\end{align}
so spectral projectors of $\bm W$ onto small or zero eigenvalues annihilate $\mathsf G_Q$. The corresponding inequality $\bm G^\top\bm G \preceq \|\Gamma_Q(s,T)\|_{\op}\bm W$ on bounded functions of $\bm W$ is recorded in \autoref{sec:proofs_dynamics}.
\end{proposition}

Reservoir directions cannot affect any test prediction: residual error sitting in $\ker\mathcal W_S$ shows up on training outputs while contributing nothing at test. The same conclusion holds after any positive-semidefinite preconditioning of the parameter updates: substituting $\bm J_S\mapsto\bm J_S M^{1/2}$ in the proof yields $\ker\mathcal W_S^M\subseteq\ker\mathsf G_Q^M$ for the preconditioned operators (\autoref{sec:operator_derivation}), so the optimizer's choice of $M_t$ at each step determines which parameter directions enter the signal channel; \autoref{sec:pop_risk_training} picks the $M_t$ that maximizes a population-safe rate. The frozen-kernel limit reduces $\mathcal W_S$ to a closed-form spectral filter that recovers benign overfitting, double descent, implicit bias, grokking, and ridge regression as different choices of one preconditioner (\autoref{thm:unified_bias_variance}); we record this in \autoref{sec:classical_phenomena}, with the linear-model worked example in \autoref{sec:linear_model}.

Two questions remain about the signal channel. Inside it, what happens to the residual label noise the optimizer fitted? \autoref{sec:minibatch_coherence} shows that minibatch SGD's drift accumulates linearly along population-gradient directions while its centered fluctuation diffuses, so noise channels die at rate $1/\sqrt n + \sqrt{\eta T/b}$ against the $\Theta(T)$ accumulation on signal channels. And what predicts test motion from training motion when the kernel evolves over a typical run? \autoref{sec:third_law} shows that $\bm D$ and $\mathsf G_Q$ both factor through $\mathcal W_S^{1/2}$, and that under squared loss the test displacement is determined exactly by the training displacement on the realized window.

\section{Minibatch Drift Versus Diffusion}\label{sec:minibatch_coherence}

The reservoir cannot affect any test prediction. Inside the signal channel, the residual label noise the optimizer fitted is suppressed by minibatch SGD itself: the centered fluctuation $\bm\xi_k=\hat{\bm g}_k-\bm\mu_k$ has $\mathbb E[\bm\xi_k\mid\mathcal F_k]=0$, so it contributes only diffusion, and on a noise direction the population gradient at fresh draws also vanishes, so the drift dies as $1/\sqrt n$. A genuine signal direction has $\Theta(1)$ drift and accumulates linearly.

To make this precise, decompose each minibatch gradient into its conditional mean and a centered fluctuation, and write the test prediction under preconditioner $\bm M_k$,
\begin{align}
\hat{\bm g}_k &= \bm\mu_k + \bm\xi_k,
&
\bm\mu_k &\triangleq \mathbb E[\hat{\bm g}_k\mid\mathcal F_k],
&
\bm L_{Q,k} &\triangleq \bm J_Q(\bm w_k)\bm M_k.
\end{align}
Taylor-expanding the test prediction along $N$ SGD steps with step size $\eta$ and horizon $T=N\eta$ gives the decomposition together with a step-size bound on the second-order remainder, valid whenever $\bm J_Q$ is $\beta_Q$-Lipschitz along the trajectory:
\begin{align}
\bm U_Q(\bm w_N)-\bm U_Q(\bm w_0)
&=
\underbrace{-\eta\sum_{k=0}^{N-1}\bm L_{Q,k}\bm\mu_k}_{\bm\Delta_Q^{\mathrm{drift}}}
-
\underbrace{\eta\sum_{k=0}^{N-1}\bm L_{Q,k}\bm\xi_k}_{\bm\Delta_Q^{\mathrm{diff}}}
+\bm{\mathcal R}_Q,
\label{eq:drift_diffusion_decomposition}
\\
\|\bm{\mathcal R}_Q\|_2
&\le
\frac{\beta_Q}{2}\sum_{k=0}^{N-1}\|\bm w_{k+1}-\bm w_k\|_2^2.
\label{eq:drift_remainder}
\end{align}
The fluctuations $\eta\bm L_{Q,k}\bm\xi_k$ are martingale differences with respect to $\{\mathcal F_k\}$.

\begin{theorem}[Drift--diffusion separation]\label{thm:minibatch_coherence}
If the test-projected fluctuation covariance is uniformly bounded by $V_k/b$, the drift and diffusion terms accumulate at separated rates,
\begin{align}\label{eq:scale_separation}
\bigl\|\bm\Delta_Q^{\mathrm{drift}}\bigr\|_2 &= O(T),
&
\bigl\|\bm\Delta_Q^{\mathrm{diff}}\bigr\|_{L^2} &= O\!\left(\sqrt{\eta T/b}\right),
&
\mathbb E\|\bm\Pi\bm\mu_k\|_2^2 &= O(1/n),
\end{align}
where the last bound holds on a noise channel with vanishing population gradient at fresh draws under a replace-two stability hypothesis on the projected gradient (\autoref{app:proof_minibatch_coherence}), and is exact when the minibatch is independent of $\mathcal F_k$. The channel's test displacement is therefore $O(T/\sqrt n+\sqrt{\eta T/b})$, asymptotically smaller than the $\Theta(T)$ accumulation of a genuine signal channel.
\end{theorem}

A genuine signal channel has $\mathbb E_{Z\sim\mathcal D}[\bm\Pi\nabla_{\bm w}\ell\mid\mathcal F_k]\ne 0$, so $\|\bm\Pi\bm\mu_k\|_2=\Theta(1)$ and the drift accumulates linearly (\autoref{app:proof_minibatch_coherence}). The squared-mean-versus-trace comparison driving the off-diagonal agreement $\Omega_B$ reappears here as the separation between $\Theta(T)$ drift and $O(\sqrt{\eta T/b})$ diffusion (\autoref{thm:loo_test_transfer}). The next section restricts this trajectory-level statement to a per-parameter update at a single optimizer step.

\section{Train-Test Coupling under Feature Learning}\label{sec:third_law}

The reservoir is invisible at test, and \autoref{sec:minibatch_coherence} showed that residual label noise inside the signal channel decays under SGD. The remaining piece is the signal in the signal channel: when does training motion determine test motion in the feature-learning regime? Abbreviating $\bm W\triangleq\mathcal W_S(s,T)$, the training-side analogue of $\mathsf G_Q$ and its dissipation-normalized counterpart are
\begin{align}
\bm D
&\triangleq
\int_s^T \Kss(\tau)\mathcal P_g(\tau,s)d\tau
=
\mathsf C_S\bm W^{1/2},
&
\mathsf C_S
&\triangleq
\bm D\psinv{\bm W},
\label{eq:train_disp_def}
\\
\bm G
&\triangleq
\int_s^T \Kqs(\tau)\mathcal P_g(\tau,s)d\tau
=
\mathsf C_Q\bm W^{1/2},
&
\mathsf C_Q
&\triangleq
\bm G\psinv{\bm W},
\label{eq:dissipation_factorization}
\end{align}
both well-defined since $\bm D$ and $\bm G$ vanish on $\ker \bm W$ (\autoref{thm:reservoir_invisibility}); the chain rule gives the train-side companion $\bm U_S(T)-\bm U_S(s)=-\bm D\bm g(s)$ to the test relation $\bm U_Q(T)-\bm U_Q(s)=-\bm G\bm g(s)$. Orthogonally projecting $\mathsf C_Q$ onto $\range(\mathsf C_S^\top)$ produces the optimal linear predictor $\bm A_\circ$ and an irreducible remainder $\bm R_\perp$,
\begin{align}\label{eq:Aopt_def}
\bm A_\circ
\triangleq
\mathsf C_Q\mathsf C_S^\dagger,
\qquad
\bm R_\perp
\triangleq
\mathsf C_Q\bigl(\bm I-\mathsf C_S^\dagger\mathsf C_S\bigr),
\end{align}
in the operator analogue of regressing one Gaussian variable on another. For an arbitrary linear predictor $\bm A$, the error operator $(\bm G-\bm A\bm D)\bm W^\dagger(\bm G-\bm A\bm D)^\top$ splits orthogonally into the irreducible piece $\bm R_\perp\bm R_\perp^\top$ and a quadratic penalty in $\bm A-\bm A_\circ$ (\autoref{thm:transfer_error_decomp}), so $\bm A_\circ$ is the unique minimizer in the positive-semidefinite order. The frozen-kernel limit recovers classical kernel regression, $\bm A_\circ=\Kqs\Kss^\dagger$ and $\bm R_\perp=\bm 0$ (\autoref{sec:classical_phenomena}, \autoref{prop:linear_operators}).

\begin{theorem}[Train-test coupling]\label{thm:train_test_coupling}
Suppose $\Phi_S(\bm u)=\tfrac12(\bm u-\bm y)^\top \bm B(\bm u-\bm y)$ for some $\bm B\succ 0$ (squared loss has $\bm B=\tfrac1n\bm I$). On every finite window along the realized trajectory, $\ker\bm D=\ker\bm W$, the remainder vanishes, and the test displacement is determined exactly by the training displacement,
\begin{align}\label{eq:causal_closure_displacement}
\bm U_Q(T)-\bm U_Q(s)
=
\bm A_\circ\bigl(\bm U_S(T)-\bm U_S(s)\bigr),
\qquad
\bm A_\circ
=
\bm G\bm D^\dagger,
\end{align}
with no kernel-stability or asymptotic hypothesis.
\end{theorem}

\begin{proof}[Sketch]
With constant $\bm B$, integrating $\eqref{eq:force_propagator}$ gives explicitly
\begin{align}
\bm D
=
\bm B^{-1}\bigl(\bm I-\mathcal P_g(T,s)\bigr),
\qquad
\bm h^\top \bm W \bm h
=
\tfrac12\bigl(\|\bm h\|_{\bm B^{-1}}^2-\|\mathcal P_g(T,s)\bm h\|_{\bm B^{-1}}^2\bigr),
\end{align}
so $\bm D\bm h=\bm 0$ forces $\mathcal P_g(T,s)\bm h=\bm h$ and then $\bm h^\top \bm W\bm h=0$, giving $\ker \bm D=\ker \bm W$. Reservoir test-invisibility yields $\ker \bm D\subseteq\ker \bm G$, equivalent to $\bm R_\perp=\bm 0$. Full proof in \autoref{sec:proofs_coupling}.
\end{proof}

\paragraph{Test error decomposes into bias plus signal-channel variance.}
Write the labels as $\bm y = f^\star(S) + \bm\varepsilon$ and split the noise along the signal channel $\range(\bm W)$ and the reservoir $\ker\bm W$ via projectors $\bm P_{\mathrm{sig}}, \bm P_{\mathrm{res}}$. Under squared loss the initial gradient is $\bm g(0) = \tfrac{1}{n}(\bm U_S(0) - \bm y)$, so the exact test displacement $\bm U_Q(T) - \bm U_Q(0) = -\bm G\bm g(0) = \tfrac{1}{n}\bm G(\bm y - \bm U_S(0))$ separates as
\begin{align}\label{eq:thesis_decomposition}
\bm U_Q(T) - f^\star(Q)
=
\underbrace{\bm U_Q(0) + \tfrac{1}{n}\bm A_\circ\bm D\bigl(f^\star(S) - \bm U_S(0)\bigr) - f^\star(Q)}_{\text{bias, controlled by }\bm R_\perp}
+
\underbrace{\tfrac{1}{n}\bm G\bm P_{\mathrm{res}}\bm\varepsilon}_{=\bm 0}
+
\underbrace{\tfrac{1}{n}\bm G\bm P_{\mathrm{sig}}\bm\varepsilon}_{\text{signal-channel variance}}.
\end{align}
The bias is the optimal train-to-test predictor $\bm A_\circ$ applied to the clean training displacement $\tfrac{1}{n}\bm D(f^\star(S) - \bm U_S(0))$, exact under squared loss ($\bm G = \bm A_\circ\bm D$, $\bm R_\perp = \bm 0$); the Sobolev refinement $\|\bm R_\perp\|_{\op} \le C h_S^{m-d_{\mathcal M}/2}$ in \autoref{thm:sobolev_remainder} extends this to smooth networks on dense samples (tight without smoothness, \autoref{app:smoothness_required}). The reservoir term vanishes unconditionally by reservoir invisibility: $\ker\bm W \subseteq \ker\bm G$ (\autoref{thm:reservoir_invisibility}) and $\range(\bm P_{\mathrm{res}}) = \ker\bm W$, so $\bm G\bm P_{\mathrm{res}} = \bm 0$ as an algebraic identity, with no appeal to a pseudoinverse. The signal-channel term is the only remaining failure mode; \autoref{sec:minibatch_coherence} shows SGD's drift along it accumulates as $O(T)$ against $O(\sqrt T)$ noise diffusion, and \autoref{sec:pop_risk_training} derives the population-risk gate that targets it directly. Full statement and proof in \autoref{app:bias_variance_decomp}; the operators are computable on the realized path (\autoref{sec:forward_backward}).

\begin{figure}[t]
    \centering
    \safeincludegraphics[width=0.78\textwidth]{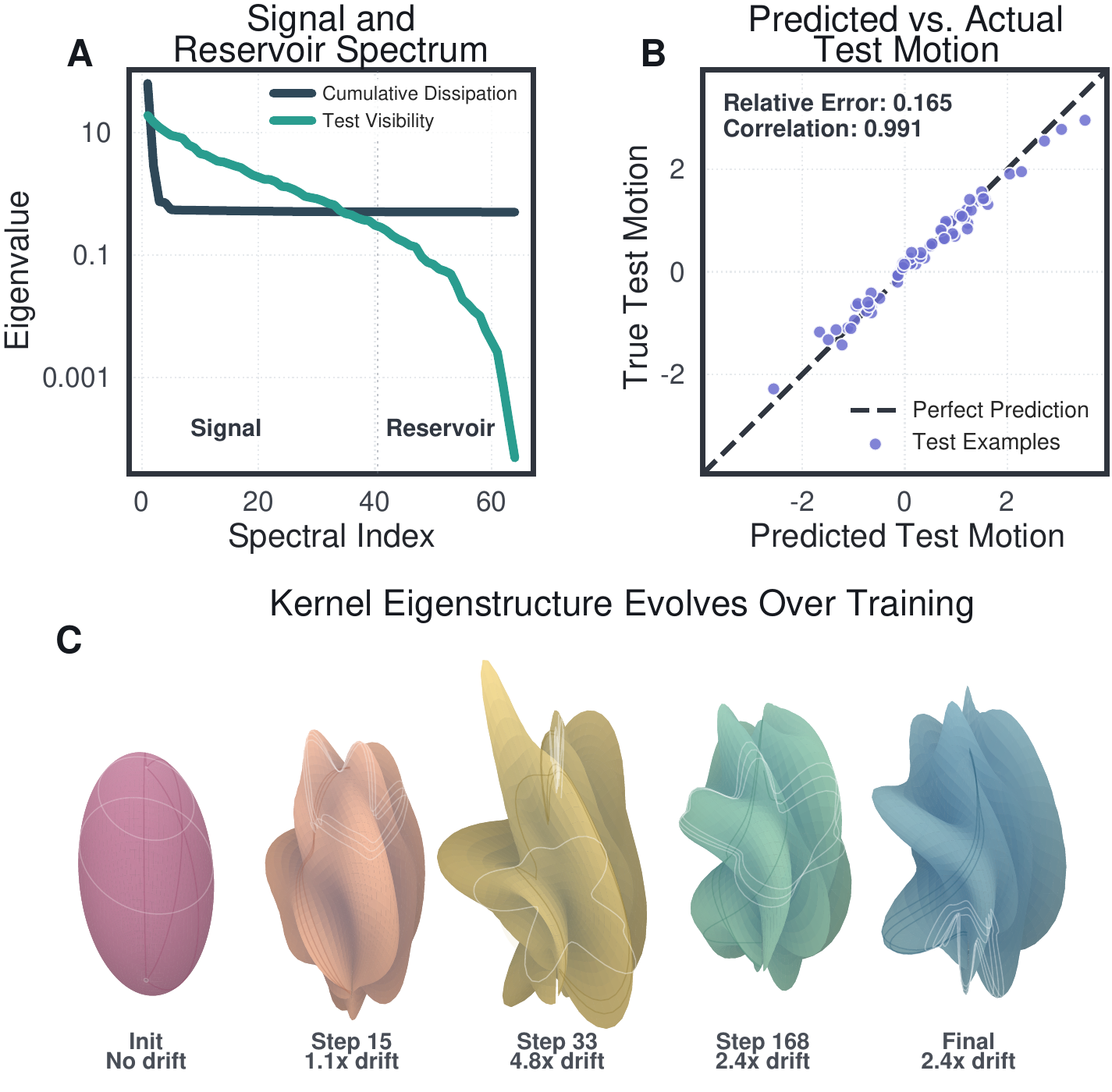}
    \caption{\textbf{Train-Test Coupling under Feature Learning (\autoref{thm:train_test_coupling}).}
    \textbf{(A)}~The test visibility spectrum $\lambda(\Gamma_Q)$ is strictly bounded
    by cumulative dissipation $\lambda(\mathcal{W}_S)$ at every spectral index.
    Directions past the dashed line make up the reservoir: they retain residual
    training error but cannot move test predictions.
    \textbf{(B)}~The optimal linear predictor $\bm A_\circ$ applied to observed training
    displacement recovers the true test displacement (correlation $0.991$,
    relative error $0.165$), confirming train-test coupling under full feature learning.
    \textbf{(C)}~Relative operator-norm drift $\|\Kss(t)-\Kss(0)\|_{\op}/\|\Kss(0)\|_{\op}$ of the empirical tangent kernel \eqref{def:tangent_kernel} over training, peaking at $4.8\times$ and settling at $2.4\times$; far outside the lazy regime, where this ratio stays $o(1)$. Despite this, the predictor in panels A--B holds exactly.}
    \label{fig:train_test_coupling}
\end{figure}

Inside the signal channel, both stable structure and idiosyncratic memorization reduce training loss, so a complexity measure $R\succ 0$ on output space turns the ranking into an eigenvalue problem for $R^{-1/2}\mathcal W_S R^{-1/2}$; the natural $R$ is the self-influence metric (\autoref{app:full_signal_direction_theorem}), which at a single optimizer step collapses to the gradient covariance $\bm\Sigma_B$ used in \autoref{sec:pop_risk_training}.

\section{Population Risk Training}\label{sec:pop_risk_training}\label{sec:population_risk}

The previous sections describe how training motion shapes test motion. The same operators turn the training data into an unbiased rate of population-risk decrease, and localizing to a single optimizer step yields a per-parameter rule the optimizer can compute from the gradients it already sees.

We use i.i.d.\ sampling only through exchangeability of the training $n$-tuple: $(S_{-i},Z_i)\stackrel{d}{=}(S_{n-1},Z)$. A held-out point's loss is then an unbiased sample of population risk for the model trained on the remaining data.

\begin{lemma}[Exchangeability]\label{thm:exchangeability_identity}
For a held-out subset $I\subset[n]$ and the model $\bm w_T(S_{-I})$ trained without $S_I$,
\begin{align}\label{eq:exchangeability_identity}
\mathbb E\!\left[\frac{1}{|I|}\sum_{i\in I}\ell\bigl(\bm w_T(S_{-I}),Z_i\bigr)\right]
=
\mathbb E\!\left[\mathcal L_{\mathcal D}\bigl(\bm w_T(S_{n-|I|})\bigr)\right].
\end{align}
\end{lemma}

For $|I|=1$ this is leave-one-out; for $|I|=n/K$ it is $K$-fold cross-validation; for $|I|=b$ on a fresh online minibatch it is the case the algorithm uses.

\paragraph{Population risk through the test transfer.}
Fix a window $[s,T]$ and an exchangeable batch $B=(z_1,\dots,z_b)$. For each $a$, set $Q_a\triangleq\{z_a\}$ and $S_{-a}\triangleq B\setminus\{z_a\}$, and read off the test transfer of \autoref{def:test_transfer} for this pair under preconditioner $M$. The new ingredient is the choice of $z_a$ as a one-point test set against the remainder of the batch; the operator itself is the same one that governed train-test coupling in \autoref{sec:third_law}.

\begin{theorem}[Population-risk rate]\label{thm:loo_test_transfer}
On a one-step window $[t,t+\eta]$ from $\bm w_t$ the propagator is the identity to first order, so
\begin{align}\label{eq:loo_test_transfer_collapse}
\mathsf G_{Q_a,S_{-a}}^M(t+\eta,t)
=
\eta\,\bm K^M_{Q_a,S_{-a}}(t)+O(\eta^2),
\end{align}
the row block of $\Kss^M$ pairing point $a$ with the rest of the batch (the self-block $K_{aa}^M$ is absent because $a\notin S_{-a}$). Averaging the leave-one-out improvement over $a$,
\begin{align}\label{eq:loo_average_omega}
\frac{1}{b}\sum_{a=1}^b\bigl(\ell_a(\bm w_t)-\ell_a(\bm w_{-a}^+)\bigr)
&= \eta\Omega_B(M)+O(\eta^2),
&
\Omega_B(M)
&\triangleq
\frac{1}{b(b-1)}\sum_{a\ne c}\bm r_a^\top K_{ac}^M\bm r_c,
\end{align}
and by \autoref{thm:exchangeability_identity} the conditional expectation of the left side given $\mathcal F_t$ is the population risk of the one-step learner trained on an independent $(b-1)$-sample.
\end{theorem}

The off-diagonal agreement $\Omega_B(M)$ is the kernel-block expression of the only failure mode left after \autoref{thm:reservoir_invisibility} and \autoref{thm:bias_variance_decomp}: the signal-channel noise $\tfrac{1}{n}\bm G\bm P_{\mathrm{sig}}\bm\varepsilon$ in \eqref{eq:thesis_decomposition}. The reservoir is invisible to both sides at once: $\ker\mathcal W_S\subseteq\ker\mathsf G$ kills the directions the run did not dissipate for every test prediction and for the population-safe rate. The training transfer $\bm D$ rides the full quadratic form $\bm g^\top\Kss^M\bm g$ including the self-blocks $K_{aa}^M$ that empirical-risk minimization sees, while $\mathsf G_{Q_a,S_{-a}}$ excludes them. Population-risk training is therefore the test-side analogue of empirical-risk descent. Specializing to a one-step window and lifting through $\bm g_a=\bm J_a^\top \bm r_a$ collapses $\Omega_B(M)$ to a parameter-space objective $\operatorname{tr}(M\bm A_B)$ with $\bm A_B=\bar{\bm g}_B\bar{\bm g}_B^\top-\tfrac{1}{b-1}\bm\Sigma_B$ (\autoref{sec:kernel_block_derivation}, \autoref{thm:kernel_increment}, \autoref{thm:minibatch_covariance}).

\begin{corollary}[Population-Risk Descent]\label{cor:pop_risk_descent}
For diagonal $\bm P_t=\operatorname{diag}(p_k)$, the unique binary mask maximizing $\operatorname{tr}(M\bm A_B)$ over $0\preceq M\preceq\bm P_t$ updates parameter $k$ exactly when $\mu_k^2>\sigma_k^2/(b-1)$, where $\mu_k=\bar g_{B,k}$ and $\sigma_k^2=(\bm\Sigma_B)_{kk}$. The cutoff is tight in both directions: a parameter that fails it admits an adversarial loss curvature forcing a strict first-order increase in population risk (\autoref{thm:diagonal_gate}).
\end{corollary}

\paragraph{Influence functions on the realized path.}
Evaluated at the full window $T$ rather than a single step, the same operator $\mathsf G_{Q_i,S}(T)$ becomes the leave-one-out displacement that controls each training point's contribution to the generalization gap, and averaging over $i$ recovers the gap exactly (\autoref{thm:self_influence}, \autoref{sec:self_influence_full}). This generalizes classical influence functions \citep{cook1982residuals,koh2017understanding}: where those linearize a Hessian and lose the path, the operator form integrates $\mathsf G$ along the realized trajectory and so remains exact under $\mathcal{O}(1)$ kernel drift.

\paragraph{From the rate to the algorithm.}
Each one-step kernel increment $\eta\Kss^{M_t}$ integrates into $\mathcal W_S^M$, so a sequence of one-step rate-maximizers is the greedy policy whose integral is the signal-channel content of the trajectory through $\mathsf G$, exactly as plain SGD is the greedy step whose integral is empirical-risk descent through $\bm D$. The diagonal cutoff $\mu_k^2>\sigma_k^2/(b-1)$ is the optimal first-order preconditioner for population risk on any diagonal base, and a streaming variance EMA $\hat{\bm s}_t$ of squared gradient deviations realizes it as a one-line change to AdamW: one extra parameter-sized state vector and a per-parameter gate that multiplies the standard moment update (\autoref{eq:soft_filter}, \autoref{alg:pop_risk}). The leave-one-out coefficient is $\alpha=1$ on the fresh-batch boundary and $\alpha=b/(n-b)$ on the finite-dataset boundary (\autoref{thm:minibatch_covariance}); soft and SNR forms of the gate and multi-epoch corrections are in \autoref{sec:variance_estimator_details}.

\begin{figure}[!tbp]
    \centering
    \safeincludegraphics[width=0.8\textwidth]{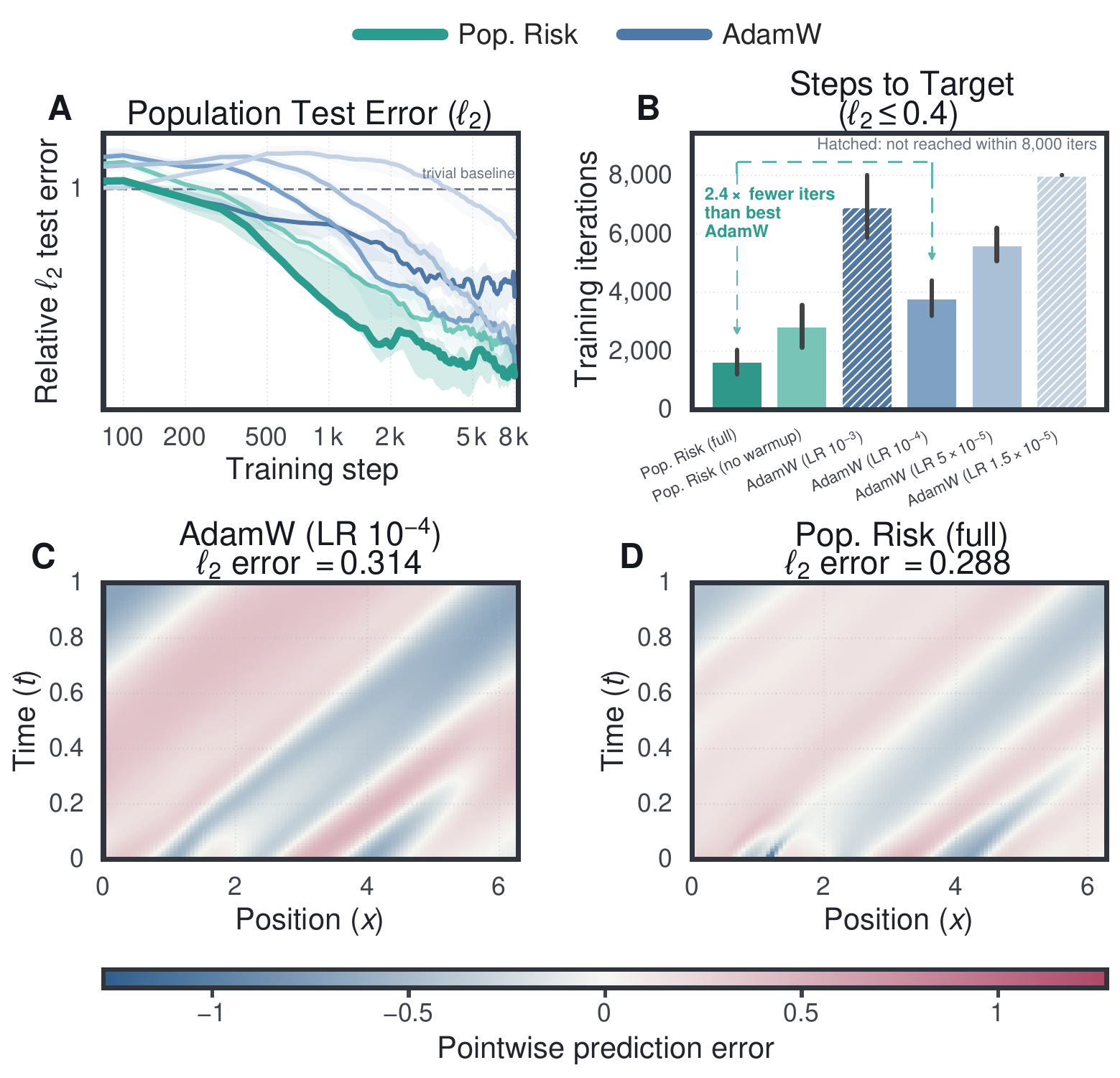}
    \caption{\textbf{Population-risk training on a noisy-IC PINN.} Periodic $u_t+\beta u_x=0$ at $\beta=5$, trained from a Gaussian-noisy initial condition. \textbf{(A)}~Relative $\ell_2$ test error vs.\ iterations. \textbf{(B)}~Iterations to $\ell_2\le 0.40$: $2.4\times$ fewer than the best learning-rate-tuned AdamW; hatched bars mark runs that did not reach the target in $8{,}000$ iterations. \textbf{(C,D)}~Pointwise error fields. Full ablation in \autoref{tab:pinn_ablation}.}
    \label{fig:pinn_convection}
\end{figure}

\begin{figure}[!tbp]
    \centering
    \safeincludegraphics[width=0.64\textwidth]{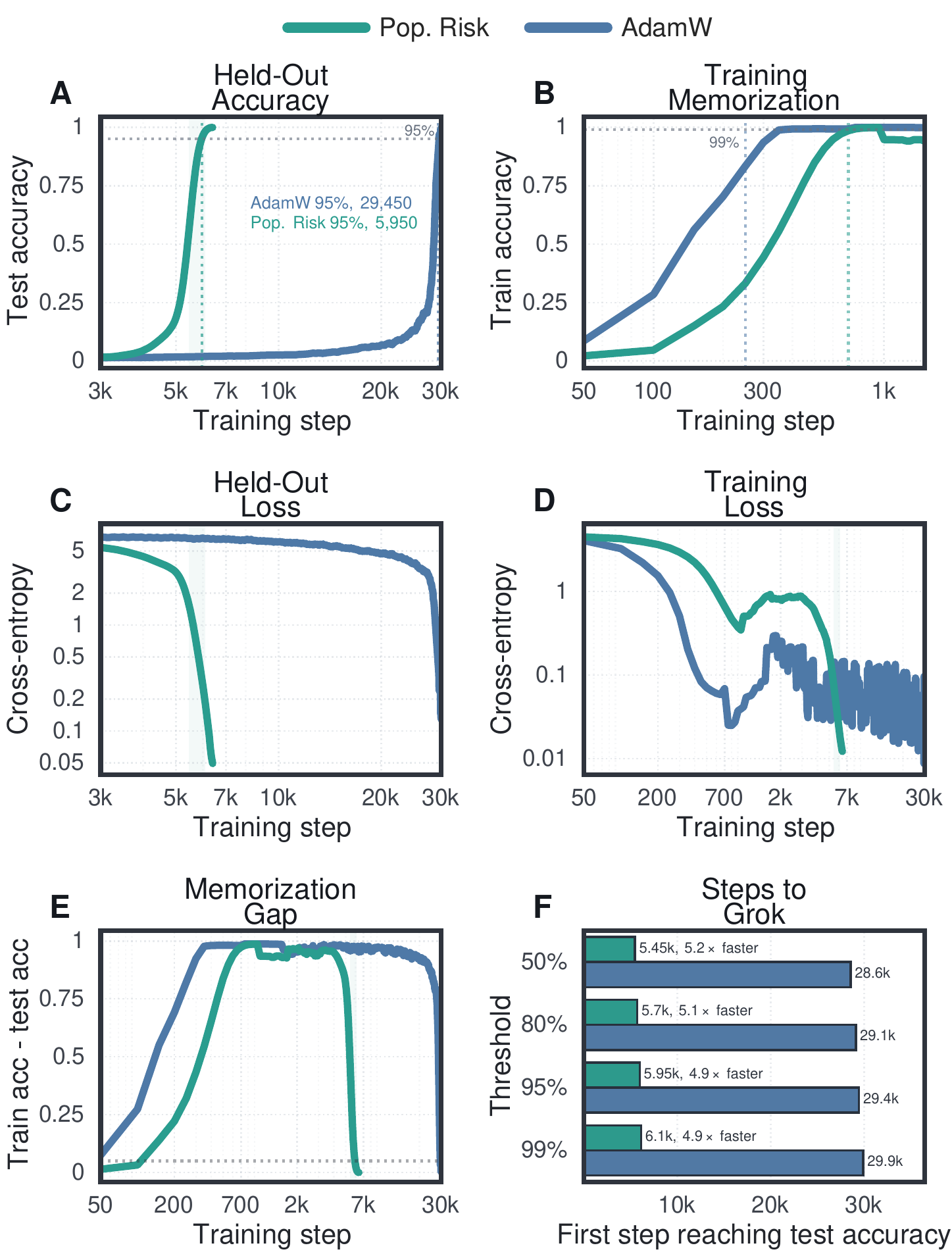}
    \caption{\textbf{Population-risk training collapses the grokking delay.} Same $2$-layer Transformer on modular division $a\cdot b^{-1}\bmod 97$ with $25\%$ training fraction. Population-risk training reaches $95\%$ held-out accuracy at step $5{,}950$ versus $29{,}450$ for AdamW ($4.9\times$ fewer steps).}
    \label{fig:grokking}
\end{figure}

\begin{figure}[!htbp]
    \centering
    \safeincludegraphics[width=0.62\textwidth]{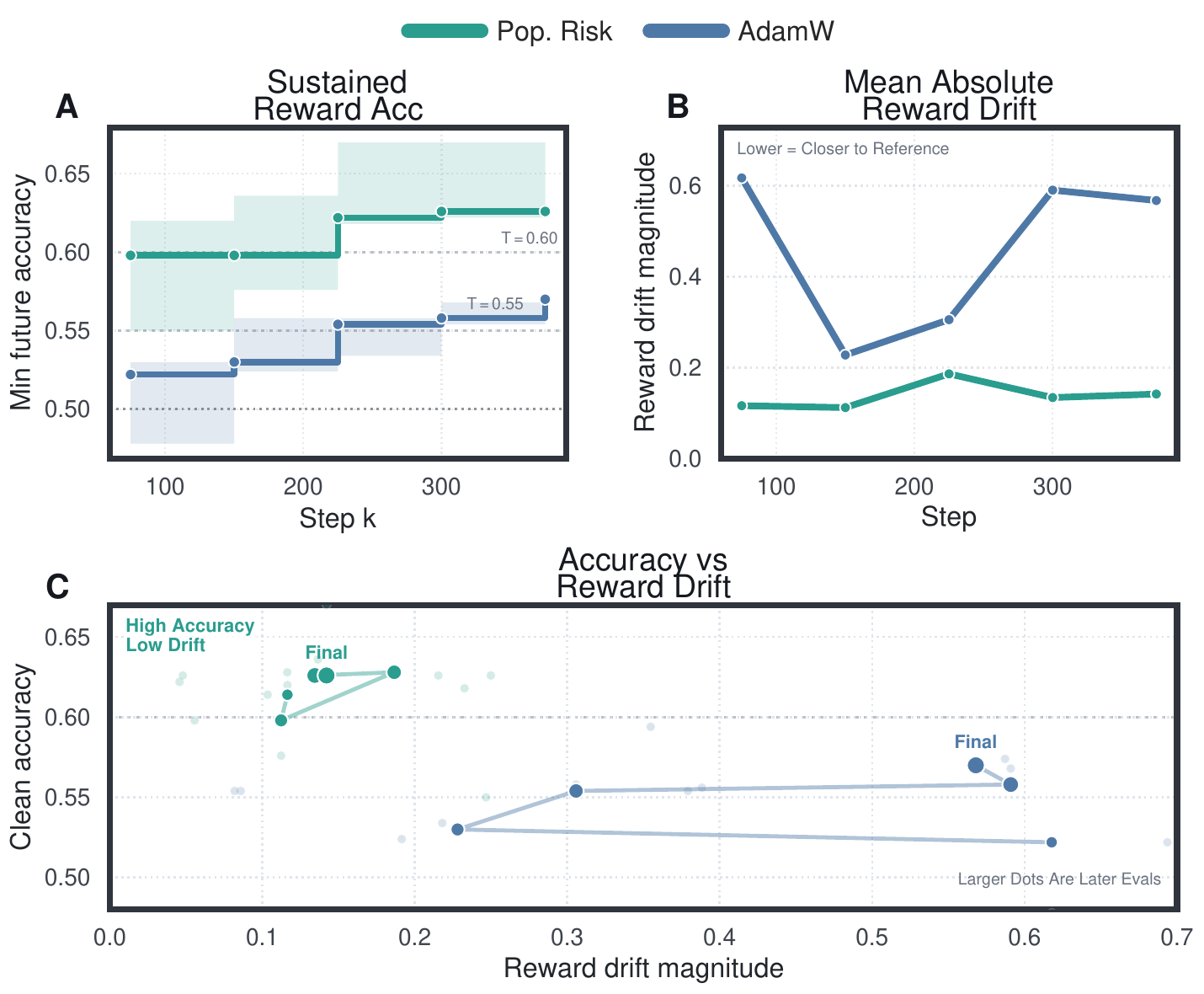}
    \caption{\textbf{Population-risk training on noisy preference alignment.} Qwen2.5-0.5B-Instruct fine-tuned with DPO on $30\%$-swapped UltraFeedback preferences, 3 seeds.
    \textbf{(A)}~Sustained reward accuracy (minimum clean-eval accuracy from each step onward); population-risk training holds above $T{=}0.60$ for the entire second half of training while AdamW only crosses $T{=}0.55$ late.
    \textbf{(B)}~Mean absolute reward drift from the reference policy.
    \textbf{(C)}~Accuracy--drift phase plot.}
    \label{fig:noisy_dpo}
\end{figure}

\paragraph{Results.}
We test the resulting rule across three regimes where empirical-risk training is known to overfit structured noise or to memorize before generalizing; architecture, data split, and optimizer hyperparameters are held fixed and only the population-risk update changes. On a PINN solving $u_t+\beta u_x=0$ at $\beta=5$ from a Gaussian-noisy initial condition, the rule reaches relative $\ell_2\le 0.40$ in $2.4\times$ fewer iterations than the best learning-rate-tuned AdamW (\autoref{fig:pinn_convection}, \autoref{tab:pinn_ablation}). On modular division $a\cdot b^{-1}\bmod 97$ at $25\%$ training fraction, where empirical-risk training is known to grok \citep{power2022grokking}, the same update reaches $95\%$ held-out accuracy at step $5{,}950$ versus $29{,}450$ for AdamW (\autoref{fig:grokking}). Fine-tuning Qwen2.5-0.5B-Instruct \citep{qwen2.5} with DPO under $30\%$ swapped UltraFeedback preferences (\autoref{fig:noisy_dpo}, \autoref{tab:noisy_dpo}) improves final reward accuracy from $0.566$ to $0.641$ while staying $3.05\times$ closer to the reference policy in mean absolute reward drift.

\section{Conclusion}\label{sec:conclusion}

As a deep network trains, its empirical NTK rotates label noise into a test-invisible reservoir, and SGD's centered fluctuation suppresses what survives in the signal channel; on the signal side, training motion determines test motion, even when the kernel drifts by $\mathcal{O}(1)$. The same operators turn a single batch into an unbiased rate of population-risk decrease, maximized through the optimizer's metric by a per-parameter gate. The frozen-kernel limit of our theory reproduces benign overfitting, double descent, implicit bias, and grokking from one spectral filter.

\bibliographystyle{plainnat}
\bibliography{refs}

\appendix

\section{Summary of Assumptions}\label{app:assumptions}

\autoref{tab:assumptions} collects every assumption used in the paper, where it enters, and what it controls. \autoref{sec:second_law} and \autoref{sec:third_law} are fully deterministic: no distributional assumption on the data.

\begin{table}[h]
\centering
\caption{Assumptions used in this paper. Rows are grouped by type: regularity conditions that hold throughout, conditions that strengthen specific results, and statistical conditions that enter only in \autoref{sec:population_risk}--\autoref{sec:pop_risk_training}. Assumptions marked with $(\star)$ are used only in the appendix.}
\label{tab:assumptions}
\small
\renewcommand{\arraystretch}{1.35}
\rowcolors{2}{white}{gray!10}
\begin{tabular}{@{}>{\raggedright\arraybackslash}p{0.30\textwidth}@{\hspace{0.7em}}>{\raggedright\arraybackslash}p{0.21\textwidth}@{\hspace{0.7em}}>{\raggedright\arraybackslash}p{0.40\textwidth}@{}}
\toprule
Assumption & Where used & Role \\
\midrule
$F(\bm w,z)$ is $C^2$ in $\bm w$ for every instance $z$
  & \autoref{sec:second_law}--\autoref{sec:pop_risk_training}
  & Chain rule for output dynamics; propagator ODE well-posed \\
$\Phi_S$ convex and $C^2$ in $\bm u$
  & \autoref{sec:second_law}--\autoref{sec:third_law}
  & Loss dissipation monotone ($\tfrac{d}{dt}\Phi_S\le 0$); Bregman divergence nonneg.\ \\
\addlinespace[5pt]
$\nabla^2\Phi_S=\bm B$ constant (loss quadratic in $\bm u$)
  & \autoref{thm:train_test_coupling}
  & $\bm R_\perp=\bm 0$: training displacement determines test displacement on the realized window. Without this, the general theory still gives a nonzero $\bm R_\perp$ bound \\
$\bm J_Q$ is $\beta_Q$-Lipschitz along the realized trajectory
  & \autoref{thm:minibatch_coherence}
  & Controls the second-order remainder in the drift--diffusion decomposition \\
$(\star)$ Compact data manifold $\mathcal M$ with $F(\bm w,\cdot)\in C^m(\mathcal M)$, $m>d_{\mathcal M}/2$
  & \autoref{thm:sobolev_remainder}
  & Fill-distance bound $\|\bm R_\perp\|_{\op}\le Ch_S^{m-d_{\mathcal M}/2}\Lambda_m$ \\
$(\star)$ Complexity measure $R\succ 0$ on output space
  & \autoref{thm:train_only_minimax_signal}
  & Defines $C_R$; ranks moved directions by loss dissipated per unit complexity \\
\addlinespace[5pt]
Training examples i.i.d.\ from $\mathcal D$ (used only through exchangeability of the $n$-tuple, plus fresh draws $Z_i'\sim\mathcal D$ independent of $S$ where stated)
  & \autoref{sec:population_risk}--\autoref{sec:pop_risk_training}
  & Defines population risk; held-out losses unbiased; martingale structure of minibatch noise \\
Current minibatch independent of optimizer history $\mathcal F_t$
  & \autoref{thm:kernel_increment}
  & One-step LOO risk is the conditional population risk on the batch boundary. Broken by multi-epoch replay; residual bounded by total variation (\autoref{app:proof_minibatch_coherence}) \\
Replace-two stability $\varepsilon_{k,n}=O(1/n)$ of the projected gradient (\autoref{eq:replace_two_stability})
  & \autoref{thm:minibatch_coherence}
  & Empirical noise mean inherits the population $O(1/n)$ cancellation rate; without it, finite-sample dependence through $\bm w_k(S)$ obstructs the bound \\
\bottomrule
\end{tabular}
\end{table}

\section{Notation}\label{app:notation}

\autoref{tab:notation} collects every mathematical object used in the main text, grouped by role: data and architecture; output trajectory and losses; time and scalar parameters; kernels and propagator; cumulative dissipation; transfer operators; population risk and influence; and optimizer state. Operators carrying a window argument $(s,T)$ are evaluated along the realized training trajectory on $[s,T]$.

\small
\renewcommand{\arraystretch}{1.25}
\begin{longtable}{@{}>{\raggedright\arraybackslash}p{0.22\textwidth}@{\hspace{0.6em}}>{\raggedright\arraybackslash}p{0.25\textwidth}@{\hspace{0.6em}}>{\raggedright\arraybackslash}p{0.45\textwidth}@{}}
\caption{Mathematical objects used in the main paper. Bold roman symbols are vectors or matrices; calligraphic and sans-serif operators ($\mathcal W$, $\mathcal P_g$, $\mathsf G$, $\mathsf C$, $\mathsf J$) act on the same spaces and are time-dependent through the realized trajectory. A superscript $M$ on an operator means the preconditioned form, in which the optimizer's preconditioner $M_t$ has been inserted at every instant.}\label{tab:notation}\\
\toprule
Symbol & Space / dimension & Meaning \\
\midrule
\endfirsthead
\multicolumn{3}{l}{\textit{Table~\ref{tab:notation} (continued).}}\\
\toprule
Symbol & Space / dimension & Meaning \\
\midrule
\endhead
\midrule
\multicolumn{3}{r}{\textit{Continued on next page.}}\\
\endfoot
\bottomrule
\endlastfoot
\multicolumn{3}{@{}l}{\textit{Data and architecture.}}\\
\rowcolor{gray!10}$\mathcal Z,\ \mathcal D$ & instance space, prob.\ measure on $\mathcal Z$ & Sample space and the unknown population law on it \\
$S=(z_1,\dots,z_n)$ & $\mathcal Z^n$ & Training set; i.i.d.\ from $\mathcal D$, used only through exchangeability \\
\rowcolor{gray!10}$Q$ & $\mathcal Z^{n_Q}$ & Test set; arbitrary and fixed \\
$n,\ b,\ d,\ p$ & $\mathbb N$ & Train size, batch size, parameter count, per-example output width \\
\rowcolor{gray!10}$F(\bm w,z)$ & $\mathbb R^d\times\mathcal Z\to\mathbb R^p$ & Network output, $C^2$ in $\bm w$ \\
$\bm w$ & $\mathbb R^d$ & Trainable parameters \\
\midrule
\multicolumn{3}{@{}l}{\textit{Output trajectory and losses.}}\\
\rowcolor{gray!10}$\bm U_S(\bm w),\ \bm U_Q(\bm w)$ & $\mathbb R^{np},\ \mathbb R^{n_Q p}$ & Stacked train and test outputs \\
$\bm u(t),\ \bm y$ & $\mathbb R^{np}$ & Output prediction along training and the (squared-loss) target \\
\rowcolor{gray!10}$\bm e_i(T)$ & $\mathbb R^p$ & Training residual at point $i$ at time $T$, $\bm U_{Q_i}^S(T)-\bm y_i$ \\
$\bm J_S(\bm w),\ \bm J_Q(\bm w)$ & $\mathbb R^{np\times d},\ \mathbb R^{n_Q p\times d}$ & Output Jacobians, $D_{\bm w}\bm U_S$ and $D_{\bm w}\bm U_Q$ \\
\rowcolor{gray!10}$\Phi_S$ & $\mathbb R^{np}\to\mathbb R$ & Convex $C^2$ training loss on outputs; squared loss is $\tfrac{1}{2n}\|\bm u-\bm y\|_2^2$ \\
$\bm g(t),\ \bm B(t)$ & $\mathbb R^{np},\ \mathbb R^{np\times np}$ & Output gradient $\nabla_{\bm u}\Phi_S$ and Hessian $\nabla^2_{\bm u}\Phi_S$ at time $t$ \\
\rowcolor{gray!10}$\ell(\bm w,z),\ \ell_i(\bm w)$ & $\mathbb R$ & Per-example loss and its evaluation at $z_i$, $\ell_i(\bm w)=\ell(\bm w,z_i)$ \\
\midrule
\multicolumn{3}{@{}l}{\textit{Time and scalar parameters.}}\\
\rowcolor{gray!10}$T,\ s,\ t,\ \tau$ & $\mathbb R_{\ge 0}$ & Training horizon and intermediate times along $[0,T]$ \\
$\eta$ & $\mathbb R_{>0}$ & Learning rate (one-step size) \\
\rowcolor{gray!10}$\alpha$ & $\{1,\ b/(n-b)\}$ & LOO coefficient: $1$ on the fresh-batch boundary, $b/(n-b)$ on the finite-dataset boundary \\
$\varepsilon,\ \epsilon$ & $\mathbb R_{>0}$ & Small constants in the soft mask numerator/denominator and the Adam denominator \\
\rowcolor{gray!10}$\lambda_{\mathrm{pop}},\ \lambda_{\mathrm{wd}}$ & $\mathbb R_{\ge 0}$ & Population-risk regularization and weight-decay coefficients \\
$\beta_1,\ \beta_2,\ \rho$ & $[0,1)$ & EMA decay rates for $\bm m_t,\ \bm v_t,\ \bm s_t$ \\
\midrule
\multicolumn{3}{@{}l}{\textit{Kernels and propagator.}}\\
\rowcolor{gray!10}$\bm K_{SS}(\bm w),\ \bm K_{QS}(\bm w)$ & $\mathbb R^{np\times np},\ \mathbb R^{n_Q p\times np}$ & Empirical tangent kernel $\bm J_S\bm J_S^\top$ and test-train kernel $\bm J_Q\bm J_S^\top$ \\
$\Kss,\ \Kqs$ & --- & Shorthand for $\bm K_{SS}$ and $\bm K_{QS}$ \\
\rowcolor{gray!10}$\Kss^M$ & $\mathbb R^{np\times np}$ & Preconditioned tangent kernel $\bm J_S M_t\bm J_S^\top$ \\
$\mathcal P_g(t,s)$ & $\mathbb R^{np\times np}$ & Output-gradient propagator: $\bm g(t)=\mathcal P_g(t,s)\bm g(s)$ \\
\midrule
\multicolumn{3}{@{}l}{\textit{Cumulative dissipation.}}\\
\rowcolor{gray!10}$\mathcal W_S(s,T)$ & $\mathbb R^{np\times np}$, $\succeq 0$ & Cumulative dissipation Gramian on the window $[s,T]$ \\
$\mathcal W_S^M(s,T)$ & $\mathbb R^{np\times np}$, $\succeq 0$ & Preconditioned dissipation Gramian; same construction with $\Kss^M$ in the integrand \\
\rowcolor{gray!10}$d\mathcal W_t^M$ & $\mathbb R^{np\times np}$ & Channel increment at time $t$, $\mathcal P_g(t,s)^\top \bm J_S(t) M_t \bm J_S(t)^\top \mathcal P_g(t,s) dt$ \\
$\bm P_{>\varepsilon}(s,T),\ \bm P_{\le\varepsilon}(s,T)$ & $\mathbb R^{np\times np}$ & Spectral projectors of $\mathcal W_S$ onto the signal channel and reservoir \\
\midrule
\multicolumn{3}{@{}l}{\textit{Transfer operators.}}\\
\rowcolor{gray!10}$\bm D(s,T),\ \mathsf C_S$ & $\mathbb R^{np\times np}$ & Train displacement integral and its dissipation-normalized form $\bm D\mathcal W_S^{\dagger/2}$ \\
$\mathsf G_Q(T,s),\ \mathsf C_Q$ & $\mathbb R^{n_Q p\times np}$ & Test transfer integral and its dissipation-normalized form $\mathsf G_Q\mathcal W_S^{\dagger/2}$ \\
\rowcolor{gray!10}$\bm A_\circ,\ \bm R_\perp$ & $\mathbb R^{n_Q p\times np}$ & Optimal linear train-to-test predictor $\mathsf C_Q\mathsf C_S^\dagger$ and its irreducible remainder \\
\midrule
\multicolumn{3}{@{}l}{\textit{Population risk and influence.}}\\
\rowcolor{gray!10}$\mathcal L_{\mathcal D}(\bm w),\ \widehat{\mathcal L}_S^\Psi$ & $\mathbb R$ & Population risk under loss $\Psi$ and its empirical estimate on $S$ \\
$\widehat L_B(\bm w)$ & $\mathbb R$ & Empirical loss on a minibatch $B$, $\tfrac{1}{b}\sum_{a\in B}\ell_a(\bm w)$ \\
\rowcolor{gray!10}$\mathcal L_{\mathrm{pop}}^\Psi(T,S)$ & $\mathbb R$ & First-order training-only estimate of population risk \\
$\mathcal R_{\mathrm{1ex}}^\eta,\ \mathcal R_{\mathrm{1ex},B}^\eta$ & $\mathbb R$ & One-step LOO risks on the full dataset and on a fresh batch \\
\rowcolor{gray!10}$\Psi_z$ & $\mathbb R^p\to\mathbb R$ & Per-example test loss at instance $z$ \\
$\bm C_n$ & $\mathbb R^{n\times n}$ & Centering projector $\bm I-\tfrac{1}{n}\bm 1\bm 1^\top$ \\
\rowcolor{gray!10}$\bm \nu^{(i)}$ & $\mathbb R^n$ & Mass-preserving delete-one direction, $\tfrac{n}{n-1}\bm C_n\bm e_i$ \\
$\mathsf J_\Psi(T,S),\ \mathsf J_\Psi^M$ & $\mathbb R^{n\times n}$ & Influence matrix and its preconditioned counterpart; $\mathsf J_{ij}$ is the linearized effect of downweighting $j$ on $\Psi$ at $i$ \\
\midrule
\multicolumn{3}{@{}l}{\textit{Optimizer state.}}\\
\rowcolor{gray!10}$\mathcal F_t$ & $\sigma$-algebra & Optimizer history up to step $t$ \\
$M_t,\ \bm P_t$ & $\mathbb R^{d\times d}$, $\succeq 0$ & Optimizer preconditioner and its diagonal restriction $\operatorname{diag}(p_k)$ \\
\rowcolor{gray!10}$\hat{\bm g}_k=\bm \mu_k+\bm \xi_k$ & $\mathbb R^d$ & Minibatch gradient as conditional mean plus zero-mean fluctuation \\
$\bm g_a,\ \bar{\bm g}_B$ & $\mathbb R^d$ & Per-example gradient $\nabla_{\bm w}\ell(\bm w_t,z_a)$ and batch mean $b^{-1}\sum_a \bm g_a$ \\
\rowcolor{gray!10}$\bm c_a,\ \bm \Sigma_B$ & $\mathbb R^d,\ \mathbb R^{d\times d}$ & Centered per-example gradient $\bm g_a-\bar{\bm g}_B$ and batch covariance $b^{-1}\sum_a \bm c_a\bm c_a^\top$ \\
$\mu_k,\ \sigma_k^2$ & $\mathbb R$ & Per-parameter mean $(\bar{\bm g}_B)_k$ and variance $(\bm \Sigma_B)_{kk}$ \\
\rowcolor{gray!10}$\bm A_B$ & $\mathbb R^{d\times d}$ & Off-diagonal rate matrix $\bar{\bm g}_B\bar{\bm g}_B^\top-\tfrac{1}{b-1}\bm \Sigma_B$ \\
$\Omega_B(\bm P_t)$ & $\mathbb R$ & Off-diagonal exchangeable rate $\operatorname{tr}(\bm P_t \bm A_B)$ \\
\rowcolor{gray!10}$q_k$ & $[0,1]$, per parameter & Population-safety mask; binary form is $\bm 1\{\mu_k^2>\sigma_k^2/(b-1)\}$ \\
$\bm m_t,\ \bm v_t,\ \bm s_t$ & $\mathbb R^d$ & Raw EMAs of $\bm g_t$, $\bm g_t^{\odot 2}$, and $(\bm g_t-\bm m_{t-1})^{\odot 2}$ \\
\rowcolor{gray!10}$\hat{\bm m}_t,\ \hat{\bm v}_t,\ \hat{\bm s}_t$ & $\mathbb R^d$ & Bias-corrected versions of $\bm m_t, \bm v_t, \bm s_t$ used in the update rule \\
\end{longtable}

\section{Output-Space Dynamics: Proofs}\label{sec:proofs_dynamics}\label{sec:proofs_dynamics_alias}

\paragraph{Intuition.} Under gradient flow the loss gradient decays at rates set by a tangent kernel weighted by loss curvature, and every quantity of interest (training displacement, test displacement, dissipation) is the integral of how that vector flows. The proofs below extend the neural tangent kernel framework of \citet{jacot2018neural,du2019gradient} to the feature-learning regime where the kernel itself evolves with the parameters, and isolate the operators used in the train-test coupling and population-risk arguments later.

\subsection{Deriving the Operators from Output Dynamics}\label{sec:operator_derivation}

\paragraph{Intuition.} The operators $\bm D_S$, $\bm G_Q$, and $\mathcal W_S$ used throughout \autoref{sec:second_law} and \autoref{sec:third_law} derive from integrating the output dynamics of \autoref{thm:output_dynamics} over time. Here, we write that integration explicitly so the dissipation Gramian's spectrum has a direct reading: the eigenvalue along $\bm \psi_j$ is the total integrated squared reachability of $\bm \psi_j$ across the window.

\paragraph{Train and test displacement integrals.} The output dynamics give $\partial_\tau \bm u(\tau) = -\Kss(\tau)\bm g(\tau)$, and the propagator gives $\bm g(\tau) = \mathcal P_g(\tau,s)\bm g(s)$. Substituting and integrating over $[s,T]$, with $\bm g(s)$ constant in $\tau$,
\begin{align}
\bm u(T) - \bm u(s)
&=
-\int_s^T \Kss(\tau)\mathcal P_g(\tau,s)\bm g(s)d\tau
=
-\bm D_S(T,s)\bm g(s),
\\
\bm D_S(T,s)
&\triangleq
\int_s^T \Kss(\tau)\mathcal P_g(\tau,s)d\tau.
\end{align}
The same derivation with the test-side dynamics $\partial_\tau \bm U_Q = -\Kqs\bm g$ and $\Kqs(\tau) = \bm J_Q(\tau)\bm J_S(\tau)^\top$ in place of $\Kss(\tau)$ yields
\begin{align}
\bm U_Q(T) - \bm U_Q(s)
&=
-\bm G_Q(T,s)\bm g(s),
\\
\bm G_Q(T,s)
&\triangleq
\int_s^T \Kqs(\tau)\mathcal P_g(\tau,s)d\tau.
\end{align}
Same integral, same propagator, different kernel: $\bm D_S$ records how training outputs moved, $\bm G_Q$ how test outputs moved.

\paragraph{Cumulative dissipation as a quadratic form.} Substituting $\bm g(\tau) = \mathcal P_g(\tau,s)\bm g(s)$ into the instantaneous dissipation rate from \autoref{thm:output_dynamics}, integrating over $[s,T]$, and pulling out $\bm g(s)$,
\begin{align}
\Phi_S(\bm u(s)) - \Phi_S(\bm u(T))
&=
\int_s^T \bm g(s)^\top \mathcal P_g(\tau,s)^\top \Kss(\tau)\mathcal P_g(\tau,s)\bm g(s)d\tau
\\
&=
\bm g(s)^\top \mathcal W_S(s,T)\bm g(s).
\end{align}
Total loss dissipated over $[s,T]$ is therefore the quadratic form of $\mathcal W_S$ evaluated at the initial gradient. For an arbitrary direction $\bm h \in \mathbb R^{np}$, factoring $\Kss(\tau) = \bm J_S(\tau)\bm J_S(\tau)^\top$ rewrites the quadratic form as a squared-norm integrand,
\begin{align}\label{eq:W_S_squared_norm_app}
\bm h^\top \mathcal W_S(s,T)\bm h
=
\int_s^T \|\bm J_S(\tau)^\top \mathcal P_g(\tau,s)\bm h\|_2^2 d\tau
\ge
0,
\end{align}
which gives $\mathcal W_S \succeq 0$ from the integrand and reads $\bm h^\top \mathcal W_S \bm h$ as the dissipation that direction $\bm h$ \emph{would have} experienced.

\paragraph{Reading off the eigenvalues.} The dissipation Gramian is symmetric PSD, so it has orthonormal eigenvectors $\bm\psi_j$ with eigenvalues $\lambda_j \ge 0$, $\mathcal W_S\bm\psi_j = \lambda_j \bm\psi_j$. Left-multiplying by $\bm\psi_j^\top$ and using $\bm\psi_j^\top \bm\psi_j = 1$, then specializing \eqref{eq:W_S_squared_norm_app} to $\bm h = \bm\psi_j$,
\begin{align}
\lambda_j
=
\bm\psi_j^\top \mathcal W_S(s,T)\bm\psi_j
=
\int_s^T \|\bm J_S(\tau)^\top \mathcal P_g(\tau,s)\bm\psi_j\|_2^2 d\tau.
\end{align}
The eigenvalue along $\bm\psi_j$ equals the total integrated squared reachability of $\bm\psi_j$ over the window. A direction $\bm h$ might align with a large eigenvalue of $\Kss(\tau)$ at one time and a small eigenvalue at another, so the instantaneous spectrum does not tell the whole story; the cumulative integral $\mathcal W_S$ records the total integrated effect over the entire training window.

\subsection{Scaled Tangent Kernel and Timescales}\label{sec:tangent_kernel_timescales}

\paragraph{Intuition.} The decay rate of the loss gradient is set by the kernel weighted by the curvature of the loss. We make this precise by introducing a scaled tangent kernel whose eigenvalues read off as relaxation rates. Directions with the largest eigenvalues are fit quickly; the smallest eigenvalues control the slowest modes, and the gap between them sets training timescales.

\autoref{thm:output_dynamics} shows that $\bm g$ decays according to $\partial_t \bm g = -\bm B\Kss \bm g$. To read off timescales, scale the kernel by the loss curvature:
\begin{align}
\widetilde K_B&=\bm B^{1/2}\Kss \bm B^{1/2}\quad\text{(general losses),} &
\bm M(t)&\triangleq \Kss(\bm w(t))/n\quad\text{(squared loss),}
\end{align}
with eigenvalues
\begin{align}\label{eq:mobility}
\lambda_1(t) \ge \lambda_2(t) \ge \cdots \ge \lambda_{np}(t) \ge 0
\end{align}
and orthonormal eigenvectors $\{\bm v_i(t)\}$. For squared loss, $\bm g=\bm r/n$ and $\bm B=\bm I/n$, so the residual obeys $\partial_t \bm r=-\bm M(t)\bm r$ and its component along $\bm v_i(t)$ decays at rate $\lambda_i(t)$. In practice the spectrum has extreme separation: the condition number among nonzero eigenvalues is typically enormous, creating a sharp split between fast- and slow-decaying directions.

\begin{proof}[Proof of \autoref{thm:output_dynamics} and \autoref{cor:test_trajectory}]
The chain rule applied to $L_S = \Phi_S \circ \bm U_S$ gives
\begin{align}
\partial_t \bm u &= \bm J_S(\bm w)\partial_t \bm w = -\bm J_S \bm J_S^\top \bm g = -\Kss\bm g, \\
\partial_t \bm g &= \nabla^2\Phi_S(\bm u(t))\partial_t \bm u = -\bm B(t)\Kss(t)\bm g(t), \\
\partial_t \bm w &= -\bm J_S(\bm w)^\top \bm g,
\intertext{proving \eqref{eq:output_dynamics_dynamics}. Dissipation follows by the same substitution:}
\frac{d}{dt}\Phi_S(\bm u(t))
&= \langle \bm g(t),\partial_t \bm u(t)\rangle
= -\bm g(t)^\top \Kss(t)\bm g(t)
= -\|\bm J_S^\top \bm g\|_2^2
= -\|\partial_t \bm w\|_2^2,
\end{align}
which proves \eqref{eq:output_dynamics_dissipation}. \autoref{eq:force_propagator} is the linear ODE representation of $\partial_t\bm g = -\bm B\Kss\bm g$. For the test trajectory and the propagator complement,
\begin{align}
\partial_t \bm U_Q
&=
\bm J_Q(\bm w)\partial_t \bm w
=
-\Kqs\bm g,
\\
\frac{d}{dT}\mathsf{F}_{SS}(T)
&=
\bm B(T)\Kss(T)\mathcal{P}_g(T,0),
\\
\frac{d}{dT}\mathcal{P}_g(T,0)
&=
-\bm B(T)\Kss(T)\mathcal{P}_g(T,0).
\end{align}
Integrating the first display together with \eqref{eq:force_propagator} gives \eqref{eq:test_trajectory}. The second and third lines show that $\mathsf{F}_{SS}(T)+\mathcal{P}_g(T,0)$ has zero $T$-derivative; at $T=0$ it equals $\bm I$, proving \eqref{eq:train_complement}.

For \eqref{eq:bregman_loss_gap}, applying the Fenchel--Young equality $\Phi_S(\bm u)+\Phi_S^*(\bm g)=\langle \bm u,\bm g\rangle$ (valid when $\bm g=\nabla \Phi_S(\bm u)$) and $\Phi_S^*(0)=-\Phi_S(\bar{\bm u})$:
\begin{align}
D_{\Phi_S^*}(0,\bm g)
&= \Phi_S^*(0)-\Phi_S^*(\bm g)-\langle \nabla \Phi_S^*(\bm g),-\bm g\rangle \\
&= \Phi_S^*(0)-\Phi_S^*(\bm g)+\langle \bm u,\bm g\rangle \\
&= \Phi_S^*(0)+\Phi_S(\bm u) = \Phi_S(\bm u)-\Phi_S(\bar{\bm u}).
\end{align}
\end{proof}

\begin{proof}[Proof of \autoref{thm:reservoir_invisibility}]
For $\bm h\in\ker(\bm W)$, positive semidefiniteness gives
\begin{align}
0
=
\bm h^\top \bm W \bm h
=
\int_s^T \bigl\|\bm J_S(\tau)^\top \mathcal P_g(\tau,s)\bm h\bigr\|_2^2 d\tau,
\end{align}
so $\mathcal P_g(\tau,s)\bm h\in\ker \bm J_S(\tau)^\top = \ker \Kss(\tau)$ for a.e.\ $\tau$. Since $\Kqs(\tau)=\bm J_Q(\tau)\bm J_S(\tau)^\top$ annihilates $\ker \Kss(\tau)$,
\begin{align}
\bm G\bm h
=
\int_s^T \Kqs(\tau)\mathcal P_g(\tau,s)\bm h d\tau
=0,
\end{align}
hence $\ker(\bm W)\subseteq \ker(\bm G)$.

Since $\bm G$ vanishes on $\ker(\bm W)$,
\begin{align}
\bm G&=\bm G \bm W\bm W^{\dagger}=\bm G \psinv{\bm W}\bm W^{1/2},
\\
\bm G^\top \bm G
&=
\bm W^{1/2}\bigl(\psinv{\bm W}\bm G^\top \bm G \psinv{\bm W}\bigr)\bm W^{1/2}
=
\bm W^{1/2}\Gamma_Q(s,T)\bm W^{1/2}.
\end{align}
Since $\Gamma_Q(s,T)\succeq 0$ and is supported on $\range(\bm W)$,
\begin{align}
0\preceq \Gamma_Q(s,T)\preceq \|\Gamma_Q(s,T)\|_{\op} \bm W\bm W^\dagger.
\end{align}

For any bounded function $\varphi$, since $\varphi(\bm W)$ commutes with $\bm W^{1/2}$,
\begin{align}\label{eq:pathwise_filtered_bound}
\|\bm G\varphi(\bm W)\bm h\|_2^2
&=
\langle \varphi(\bm W)\bm h,\bm G^\top \bm G\varphi(\bm W)\bm h\rangle
=
\Big\langle \bm W^{1/2}\varphi(\bm W)\bm h,\Gamma_Q(s,T)\bm W^{1/2}\varphi(\bm W)\bm h\Big\rangle
\nonumber\\
&\le
\|\Gamma_Q(s,T)\|_{\op}\|\bm W^{1/2}\varphi(\bm W)\bm h\|_2^2.
\end{align}
The kernel inclusion and the equality $\mathsf G_Q=\mathsf C_Q\mathcal W_S^{1/2}$ follow from $\bm G=\bm G\bm W\bm W^\dagger=\bm G\psinv{\bm W}\bm W^{1/2}$.
\end{proof}

\begin{proof}[Proof of \autoref{cor:low_dissipation}]
Apply \autoref{thm:reservoir_invisibility} with $\varphi(\lambda)=\mathbf 1_{\{0\}}(\lambda)$ and $\varphi(\lambda)=\mathbf 1_{[0,\varepsilon]}(\lambda)$; the mobility lower bound is the contrapositive of \eqref{eq:pathwise_filtered_bound} with $\varphi\equiv 1$.
\end{proof}

\subsection{Deferred Corollaries from Section~\ref{sec:second_law}}\label{sec:dynamics_corollaries}

\paragraph{Intuition.} This subsection records two consequences of the output-dynamics theorem that we reuse later. The first writes the test trajectory through the loss-gradient propagator and pairs it with a complementary relation that splits the cumulative training operator into two pieces. The second translates a small-dissipation direction into a quantitative bound on test displacement.

For a convex differentiable $\psi$, the Bregman divergence is
\begin{align}
D_{\psi}(a,b)\triangleq \psi(a)-\psi(b)-\langle \nabla \psi(b),a-b\rangle.
\end{align}

The first corollary expresses the test trajectory through the propagator and yields the complementary relation $\mathsf F_{SS}+\mathcal P_g=\bm I$ used throughout the proofs.

\begin{corollary}[Test Trajectory and Propagator Complement]\label{cor:test_trajectory}
For any test set $Q$,
\begin{align}\label{eq:test_trajectory}
\bm U_Q(t)
=
\bm U_Q(0)
-
\int_0^t \Kqs(\tau)\mathcal{P}_g(\tau,0)\bm g(0) d\tau.
\end{align}
Define the train-on-train cumulative output operator
\begin{align}\label{eq:force_completion_operator}
\mathsf{F}_{SS}(T)
\triangleq
\int_0^T \bm B(\tau)\Kss(\tau)\mathcal{P}_g(\tau,0) d\tau.
\end{align}
Then the complement holds:
\begin{align}\label{eq:train_complement}
\mathsf{F}_{SS}(T)+\mathcal{P}_g(T,0)=\bm I.
\end{align}
If $\Phi_S^*$ is differentiable at $\bm g(t)$ and $\bar{\bm u}$ minimizes $\Phi_S$ with
$\nabla \Phi_S(\bar{\bm u})=0$, then
\begin{align}\label{eq:bregman_loss_gap}
\Phi_S(\bm u(t))-\Phi_S(\bar{\bm u})=D_{\Phi_S^*}(0,\bm g(t)).
\end{align}
\end{corollary}

The second corollary makes the small-dissipation principle quantitative: whenever $\mathcal W_S$ is small along a direction, $\bm G$ is small there too, with the same constant.

\begin{corollary}[Low-Loss-dissipation Bound]\label{cor:low_dissipation}
Under the notation of \autoref{def:test_transfer},
\begin{align}
\bm GP_{\{0\}}(s,T)&=0, \label{eq:pathwise_hard_reservoir}\\
\|\bm GP_{[0,\varepsilon]}(s,T)\bm h\|_2
&\le
\|\Gamma_Q(s,T)\|_{\op}^{1/2}\sqrt{\varepsilon}\|P_{[0,\varepsilon]}(s,T)\bm h\|_2. \label{eq:pathwise_soft_reservoir}
\end{align}
Equivalently, every direction producing terminal test displacement at least $\tau$ obeys
\begin{align}\label{eq:pathwise_mobility_lower_bound}
\|\bm G\bm h\|_2\ge \tau
\Longrightarrow
\bm h^\top \bm W \bm h\ge \tau^2/\|\Gamma_Q(s,T)\|_{\op}.
\end{align}
\end{corollary}

\section{Minibatch Drift--Diffusion: Proof}\label{app:proof_minibatch_coherence}

\paragraph{Intuition.} SGD's effect on the test outputs splits into two pieces of very different sizes. The mean (drift) part of the gradient produces a contribution that, in the worst case, accumulates linearly in $T$ along directions the optimizer can actually see; the zero-mean fluctuation (diffusion) part is a martingale and, by orthogonality of its increments, only adds up like $\sqrt{T}$. Signal directions, where the drift is nonzero, dominate over time, while pure-noise directions are washed out as $\sqrt{T}/T\to 0$. The proof is a Taylor expansion of the test prediction across one optimizer step, plus a martingale variance computation, plus a replace-two argument needed to handle the dependence of the iterates on the training set.

We prove that on a signal direction SGD's drift accumulates linearly in $T$, while on a noise direction it grows only like $\sqrt{T}$.

\paragraph{Proof of \autoref{thm:minibatch_coherence}.}
Taylor-expanding the test prediction over a single SGD step,
\begin{align}
\bm U_Q(\bm w_{k+1})-\bm U_Q(\bm w_k)
&=
\bm J_Q(\bm w_k)(\bm w_{k+1}-\bm w_k)+\bm r_{Q,k},
\\
\|\bm r_{Q,k}\|_2 &\le \tfrac{\beta_Q}{2}\|\bm w_{k+1}-\bm w_k\|_2^2.
\end{align}
Substituting $\bm w_{k+1}-\bm w_k=-\eta \bm M_k(\bm\mu_k+\bm\xi_k)$ and summing yields \eqref{eq:drift_diffusion_decomposition}. Since $\bm L_{Q,k},\bm\mu_k$ are $\mathcal F_k$-measurable and $\mathbb E[\bm\xi_k\mid\mathcal F_k]=\bm 0$, the increments $\eta \bm L_{Q,k}\bm\xi_k$ are martingale differences, so
\begin{align}
\mathbb E\Bigl\|\eta\sum_{k=0}^{N-1}\bm L_{Q,k}\bm\xi_k\Bigr\|_2^2
&=
\eta^2\sum_{k=0}^{N-1}\mathbb E \operatorname{tr}\bigl(\bm L_{Q,k}\bm\Sigma_k^{(b)}\bm L_{Q,k}^\top\bigr)
\\
&\le \frac{\eta T}{b}\bar V_T,
\end{align}
where the bound uses
\begin{align}
\operatorname{tr}(\bm L_{Q,k}\bm\Sigma_k^{(b)}\bm L_{Q,k}^\top)
\le
\frac{V_k}{b},
\qquad
\bar V_T=\tfrac{1}{T}\sum_k\eta V_k,
\end{align}
and the drift bound is the deterministic estimate
\begin{align}
\sum_k\eta\|\bm L_{Q,k}\bm\mu_k\|_2
\le
T\sup_k\|\bm L_{Q,k}\bm\mu_k\|_2.
\end{align}

For the empirical noise mean, the summands of
\begin{align}
\bm\Pi\bm\mu_k
=
\tfrac{1}{n}\sum_i\bm\Pi\nabla_{\bm w}\ell(\bm w_k(S),Z_i)
\end{align}
are not independent because $\bm w_k(S)$ depends on every $Z_i$, so the cross terms must be decoupled by replacing the two examples that appear in each one. Let $S^{(ij)}$ be the dataset obtained from $S$ by replacing $Z_i,Z_j$ with independent fresh draws from $\mathcal D$, write $\bm w_k^{(ij)}\triangleq\bm w_k(S^{(ij)})$, and assume the bounded second moment $V_k$ on the projected gradient together with the replace-two defect
\begin{align}\label{eq:replace_two_stability}
\varepsilon_{k,n}
&\triangleq
\sup_{i\ne j}\Bigl(\mathbb E\bigl\|\bm\Pi\nabla_{\bm w}\ell(\bm w_k(S),Z_i)
\nonumber\\
&\qquad-\bm\Pi\nabla_{\bm w}\ell(\bm w_k^{(ij)},Z_i)\bigr\|_2^2\Bigr)^{1/2}.
\end{align}
Then $\bm w_k^{(ij)}$ is independent of $(Z_i,Z_j)$, so fresh-sample centering gives
\begin{align}
\mathbb E\langle\bm\Pi\nabla_{\bm w}\ell(\bm w_k^{(ij)},Z_i),\bm\Pi\nabla_{\bm w}\ell(\bm w_k^{(ij)},Z_j)\rangle=0.
\end{align}
Writing
\begin{align}
\bm\Pi\nabla_{\bm w}\ell(\bm w_k(S),Z_i)
=
\bm\Pi\nabla_{\bm w}\ell(\bm w_k^{(ij)},Z_i)
+
\bm r_i^{(ij)}
\end{align}
and expanding the cross inner product, Cauchy--Schwarz with \eqref{eq:replace_two_stability} bounds the three residual terms by $\sqrt{V_k}\varepsilon_{k,n}$, $\sqrt{V_k}\varepsilon_{k,n}$, and $\varepsilon_{k,n}^2$, so summing diagonal and off-diagonal contributions yields
\begin{align}\label{eq:exchange_noise_mean}
\mathbb E\|\bm\Pi\bm\mu_k\|_2^2 &\le \tfrac{V_k}{n}+2\sqrt{V_k}\varepsilon_{k,n}+\varepsilon_{k,n}^2.
\end{align}
Under $\varepsilon_{k,n}=O(1/n)$ this is $O(1/n)$, and the drift contribution becomes $T^2\bar V_T/n$.\qed

\paragraph{Multi-epoch caveat.}
The one-step LOO result in \autoref{thm:kernel_increment} requires $B$ to be independent of the optimizer history $\mathcal F_t$. In streaming or online training this holds by construction; in multi-epoch training a replayed batch carries information about $\bm w_t$ and the equality acquires a residual,
\begin{align}
\mathbb E[\mathcal R^{\eta}_{\mathrm{1ex},B_t}\mid\mathcal F_t]
=
\mathcal R^{\eta}_{\mathrm{fresh}}(\bm w_t,\bm P_t)+r_t,
\qquad
|r_t|\le 2M\mathrm{TV}(P_{B_t\mid\mathcal F_t},\mathcal D^b),
\end{align}
where $\mathcal R^{\eta}_{\mathrm{fresh}}$ is the fresh-boundary one-step risk, $r_t=\int\varphi_t(B) d(P_{B_t\mid\mathcal F_t}-\mathcal D^b)$ for a bounded test function $|\varphi_t|\le M$, and $r_t$ vanishes in the first epoch.

\section{Train-Test Coupling: Proofs}\label{sec:proofs_coupling}

\paragraph{Intuition.} This appendix proves the train-test coupling results of \autoref{sec:third_law}. The basic question is when motion on a held-out point can be reconstructed from motion observed on the training set. Classical NTK analysis assumes the kernel is frozen at initialization \citep{jacot2018neural,du2019gradient}; here the kernel evolves with the parameters and can move by $\mathcal{O}(1)$ in operator norm over a typical run. We ask when the test displacement is predictable from the observed training displacement, and what quantity controls the prediction error in the feature-learning regime. The bounds we derive are computable from a single run.

\subsection{Feature Learning with Stable Transfer}\label{sec:feature_learning}

\paragraph{Intuition.} This subsection sets up the proof of \autoref{thm:train_test_coupling}. The dissipation Gramian $\mathcal W_S(s,T)$ and the transfer operator $\mathsf G_Q(T,s)$ from \autoref{thm:reservoir_invisibility} already absorb kernel motion, so the analysis applies in full feature learning. We work directly with these pathwise objects and show that a stable train-test relationship survives even when each kernel moves by $\mathcal{O}(1)$.

The stability condition holds for the relationship between training and test kernels even when the individual kernels move by $\mathcal{O}(1)$ in operator norm. The next subsection proves the train-test coupling theorem under this condition. Subsequent subsections develop the path-error decomposition and the test-prediction bounds in \autoref{sec:exact_effective_closure}.

\subsection{Proof of the Train-Test Coupling Theorem}\label{sec:coupling_proof}

\paragraph{Intuition.} The proof of \autoref{thm:train_test_coupling} rests on a simple observation: once both the training and test displacement operators are normalized by the dissipation Gramian, the test side decomposes orthogonally into a piece predictable from the training side plus an orthogonal remainder. This is the operator analogue of regressing one Gaussian variable on another. Fix $0\le s\le T$ and a test set $Q$, and write $\bm D\triangleq \mathsf D_S(T,s)$, $\bm G\triangleq \mathsf G_Q(T,s)$, $\bm W\triangleq \mathcal W_S(s,T)$.

\paragraph{Per-eigendirection reading of the normalization.} On each eigenvector $\bm u_i$ of $\bm W$ with eigenvalue $\sigma_i > 0$,
\begin{align}
\mathsf C_S\bm u_i
&=
\frac{\bm D\bm u_i}{\sqrt{\sigma_i}},
&
\mathsf C_Q\bm u_i
&=
\frac{\bm G\bm u_i}{\sqrt{\sigma_i}}.
\end{align}
Both channels measure displacement along $\bm u_i$ divided by the square root of how much training dissipated along $\bm u_i$, putting train displacement and test transfer in the same per-root-dissipation units. In these units the optimal linear predictor and irreducible remainder of \autoref{def:normalized_channels},
\begin{align}
\bm A_\circ
=
\mathsf C_Q\mathsf C_S^\dagger,
\qquad
\bm R_\perp
=
\mathsf C_Q(\bm I - P_{\mathrm{tr}}),
\end{align}
read as the regression of test-transfer on training-displacement along the eigenbasis of $\bm W$, plus the part of test transfer that cannot be predicted from training displacement.

\begin{definition}[Dissipation-normalized channels, optimal linear predictor, and irreducible remainder]\label{def:normalized_channels}
The \emph{dissipation-normalized train channel} and \emph{dissipation-normalized test channel} are
\begin{align}
\mathsf C_S(s,T)&\triangleq \bm D\psinv{\bm W}, &
\mathsf C_Q(s,T)&\triangleq \bm G\psinv{\bm W}.
\end{align}
The \emph{transfer projector} $P_{\mathrm{tr}}\triangleq \mathsf C_S^\dagger\mathsf C_S$ is the orthogonal projector onto the column space of $\mathsf C_S$.
The \emph{optimal linear predictor} and \emph{irreducible remainder} are
\begin{align}
\bm A_\circ&\triangleq \mathsf C_Q\mathsf C_S^\dagger
=\bm G \bm W^\dagger \bm D^\top\bigl(\bm D\bm W^\dagger \bm D^\top\bigr)^\dagger,\\
\bm R_\perp&\triangleq \mathsf C_Q\bigl(\bm I-P_{\mathrm{tr}}\bigr).
\end{align}
For an arbitrary predictor $\Psi:\range(\bm D)\to\mathbb R^{|Q|p}$, the \emph{normalized worst-case error} is
\begin{align}
\operatorname{Err}(\Psi)
\triangleq
\sup_{\bm h\notin\ker \bm W}
\frac{\|\bm G \bm h-\Psi(\bm D\bm h)\|_2}{\sqrt{\bm h^\top \bm W \bm h}}.
\end{align}
\end{definition}

Both $\bm D$ and $\bm G$ pass through $\bm W^{1/2}$, and the part of the test channel orthogonal to the train channel becomes $\bm R_\perp$.

\begin{theorem}[Train-Test Coupling: Factorization]
\label{thm:coupling_factorization}
Adopt the notation of \autoref{def:normalized_channels}. The train and test displacements share the same right-hand dissipation factor, and the dissipation-normalized test channel decomposes orthogonally into a part driven by training and an irreducible remainder:
\begin{align}
\bm D&=\mathsf C_S(s,T)\bm W^{1/2},
&
\bm G&=\mathsf C_Q(s,T)\bm W^{1/2},
\label{eq:latent_channel_factorization}
\\
\mathsf C_Q(s,T)&=\bm A_\circ\mathsf C_S(s,T)+\bm R_\perp,
&
\bm R_\perp\mathsf C_S(s,T)^\top&=0.
\label{eq:canonical_screen_decomposition}
\end{align}
\end{theorem}

Every linear predictor's error splits as $\bm R_\perp\bm R_\perp^\top$ plus a quadratic penalty in $\bm A-\bm A_\circ$, so $\bm A_\circ$ minimizes the error in the positive-semidefinite order.

\begin{theorem}[Transfer error decomposition]\label{thm:transfer_error_decomp}
Under the hypotheses of \autoref{thm:coupling_factorization}, every linear predictor $\bm A$ satisfies
\begin{align}\label{eq:path_error_decomp}
(\bm G-\bm A\bm D)\bm W^\dagger(\bm G-\bm A\bm D)^\top
=
\bm R_\perp \bm R_\perp^\top
+
(\bm A-\bm A_\circ)\bm D\bm W^\dagger \bm D^\top(\bm A-\bm A_\circ)^\top
\succeq \bm R_\perp \bm R_\perp^\top,
\end{align}
so $\bm A_\circ$ is the unique linear predictor minimizing the path error in the positive-semidefinite order.
\end{theorem}

Allowing the predictor to be nonlinear yields the same minimum error, achieved by the linear $\bm A_\circ$.

\begin{corollary}[Nonlinear optimality]\label{cor:linear_predictor_optimal}
The optimal linear predictor matches the best nonlinear post-processing of the training displacement:
\begin{align}\label{eq:nonlinear_gap_identity}
\inf_{\Psi}\operatorname{Err}(\Psi)
=
\|\bm R_\perp\|_{\op}
=
\sup_{\substack{\bm h\notin\ker \bm W\\ \bm D\bm h=0}}
\frac{\|\bm G\bm h\|_2}{\sqrt{\bm h^\top \bm W \bm h}}
=
\inf_{\bm A} \|(\bm G-\bm A\bm D)\psinv{\bm W}\|_{\op}.
\end{align}
\end{corollary}

\begin{remark}[Equivalent formulations]\label{rem:coupling_equivalent}
In terms of the original operators,
\begin{align}\label{eq:latent_displacement}
\bm u(T)-\bm u(s)=-\bm D\bm g(s),
\qquad
\bm U_Q(T)-\bm U_Q(s)=-\bm G\bm g(s),
\end{align}
\begin{align}\label{eq:canonical_screen_decomposition_original}
\bm G=\bm A_\circ \bm D+\bm R_\perp \bm W^{1/2},
\qquad
(\bm G-\bm A_\circ \bm D)\bm W^\dagger \bm D^\top=0.
\end{align}
On the true trajectory,
\begin{align}\label{eq:coupling_trajectory}
\bm U_Q(T)-\bm U_Q(s)
=
\bm A_\circ\bigl(\bm u(T)-\bm u(s)\bigr)-\bm R_\perp \bm W^{1/2}\bm g(s).
\end{align}
\end{remark}

\begin{proof}[Proof of the four results above]
By \autoref{thm:reservoir_invisibility}, $\ker \bm W\subseteq \ker \bm D\cap\ker \bm G$, and since $\psinv{\bm W}\bm W^{1/2}=P_W$ is the orthogonal projector onto $\range(\bm W)$,
\begin{align}
\bm D=\bm DP_W=\bm D\psinv{\bm W}\bm W^{1/2}=\mathsf C_S(s,T)\bm W^{1/2},
\\
\bm G=\bm GP_W=\bm G\psinv{\bm W}\bm W^{1/2}=\mathsf C_Q(s,T)\bm W^{1/2},
\end{align}
which is \eqref{eq:latent_channel_factorization}; \eqref{eq:latent_displacement}
then follows from the definitions of $\bm D$ and $\bm G$.

Since $P_{\mathrm{tr}}=\mathsf C_S^\dagger\mathsf C_S$ is the orthogonal projector onto
$\range(\mathsf C_S^\top)$,
\begin{align}
\mathsf C_Q
=
\mathsf C_QP_{\mathrm{tr}}+\mathsf C_Q(\bm I-P_{\mathrm{tr}})
=
\bm A_\circ\mathsf C_S+\bm R_\perp.
\end{align}
Also,
\begin{align}
\bm R_\perp\mathsf C_S^\top
=
\mathsf C_Q(\bm I-P_{\mathrm{tr}})\mathsf C_S^\top
=0,
\end{align}
because $\range(\mathsf C_S^\top)\subseteq \range(P_{\mathrm{tr}})$.
This proves \eqref{eq:canonical_screen_decomposition}. Multiplying by $\bm W^{1/2}$ gives
\eqref{eq:canonical_screen_decomposition_original}, and then
\eqref{eq:coupling_trajectory} follows from
\eqref{eq:latent_displacement}. The explicit formula for $\bm A_\circ$ is the Moore--Penrose
pseudoinverse $\bm M^\dagger=\bm M^\top(\bm M\bm M^\top)^\dagger$ applied to
$\bm M=\mathsf C_S=\bm D\psinv{\bm W}$.

For any linear $\bm A$, expanding the residual and using $\bm R_\perp\mathsf C_S^\top=0$ gives
\begin{align}
\mathsf C_Q-\bm A\mathsf C_S
&=
\bm R_\perp+(\bm A_\circ-\bm A)\mathsf C_S,
\\
(\mathsf C_Q-\bm A\mathsf C_S)(\mathsf C_Q-\bm A\mathsf C_S)^\top
&=
\bm R_\perp \bm R_\perp^\top
+
(\bm A-\bm A_\circ)\mathsf C_S\mathsf C_S^\top(\bm A-\bm A_\circ)^\top.
\end{align}
With $\mathsf C_Q-\bm A\mathsf C_S=(\bm G-\bm A\bm D)\psinv{\bm W}$ and
$\mathsf C_S\mathsf C_S^\top=\bm D\bm W^\dagger \bm D^\top$, this is
\eqref{eq:path_error_decomp}, and the operator and Frobenius minimizers follow.

For \eqref{eq:nonlinear_gap_identity}, let
\begin{align}
\operatorname{Err}(\Psi)
\triangleq
\sup_{\bm h\notin\ker \bm W}
\frac{\|\bm G\bm h-\Psi(\bm D\bm h)\|_2}{\sqrt{\bm h^\top \bm W \bm h}}.
\end{align}
If $\bm D\bm h=0$, then also $\bm D(-\bm h)=0$. Thus for any predictor $\Psi$,
\begin{align}
\operatorname{Err}(\Psi)
\ge
\max\left\{
\frac{\|\bm G\bm h-\Psi(0)\|_2}{\sqrt{\bm h^\top \bm W \bm h}},
\frac{\|-\bm G\bm h-\Psi(0)\|_2}{\sqrt{\bm h^\top \bm W \bm h}}
\right\}
\ge
\frac{\|\bm G\bm h\|_2}{\sqrt{\bm h^\top \bm W \bm h}}.
\end{align}
Taking the supremum over $\bm D\bm h=0$ gives
\begin{align}
\operatorname{Err}(\Psi)
\ge
\sup_{\substack{\bm h\notin\ker \bm W\\ \bm D\bm h=0}}
\frac{\|\bm G\bm h\|_2}{\sqrt{\bm h^\top \bm W \bm h}}.
\end{align}
Writing $z=\bm W^{1/2}\bm h$, we have $\bm D\bm h=0 \iff \mathsf C_S z=0$,
$\|\bm G\bm h\|_2=\|\mathsf C_Q z\|_2$, and $\bm h^\top \bm W \bm h=\|z\|_2^2$. Hence
\begin{align}
\sup_{\substack{\bm h\notin\ker \bm W\\ \bm D\bm h=0}}
\frac{\|\bm G\bm h\|_2}{\sqrt{\bm h^\top \bm W \bm h}}
=
\sup_{z\in\ker\mathsf C_S\setminus\{0\}}
\frac{\|\mathsf C_Q z\|_2}{\|z\|_2}
=
\|\mathsf C_Q(\bm I-P_{\mathrm{tr}})\|_{\op}
=
\|\bm R_\perp\|_{\op}.
\end{align}
On the other hand, the linear predictor $x\mapsto \bm A_\circ x$ satisfies
\begin{align}
\|(\bm G-\bm A_\circ \bm D)\psinv{\bm W}\|_{\op}
=
\|\bm R_\perp\|_{\op},
\end{align}
so both infima in \eqref{eq:nonlinear_gap_identity} equal $\|\bm R_\perp\|_{\op}$.
\end{proof}
The fixed-transfer problem becomes an operator regression: $\bm A_\circ \bm D$ captures every test component determined by the observed training displacement, while $\bm R_\perp \bm W^{1/2}\bm g(s)$ is the irreducible remainder. Each anchored predictor $\bm A$ pays the positive-semidefinite misspecification penalty
\begin{align}
(\bm A-\bm A_\circ)\bm D\bm W^\dagger \bm D^\top(\bm A-\bm A_\circ)^\top
\end{align}
on top of the irreducible $\bm R_\perp \bm R_\perp^\top$.

The next corollary collects equivalent characterizations of $\bm A_\circ$ that we use later and gives a trajectory bound on the realized path.

\begin{corollary}[Exact transfer equivalences and trajectory bounds]\label{cor:transfer_equivalences}
Under the hypotheses of \autoref{thm:coupling_factorization}, the following are equivalent:
\begin{align}\label{eq:exact_screen_equivalences}
\bm R_\perp=0
&\iff
\ker \bm D\subseteq \ker \bm G
\iff
\range(\bm G^\top)\subseteq \range(\bm D^\top)
\\
&\iff
\bm G^\top \bm G\preceq \lambda \bm D^\top \bm D\textup{ for some }\lambda\ge 0.
\end{align}
When these hold, the predictor on $\range(\bm D)$ is unique, necessarily linear, and given by $\Psi_\sharp(x)=\bm G \bm D^\dagger x$ with Lipschitz norm $\Sigma_Q(s,T)=\|\bm G \bm D^\dagger\|_{\op}$.

On the true trajectory,
\begin{align}\label{eq:latent_screen_bound}
\|\bm U_Q(T)-\bm U_Q(s)-\bm A_\circ\bigl(\bm u(T)-\bm u(s)\bigr)\|_2
&\le
\|\bm R_\perp\|_{\op}
\sqrt{\Phi_S(\bm u(s))-\Phi_S(\bm u(T))}.
\end{align}
More generally, for any linear predictor $\bm A$,
\begin{align}\label{eq:general_screen_bound}
\|\bm U_Q(T)-\bm U_Q(s)-\bm A\bigl(\bm u(T)-\bm u(s)\bigr)\|_2
&\le
\|\mathsf C_Q(s,T)-\bm A\mathsf C_S(s,T)\|_{\op}
\\
&\qquad\cdot
\sqrt{\Phi_S(\bm u(s))-\Phi_S(\bm u(T))}.
\end{align}
\end{corollary}

\begin{proof}
Since $\ker P_{\mathrm{tr}}=\ker \mathsf C_S$,
\begin{align}
\bm R_\perp=0
\iff
\mathsf C_Q(\bm I-P_{\mathrm{tr}})=0
\iff
\ker \mathsf C_S\subseteq \ker \mathsf C_Q.
\end{align}
Using \eqref{eq:latent_channel_factorization} and the fact that both $\mathsf C_S$ and
$\mathsf C_Q$ vanish on $\ker \bm W$, this is equivalent to $\ker \bm D\subseteq \ker \bm G$.
The equivalence with $\range(\bm G^\top)\subseteq \range(\bm D^\top)$ is
finite-dimensional duality. For the positive-semidefinite condition,
$\bm G=\bm A\bm D$ gives $\bm G^\top \bm G\preceq \|\bm A\|_{\op}^2 \bm D^\top \bm D$, and conversely
$\bm G^\top \bm G\preceq \lambda \bm D^\top \bm D$ implies $\bm D\bm h=0\Rightarrow \bm G\bm h=0$.
When these hold, the predictor on $\range(\bm D)$ is $\bm G\bm D^\dagger$, unique by
well-definedness on the range, and its Lipschitz norm is $\|\bm G\bm D^\dagger\|_{\op}$.

\eqref{eq:latent_screen_bound} follows from
\eqref{eq:coupling_trajectory} and
$\bm g(s)^\top \bm W \bm g(s)=\Phi_S(\bm u(s))-\Phi_S(\bm u(T))$.
For the general linear predictor bound, use
\begin{align}
\bm U_Q(T)-\bm U_Q(s)-\bm A\bigl(\bm u(T)-\bm u(s)\bigr)
&= -(\bm G-\bm A\bm D)\bm g(s) \\
&= -(\mathsf C_Q(s,T)-\bm A\mathsf C_S(s,T))\bm W^{1/2}\bm g(s).
\end{align}
Therefore
\begin{align}
\|\bm U_Q(T)-\bm U_Q(s)-\bm A\bigl(\bm u(T)-\bm u(s)\bigr)\|_2
&\le
\|\mathsf C_Q(s,T)-\bm A\mathsf C_S(s,T)\|_{\op}
\|\bm W^{1/2}\bm g(s)\|_2
\\
&=
\|\mathsf C_Q(s,T)-\bm A\mathsf C_S(s,T)\|_{\op}
\\
&\qquad\cdot\sqrt{\Phi_S(\bm u(s))-\Phi_S(\bm u(T))},
\end{align}
which is \eqref{eq:general_screen_bound}.
\end{proof}

\subsection{Smoothness is Required for Test Prediction}\label{app:smoothness_required}

\paragraph{Intuition.} Predicting test motion from training data requires a smoothness hypothesis on the network. We exhibit two networks producing identical training trajectories, Jacobian paths, and training losses but different test displacements at a held-out point, so the smoothness assumption used in \autoref{sec:sobolev_remainder} is necessary.

\begin{theorem}[Smoothness is necessary]\label{thm:smoothness_required}
Fix $z_\star \notin S$. There exist two $C^\infty$ networks $F$ and $\widetilde F$ such that for every $z_i \in S$ and every $t \le T$,
\begin{align}
F(w(t),z_i)
&=
\widetilde F(w(t),z_i),
&
D_w F(w(t),z_i)
&=
D_w\widetilde F(w(t),z_i),
\end{align}
so the entire training trajectory, Jacobian path, kernel path, and training losses are identical, but
\begin{align}
F(w(T),z_\star)-F(w(0),z_\star)
\ne
\widetilde F(w(T),z_\star)-\widetilde F(w(0),z_\star).
\end{align}
\end{theorem}

\begin{proof}
Pick a $C^\infty$ bump $\psi$ on input space and a $C^\infty$ scalar $\phi$ on weight space with
\begin{align}
\psi(z_i) = 0,
\quad
\nabla_z\psi(z_i) = 0
\ \ \forall z_i\in S,
\qquad
\psi(z_\star) \ne 0,
\qquad
\phi(w(T)) - \phi(w(0)) \ne 0,
\end{align}
and define the perturbed network
\begin{align}
\widetilde F(w,z) = F(w,z) + \psi(z)\phi(w).
\end{align}
Then $\widetilde F$ agrees with $F$ to first order on $S$, so the training dynamics are identical, but the test displacement at $z_\star$ differs by $\psi(z_\star)\bigl[\phi(w(T))-\phi(w(0))\bigr]\ne 0$.
\end{proof}

\subsection{Sobolev Bound on the Transfer Remainder}\label{sec:sobolev_remainder}

\paragraph{Intuition.} When the data lies on a low-dimensional manifold and the network is sufficiently smooth, the irreducible remainder $\bm R_\perp$ shrinks with sample density. The argument routes through a single pathwise displacement field on the manifold: test motion is its value at unseen points, training motion is its value at the observed sample, and a Sobolev sampling inequality controls how much the second determines the first. The bound applies with $\mathcal{O}(1)$ kernel drift and quantifies the role of fill distance.

Let $\mathcal M$ be a compact $d_{\mathcal M}$-dimensional data manifold and
assume $F(w,\cdot)$ is $C^m$ on $\mathcal M$, with $m>d_{\mathcal M}/2$.
For $z\in\mathcal M$ write
\begin{align}
\bm J_z(t)=D_wF(w(t),z),\qquad \bm K_{zS}(t)=\bm J_z(t)\bm J_S(t)^\top .
\end{align}
Define the pathwise displacement field
\begin{align}
(\mathscr T_{s,T}\bm h)(z)
\triangleq
\int_s^T \bm K_{zS}(\tau)\mathcal P_g(\tau,s)\bm h d\tau.
\end{align}
Then
\begin{align}
\bm D=E_S\mathscr T_{s,T},\qquad \bm G=E_Q\mathscr T_{s,T},
\end{align}
where $E_Af=(f(a_1);\ldots;f(a_{|A|}))$ is evaluation on a finite set.

Define the pathwise Jacobian-Sobolev norm
\begin{align}
\mathfrak a_m(\tau)
&\triangleq
\sup_{\|\bm v\|_2=1}
\bigl\|z\mapsto \bm J_z(\tau)\bm v\bigr\|_{H^m(\mathcal M;\mathbb R^p)},
&
\Lambda_m(s,T)
&\triangleq
\left(\int_s^T \mathfrak a_m(\tau)^2 d\tau\right)^{1/2}.
\end{align}

Let $\bm W=\mathcal W_S(s,T)$ and let
\begin{align}
\mathcal C_{s,T}\triangleq \mathscr T_{s,T}\bm W^{\dagger/2}.
\end{align}

The first piece is an operator-norm bound on the dissipation-normalized displacement field. Then we combine it with a Sobolev sampling inequality to bound the irreducible remainder in terms of the fill distance.

\begin{lemma}[Operator-norm bound on the dissipation-normalized field]\label{lem:sobolev_C_bound}
\begin{align}
\|\mathcal C_{s,T}\|_{\op(\ell_2,H^m)}
\le
\Lambda_m(s,T).
\end{align}
\end{lemma}

\begin{theorem}[Pathwise Sobolev Transfer]\label{thm:sobolev_remainder}
If $S$ has fill distance
\begin{align}
h_S=\sup_{z\in\mathcal M}\min_{i\le n} d_{\mathcal M}(z,z_i),
\end{align}
then the irreducible train-test remainder satisfies
\begin{align}\label{eq:sobolev_Rperp}
\|\bm R_\perp\|_{\op}
&\le
C_{\mathcal M,m,p}|Q|^{1/2}
h_S^{m-d_{\mathcal M}/2}
\Lambda_m(s,T).
\end{align}
The bound holds with no frozen-kernel or small-kernel-drift condition.
\end{theorem}

The remainder bound translates directly into a deterministic bound on test displacement along the realized trajectory.

\begin{corollary}[Trajectory form of the Sobolev transfer bound]\label{cor:sobolev_trajectory}
\begin{align}\label{eq:sobolev_trajectory}
\|\bm U_Q(T)-\bm U_Q(s)-\bm A_\circ(\bm u(T)-\bm u(s))\|_2
&\le
C_{\mathcal M,m,p}|Q|^{1/2}
h_S^{m-d_{\mathcal M}/2}
\Lambda_m(s,T) \\
&\qquad\cdot
\sqrt{\Phi_S(\bm u(s))-\Phi_S(\bm u(T))}.
\end{align}
\end{corollary}

\begin{proof}[Proof of \autoref{lem:sobolev_C_bound}, \autoref{thm:sobolev_remainder}, and \autoref{cor:sobolev_trajectory}]
For $\bm h\in\mathbb R^{np}$ set $\bm v_h(\tau)=\bm J_S(\tau)^\top\mathcal P_g(\tau,s)\bm h$, so
\begin{align}
(\mathscr T_{s,T}\bm h)(z)
=
\int_s^T \bm J_z(\tau)\bm v_h(\tau) d\tau.
\end{align}
Cauchy--Schwarz then gives
\begin{align}
\|\mathscr T_{s,T}\bm h\|_{H^m}
\le
\int_s^T \mathfrak a_m(\tau)\|\bm v_h(\tau)\|_2 d\tau
\le
\Lambda_m(s,T)
\left(
\int_s^T
\|\bm J_S(\tau)^\top\mathcal P_g(\tau,s)\bm h\|_2^2
d\tau
\right)^{1/2}.
\end{align}
The final integral equals $\bm h^\top \bm W\bm h$, and substituting $\bm h=\bm W^{\dagger/2}\bm a$ gives
\begin{align}
\|\mathscr T_{s,T}\bm h\|_{H^m}
&\le
\Lambda_m(s,T)\sqrt{\bm h^\top \bm W\bm h},
\\
\|\mathcal C_{s,T}\bm a\|_{H^m}
&\le
\Lambda_m(s,T)\|\bm a\|_2 .
\end{align}

Set $\Pi_S=(E_S\mathcal C_{s,T})^\dagger(E_S\mathcal C_{s,T})$ and, for any unit vector $\bm a$, $f_{\bm a}=\mathcal C_{s,T}(\bm I-\Pi_S)\bm a$. Then $E_Sf_{\bm a}=0$ and
\begin{align}
\|f_{\bm a}\|_{H^m}
\le
\Lambda_m(s,T).
\end{align}
The Sobolev sampling inequality \citep{narcowich2006sobolev} for functions vanishing on an $h_S$-net then gives
\begin{align}
\|E_Qf_{\bm a}\|_2
&\le
C_{\mathcal M,m,p}|Q|^{1/2}
h_S^{m-d_{\mathcal M}/2}
\|f_{\bm a}\|_{H^m},
\\
\bm g(s)^\top \bm W\bm g(s)&=\Phi_S(\bm u(s))-\Phi_S(\bm u(T)).
\end{align}
Taking the supremum over $\|\bm a\|_2=1$ in the first display yields the bound for $\bm R_\perp$, and the second display delivers the trajectory bound.
\end{proof}

\begin{remark}[Dimensional limitation]
When $d_{\mathcal M}$ is large, the fill distance of $n$ points on a $d_{\mathcal M}$-dimensional manifold scales as $h_S\sim n^{-1/d_{\mathcal M}}$, so the bound becomes $n^{-(m/d_{\mathcal M}-1/2)}$. This is useful only when $m\gg d_{\mathcal M}$, i.e.\ when the network map is much smoother than the intrinsic dimension requires. The result is therefore most informative for low-dimensional data manifolds; in high-dimensional settings the operator bound from \autoref{thm:train_test_coupling} remains the operative guarantee.
\end{remark}

Combining fill distance with Sobolev sampling yields a population-risk bound from finite-sample training error.

\begin{theorem}[Pathwise Sobolev Generalization]\label{thm:sobolev_generalization}
Assume squared loss and let $\bm y:\mathcal M\to\mathbb R^p$ be the population target.
Let
\begin{align}
\bm r_T(z)=F(w(T),z)-\bm y(z).
\end{align}
Let $\rho$ be the population law on $\mathcal M$, and let $V_i$ be the Voronoi cell of
$z_i$ under the manifold metric, with $\mu_i=\rho(V_i)$ and
$\mu_{\max}=\max_i\mu_i$. Then
\begin{align}\label{eq:sobolev_generalization}
\|\bm r_T\|_{L^2(\rho)}
&\le
\sqrt{n\mu_{\max}}
\|\bm r_T\|_{\ell_2(S)}
\\
&\quad+
C_{\mathcal M,m,p}
h_S^m
\left(
\|F(w(s),\cdot)-\bm y\|_{H^m}
+
\Lambda_m(s,T)
\sqrt{\Phi_S(\bm u(s))-\Phi_S(\bm u(T))}
\right).
\end{align}
Thus small training error plus finite pathwise Jacobian-Sobolev norm gives a
non-asymptotic population-risk bound under full feature learning.
\end{theorem}

\begin{proof}
Apply the deterministic sampling inequality \citep{narcowich2006sobolev} to $f=\bm r_T$ and combine with the residual decomposition:
\begin{align}
\|f\|_{L^2(\rho)}
&\le
\left(\sum_i\mu_i\|f(z_i)\|_2^2\right)^{1/2}
+
C_{\mathcal M,m,p}h_S^m\|f\|_{H^m},
\\
\left(\sum_i\mu_i\|\bm r_T(z_i)\|_2^2\right)^{1/2}
&\le
\sqrt{n\mu_{\max}}\|\bm r_T\|_{\ell_2(S)},
\\
\bm r_T
&=
F(\bm w(s),\cdot)-\bm y-\mathscr T_{s,T}\bm g(s),
\\
\|\bm r_T\|_{H^m}
&\le
\|F(\bm w(s),\cdot)-\bm y\|_{H^m}
+
\|\mathscr T_{s,T}\bm g(s)\|_{H^m}.
\end{align}
\autoref{thm:sobolev_remainder} bounds the last term by
\begin{align}
\|\mathscr T_{s,T}\bm g(s)\|_{H^m}
&\le
\Lambda_m(s,T)\sqrt{\bm g(s)^\top \bm W\bm g(s)}
=
\Lambda_m(s,T)
\sqrt{\Phi_S(\bm u(s))-\Phi_S(\bm u(T))},
\end{align}
which proves the claim.
\end{proof}

\subsection{Test Error Decomposition: Bias, Reservoir Variance, Signal-Channel Variance}\label{app:bias_variance_decomp}

\paragraph{Intuition.} Train-test coupling, reservoir invisibility, and SGD drift-diffusion act on three different parts of the test error. We record the decomposition as a single algebraic identity: under labels $\bm y = f^\star(S) + \bm\varepsilon$ at interpolation, the test error splits into a bias controlled by the irreducible remainder $\bm R_\perp$, a reservoir-noise term that vanishes, and a signal-channel noise term that the algorithm of \autoref{sec:pop_risk_training} suppresses.

\begin{theorem}[Bias and signal-channel variance under train-test coupling]\label{thm:bias_variance_decomp}
Assume the hypotheses of \autoref{thm:train_test_coupling} on $[0,T]$, so $\bm G = \bm A_\circ\bm D$ and $\bm R_\perp = \bm 0$ on the realized trajectory. Take labels $\bm y = f^\star(S) + \bm\varepsilon$ with target $f^\star : \mathcal Z \to \mathbb R^p$ and noise $\bm\varepsilon \in \mathbb R^{np}$, and let $\bm P_{\mathrm{sig}}, \bm P_{\mathrm{res}}$ be the orthogonal projectors onto $\range(\mathcal W_S)$ and $\ker(\mathcal W_S)$. Under squared loss with $\bm B = \tfrac{1}{n}\bm I$,
\begin{align}\label{eq:thesis_decomposition_general}
\bm U_Q(T) - f^\star(Q)
=
\bm U_Q(0) + \tfrac{1}{n}\bm G\bigl(\bm y - \bm U_S(0)\bigr) - f^\star(Q),
\end{align}
and decomposing the residual along $\bm y = f^\star(S) + \bm\varepsilon$ and $\bm\varepsilon = \bm P_{\mathrm{sig}}\bm\varepsilon + \bm P_{\mathrm{res}}\bm\varepsilon$, then using $\bm G = \bm A_\circ\bm D$,
\begin{align}\label{eq:thesis_decomposition_interp}
\bm U_Q(T) - f^\star(Q)
=
\underbrace{\bm U_Q(0) + \tfrac{1}{n}\bm A_\circ\bm D\bigl(f^\star(S) - \bm U_S(0)\bigr) - f^\star(Q)}_{\bm b(T,S)}
+
\underbrace{\tfrac{1}{n}\bm G\bm P_{\mathrm{res}}\bm\varepsilon}_{=\bm 0}
+
\underbrace{\tfrac{1}{n}\bm G\bm P_{\mathrm{sig}}\bm\varepsilon}_{\bm v_{\mathrm{sig}}(T,S,\bm\varepsilon)}.
\end{align}
The reservoir term vanishes: $\bm G\bm P_{\mathrm{res}} = \bm 0$.
\end{theorem}

\begin{proof}
The chain rule gives $\bm U_Q(T) - \bm U_Q(0) = -\bm G\bm g(0)$ exactly (\autoref{sec:operator_derivation}). Squared loss with $\bm B = \tfrac{1}{n}\bm I$ has $\bm g(0) = \tfrac{1}{n}(\bm U_S(0) - \bm y)$, so $-\bm G\bm g(0) = \tfrac{1}{n}\bm G(\bm y - \bm U_S(0))$, which is \eqref{eq:thesis_decomposition_general}. Substituting $\bm y = f^\star(S) + \bm\varepsilon$ and $\bm\varepsilon = \bm P_{\mathrm{sig}}\bm\varepsilon + \bm P_{\mathrm{res}}\bm\varepsilon$ splits the right side into a clean part $\tfrac{1}{n}\bm G(f^\star(S) - \bm U_S(0))$ and the two noise pieces. On the clean part, train-test coupling under squared loss gives $\bm G = \bm A_\circ\bm D$, rewriting it as $\tfrac{1}{n}\bm A_\circ\bm D(f^\star(S) - \bm U_S(0))$ and yielding \eqref{eq:thesis_decomposition_interp}. Reservoir invisibility (\autoref{thm:reservoir_invisibility}) is the inclusion $\ker\bm W \subseteq \ker\bm G$, so $\bm G$ kills $\range(\bm P_{\mathrm{res}}) = \ker\bm W$, giving $\bm G\bm P_{\mathrm{res}} = \bm 0$ directly, with no appeal to a pseudoinverse.
\end{proof}

\paragraph{What controls each term.}
The bias $\bm b(T,S)$ is the optimal train-to-test predictor $\bm A_\circ$ applied to the clean training displacement $\tfrac{1}{n}\bm D(f^\star(S) - \bm U_S(0))$: it asks, given the clean signal alone, what test prediction the realized-trajectory operators would produce. Train-test coupling under squared loss gives $\bm G = \bm A_\circ\bm D$ exactly, so this is the literal physical translator of clean structure; under a Sobolev or RKHS prior of order $m > d_{\mathcal M}/2$ on the data manifold, $\|\bm R_\perp\|_{\op} \le C h_S^{m-d_{\mathcal M}/2}$ (\autoref{thm:sobolev_remainder}). The bound is tight without a smoothness hypothesis, since two networks can share an identical training trajectory yet disagree on a held-out point (\autoref{app:smoothness_required}).

The reservoir variance $\tfrac{1}{n}\bm G\bm P_{\mathrm{res}}\bm\varepsilon$ is identically zero: whatever portion of the label noise the optimizer placed in $\ker\bm W$ during training never reaches the test predictions, because $\bm G$ kills that subspace.

The signal-channel variance $\bm v_{\mathrm{sig}} = \tfrac{1}{n}\bm G\bm P_{\mathrm{sig}}\bm\varepsilon$ is the only failure mode that survives. Two effects bound it. First, the drift-diffusion separation (\autoref{thm:minibatch_coherence}): on a parameter with population gradient $\mu$ and per-example variance $\sigma^2$, SGD's mean update grows like $\eta T \mu$ while its diffusion grows like $\sigma\sqrt{\eta T/b}$, so signal dominates noise as $T \to \infty$ at fixed batch and step size. Second, the population-risk gate (\autoref{thm:diagonal_gate}): the per-parameter cutoff $\mu_k^2 > \sigma_k^2/(b-1)$ is the unique binary mask that prevents adversarial first-order loss along a parameter whose batch signal cannot beat its noise.

\paragraph{The three pieces together.} Bias is small when $\bm A_\circ$ is the right interpolation operator, exact under squared loss and approximate under network smoothness. Reservoir variance is zero by construction. Signal-channel variance shrinks under SGD and shrinks faster under the population-risk gate. Each mechanism handles one term; \eqref{eq:thesis_decomposition_interp} is small only when all three are.

\subsection{On-Trajectory and Off-Trajectory Error}

\paragraph{Intuition.} A fixed predictor's error should be measured along the directions the gradient trajectory actually visits. Measuring it across the full time interval introduces slack from directions never excited by the run. The next theorem isolates the on-trajectory part and quantifies the off-trajectory excess.

Throughout this subsection, fix $0\le s\le T$ and a test set $Q$, and use the notation of
\autoref{thm:coupling_factorization} with
\begin{align}
\bm D&\triangleq \mathsf D_S(T,s), &
\bm G&\triangleq \mathsf G_Q(T,s), &
\bm W&\triangleq \mathcal W_S(s,T),
\end{align}
and let $\bm A_\circ,\bm R_\perp$ be the canonical predictor and irreducible remainder from
\autoref{thm:coupling_factorization}. For any linear predictor
$\bm A:\mathbb R^{np}\to\mathbb R^{|Q|p}$, define the path error
\begin{align}\label{eq:path_error}
\mathcal D_{\bm A}^{\mathrm{vis}}(s,T)
\triangleq
(\bm G-\bm A\bm D)\bm W^\dagger(\bm G-\bm A\bm D)^\top.
\end{align}
The first theorem records the orthogonal split of the on-trajectory error.

\begin{theorem}[On-trajectory error split]\label{thm:path_error_split}
The on-trajectory path error decomposes as
\begin{align}\label{eq:path_error_pyth}
\mathcal D_{\bm A}^{\mathrm{vis}}(s,T)
=
\bm R_\perp \bm R_\perp^\top
+
(\bm A-\bm A_\circ)\bm D\bm W^\dagger \bm D^\top(\bm A-\bm A_\circ)^\top,
\end{align}
which gives, in the positive-semidefinite order,
\begin{align}\label{eq:path_error_loewner}
\mathcal D_{\bm A}^{\mathrm{vis}}(s,T)&\succeq \bm R_\perp \bm R_\perp^\top,\\
\label{eq:path_error_min}
\inf_{\bm A} \mathcal D_{\bm A}^{\mathrm{vis}}(s,T)&=\bm R_\perp \bm R_\perp^\top,
\end{align}
with equality in the first line iff $\bm A \bm D=\bm A_\circ \bm D$.
\end{theorem}

The next theorem promotes the cumulative error operator to a positive-semidefinite split: the on-trajectory part plus an off-trajectory slack.

\begin{theorem}[On-/off-trajectory split of the cumulative error]\label{thm:cumulative_error_split}
The cumulative error operator
\begin{align}
\mathcal D_{\bm A}(s,T)
&\triangleq
\int_s^T
\Delta_{\bm A}(\tau)\Kss(\tau)^\dagger\Delta_{\bm A}(\tau)^\top d\tau, &
\Delta_{\bm A}(\tau)&\triangleq \Kqs(\tau)-\bm A \Kss(\tau),
\intertext{admits the split}
\label{eq:lifted_error_split}
\mathcal D_{\bm A}(s,T)
&=
\mathcal D_{\bm A}^{\mathrm{vis}}(s,T)
+
\mathcal N_{\bm A}(s,T), &
\mathcal N_{\bm A}(s,T)&\succeq 0.
\end{align}
\end{theorem}

The final theorem of this trio converts the operator-level statement into a deterministic bound on test displacement along the realized trajectory.

\begin{theorem}[Trajectory bound from the on-trajectory error]\label{thm:trajectory_bound}
On the true trajectory,
\begin{align}\label{eq:visible_trajectory_bound}
\|\bm U_Q(T)-\bm U_Q(s)-\bm A(\bm u(T)-\bm u(s))\|_2^2
\le
\|\mathcal D_{\bm A}^{\mathrm{vis}}(s,T)\|_{\op}
\bigl(\Phi_S(\bm u(s))-\Phi_S(\bm u(T))\bigr).
\end{align}
\end{theorem}

If the predictor is preserved along the trajectory, train motion fully determines test motion.

\begin{corollary}[Exact fixed transfer under perfect predictor invariance]
If $\Delta_{\bm A}(\tau)=0$ for a.e.\ $\tau\in[s,T]$, equivalently $\mathcal D_{\bm A}(s,T)=0$, then
\begin{align}
\mathcal D_{\bm A}^{\mathrm{vis}}(s,T)&=0,\\
\bm G&=\bm A\bm D,\\
\bm U_Q(T)-\bm U_Q(s)&=\bm A\bigl(\bm u(T)-\bm u(s)\bigr).
\end{align}
\end{corollary}

The off-trajectory version overestimates by the off-path slack $\mathcal N_{\bm A}$.

\begin{corollary}[Sharp on-trajectory bound and coarse off-trajectory bound]
For every linear predictor $\bm A$, the visible infimum is sharp and bounds the lifted infimum:
\begin{align}
\|\bm R_\perp\|_{\op}^2
&\le
\|\mathcal D_{\bm A}^{\mathrm{vis}}(s,T)\|_{\op}
\le
\|\mathcal D_{\bm A}(s,T)\|_{\op},\\
\inf_{\bm A} \|\mathcal D_{\bm A}^{\mathrm{vis}}(s,T)\|_{\op}
&=
\|\bm R_\perp\|_{\op}^2
\le
\inf_{\bm A} \|\mathcal D_{\bm A}(s,T)\|_{\op}.
\end{align}
\end{corollary}

\begin{remark}[Anchored transfer as a computable coarse bound]
The anchor $\bm A_s\triangleq \Kqs(s)\Kss(s)^\dagger$ gives the anchored off-trajectory bound
\begin{align}
\|\bm R_\perp\|_{\op}^2
\le
\left\|
\int_s^T
\bigl(\Kqs(\tau)-\bm A_s \Kss(\tau)\bigr)
\Kss(\tau)^\dagger
\bigl(\Kqs(\tau)-\bm A_s \Kss(\tau)\bigr)^\top
d\tau
\right\|_{\op}.
\end{align}
By \autoref{thm:cumulative_error_split}, this bound exceeds the visible
error by a positive-semidefinite off-trajectory slack and therefore vanishes under $\mathcal{O}(1)$ raw kernel drift when the transfer operator is preserved, leaving the on-path remainder $\bm R_\perp$ as the limit of fixed transfer.
\end{remark}

\begin{proof}[Proof of \autoref{thm:path_error_split}, \autoref{thm:cumulative_error_split}, and \autoref{thm:trajectory_bound}]
Set $\mathcal H_{s,T}\triangleq L^2([s,T];\mathbb R^{np})$ and define the lift sending $\bm h\in\mathbb R^{np}$ to its time-indexed path,
\begin{align}
(\mathcal T_{s,T}\bm h)(\tau)
&\triangleq
\Kss(\tau)^{1/2}\mathcal P_g(\tau,s)\bm h,\\
\mathcal T_{s,T}^\ast \bm\xi
&=
\int_s^T \mathcal P_g(\tau,s)^\top \Kss(\tau)^{1/2}\bm\xi(\tau) d\tau,\\
\mathcal T_{s,T}^\ast\mathcal T_{s,T}&=\bm W.
\end{align}
Since the domain of $\mathcal T_{s,T}$ is finite-dimensional, $\range(\mathcal T_{s,T})$ is closed, so
\begin{align}
P_{\mathrm{vis}}(s,T)
\triangleq
\mathcal T_{s,T}\bm W^\dagger \mathcal T_{s,T}^\ast
\end{align}
is the orthogonal projector onto $\range(\mathcal T_{s,T})\subset \mathcal H_{s,T}$.

Now define the extended train/test maps
\begin{align}
\mathcal L_S\bm\xi
&\triangleq
\int_s^T \Kss(\tau)^{1/2}\bm\xi(\tau) d\tau,\\
\mathcal L_Q\bm\xi
&\triangleq
\int_s^T \Kqs(\tau)\Kss(\tau)^{\dagger/2}\bm\xi(\tau) d\tau.
\end{align}
Since $\Kqs(\tau)$ annihilates $\ker \Kss(\tau)$, both maps are well-defined and bounded, and
\begin{align}
\mathcal L_S\mathcal T_{s,T}\bm h
&=
\int_s^T \Kss(\tau)\mathcal P_g(\tau,s)\bm h d\tau
=
\bm D\bm h,\\
\mathcal L_Q\mathcal T_{s,T}\bm h
&=
\int_s^T \Kqs(\tau)\Kss(\tau)^{\dagger/2}\Kss(\tau)^{1/2}\mathcal P_g(\tau,s)\bm h d\tau\\
&=
\int_s^T \Kqs(\tau)\mathcal P_g(\tau,s)\bm h d\tau
=
\bm G\bm h.
\end{align}
Hence $\bm G-\bm A\bm D = (\mathcal L_Q-\bm A\mathcal L_S)\mathcal T_{s,T}$, and
\begin{align}\label{eq:path_error_projected}
\mathcal D_{\bm A}^{\mathrm{vis}}(s,T)
&=
(\bm G-\bm A\bm D)\bm W^\dagger(\bm G-\bm A\bm D)^\top\\
&=
(\mathcal L_Q-\bm A\mathcal L_S)\mathcal T_{s,T}\bm W^\dagger\mathcal T_{s,T}^\ast
(\mathcal L_Q-\bm A\mathcal L_S)^\top\\
&=
(\mathcal L_Q-\bm A\mathcal L_S)P_{\mathrm{vis}}(s,T)(\mathcal L_Q-\bm A\mathcal L_S)^\top.
\end{align}

Next, \eqref{eq:path_error_decomp} from \autoref{thm:coupling_factorization} gives
\begin{align}
(\bm G-\bm A\bm D)\bm W^\dagger(\bm G-\bm A\bm D)^\top
=
\bm R_\perp \bm R_\perp^\top
+
(\bm A-\bm A_\circ)\bm D\bm W^\dagger \bm D^\top(\bm A-\bm A_\circ)^\top,
\end{align}
which is \eqref{eq:path_error_pyth}. The positive-semidefinite lower bound
\eqref{eq:path_error_loewner} and the infimum \eqref{eq:path_error_min}
follow immediately, with equality iff the misspecification term vanishes, i.e.\
$\bm A\bm D=\bm A_\circ \bm D$.

It remains to identify the slack term in the off-trajectory error. For any
$\bm y\in\mathbb R^{|Q|p}$,
\begin{align}
(\mathcal L_Q-\bm A\mathcal L_S)^\top \bm y(\tau)
=
\Kss(\tau)^{\dagger/2}\Delta_{\bm A}(\tau)^\top \bm y,
\end{align}
using $\bm A \Kss(\tau)^{1/2} = \bm A \Kss(\tau)\Kss(\tau)^{\dagger/2}$. Therefore
\begin{align}
(\mathcal L_Q-\bm A\mathcal L_S)(\mathcal L_Q-\bm A\mathcal L_S)^\top
=
\int_s^T
\Delta_{\bm A}(\tau)\Kss(\tau)^\dagger\Delta_{\bm A}(\tau)^\top d\tau
=
\mathcal D_{\bm A}(s,T).
\end{align}
Splitting $I=P_{\mathrm{vis}}(s,T)+(I-P_{\mathrm{vis}}(s,T))$ gives
\begin{align}
\mathcal D_{\bm A}(s,T)
&=
(\mathcal L_Q-\bm A\mathcal L_S)P_{\mathrm{vis}}(s,T)(\mathcal L_Q-\bm A\mathcal L_S)^\top\\
&{}+
(\mathcal L_Q-\bm A\mathcal L_S)(I-P_{\mathrm{vis}}(s,T))(\mathcal L_Q-\bm A\mathcal L_S)^\top\\
&=
\mathcal D_{\bm A}^{\mathrm{vis}}(s,T)+\mathcal N_{\bm A}(s,T),
\end{align}
where
\begin{align}
\mathcal N_{\bm A}(s,T)
\triangleq
(\mathcal L_Q-\bm A\mathcal L_S)(I-P_{\mathrm{vis}}(s,T))(\mathcal L_Q-\bm A\mathcal L_S)^\top
\succeq 0.
\end{align}
This proves \eqref{eq:lifted_error_split}.

Finally, $\bm U_Q(T)-\bm U_Q(s)-\bm A(\bm u(T)-\bm u(s)) = -(\bm G-\bm A\bm D)\bm g(s)$, so by \autoref{thm:output_dynamics} and \autoref{thm:reservoir_invisibility},
\begin{align}
\|\bm U_Q(T)-\bm U_Q(s)-\bm A(\bm u(T)-\bm u(s))\|_2^2
&\le
\|(\bm G-\bm A\bm D)\psinv{\bm W}\|_{\op}^2 \bm g(s)^\top \bm W \bm g(s)\\
&=
\|\mathcal D_{\bm A}^{\mathrm{vis}}(s,T)\|_{\op}
\bigl(\Phi_S(\bm u(s))-\Phi_S(\bm u(T))\bigr),
\end{align}
which is \eqref{eq:visible_trajectory_bound}.
\end{proof}

\subsection{Test Prediction Bounds from Network Smoothness}\label{sec:exact_effective_closure}

\paragraph{Intuition.} Two scalars, the visibility Gramian $\Gamma_Q$ and the irreducible remainder $\bm R_\perp$, give the path-error bounds along the realized trajectory. Both are computable from forward and backward ODE solves on the observed window through the Dual Transfer Theorem (\autoref{thm:forward_backward}), so the bounds are practical to evaluate from a single training run.

\begin{corollary}[Path-error bounds]
\label{cor:path_error_bounds}
Under the notation of \autoref{thm:train_test_coupling}, $\|\Gamma_Q(s,T)\|_{\op}$ and $\|\bm R_\perp\|_{\op}$ satisfy
\begin{align}
\|\Gamma_Q(s,T)\|_{\op}
&=
\sup_{\bm h\notin\ker \bm W}\frac{\|\bm G\bm h\|_2^2}{\bm h^\top \bm W \bm h},
\\
\|\bm R_\perp\|_{\op}^2
&=
\inf_{\bm A} \|(\bm G-\bm A\bm D)\psinv{\bm W}\|_{\op}^2
=
\sup_{\substack{\bm h\notin\ker \bm W\\ \bm D\bm h=0}}
\frac{\|\bm G\bm h\|_2^2}{\bm h^\top \bm W \bm h}.
\end{align}
Along the true trajectory,
\begin{align}
\|\bm U_Q(T)-\bm U_Q(s)\|_2
&\le
\|\Gamma_Q(s,T)\|_{\op}^{1/2}
\sqrt{\Phi_S(\bm u(s))-\Phi_S(\bm u(T))},
\label{eq:effective_purif_bound}
\\
\|\bm U_Q(T)-\bm U_Q(s)-\bm A_\circ\bigl(\bm u(T)-\bm u(s)\bigr)\|_2
&\le
\|\bm R_\perp\|_{\op}
\sqrt{\Phi_S(\bm u(s))-\Phi_S(\bm u(T))}.
\label{eq:effective_screen_bound}
\end{align}
\end{corollary}

The proof is given in \autoref{sec:feature_learning}; the remainder characterization follows from \autoref{cor:linear_predictor_optimal}, the visibility characterization from \autoref{thm:reservoir_invisibility}, and the path-error inequalities from the path-error decomposition (\autoref{thm:path_error_split}).

\begin{remark}[Transfer characterization]\label{rem:transfer_characterization}
Prediction from the observed training displacement ($\|\bm R_\perp\|_{\op}=0$) holds if and only if $\ker \bm D\subseteq \ker \bm G$. In the frozen-kernel limit \citep{jacot2018neural}, $\|\Gamma_Q\|_{\op}^{1/2}$ recovers the classical NTK test-invisibility constant $\Pi_Q$ and $\|\bm R_\perp\|_{\op}$ recovers the kernel-regression residual.
\end{remark}

\begin{proof}[Proof of \autoref{cor:path_error_bounds}]
\textit{Operator norm of $\Gamma_Q(s,T)$.}
By \autoref{thm:reservoir_invisibility},
\begin{align}
\bm G^\top \bm G=\bm W^{1/2}\Gamma_Q(s,T)\bm W^{1/2},
\quad
\Gamma_Q(s,T)=\psinv{\bm W}\bm G^\top \bm G\psinv{\bm W}\succeq 0.
\end{align}
The substitution $\bm z=\bm W^{1/2}\bm h$ gives
\begin{align}
\|\Gamma_Q(s,T)\|_{\op}
=
\sup_{\bm h\notin\ker \bm W}
\frac{\|\bm G\bm h\|_2^2}{\bm h^\top \bm W\bm h}.
\end{align}

\textit{Variational characterization of $\|\bm R_\perp\|_{\op}$.}
By \eqref{eq:nonlinear_gap_identity},
\begin{align}
\inf_{\bm A}\|(\bm G-\bm A\bm D)\psinv{\bm W}\|_{\op}
=
\|\bm R_\perp\|_{\op}
=
\sup_{\substack{\bm h\notin\ker \bm W\\ \bm D\bm h=0}}
\frac{\|\bm G\bm h\|_2}{\sqrt{\bm h^\top \bm W \bm h}}.
\end{align}
Squaring gives the stated expression. The infimum characterization is equivalent, since the same substitution gives
\begin{align}
\inf\{c\ge 0:(\bm G-\bm A\bm D)^\top(\bm G-\bm A\bm D)\preceq c\bm W\}=\|(\bm G-\bm A\bm D)\psinv{\bm W}\|_{\op}^2.
\end{align}

\textit{Equations \eqref{eq:effective_purif_bound}--\eqref{eq:effective_screen_bound}.}
By \autoref{thm:reservoir_invisibility},
\begin{align}
\bm U_Q(T)-\bm U_Q(s)&=-\bm G\bm g(s), &
\bm u(T)-\bm u(s)&=-\bm D\bm g(s),\\
\bm g(s)^\top \bm W\bm g(s)&=\Phi_S(\bm u(s))-\Phi_S(\bm u(T)).
\end{align}
Setting $\bm h=\bm g(s)$ in the two variational formulas gives
\begin{align}
\|\bm G\bm g(s)\|_2^2
&\le
\|\Gamma_Q(s,T)\|_{\op}\bm g(s)^\top \bm W\bm g(s),\\
\|(\bm G-\bm A_\circ \bm D)\bm g(s)\|_2^2
&\le
\|\bm R_\perp\|_{\op}^2\bm g(s)^\top \bm W\bm g(s),
\end{align}
which are \eqref{eq:effective_purif_bound}--\eqref{eq:effective_screen_bound} after square roots.
\end{proof}

\subsection{Computing Transfer Operators by Forward-Backward ODE Solves}\label{sec:forward_backward}

\paragraph{Intuition.} The pathwise operators $\bm W=\mathcal W_S(s,t)$, $\bm D=\mathsf D_S(t,s)$, and $\bm G=\mathsf G_Q(t,s)$ are computable. On any observed window $[s,t]$, applying any one of them amounts to forward and backward linear ODE solves along the realized feature-learning trajectory. The visibility Gramian $\Gamma_Q$ and remainder $\bm R_\perp$ thus become matrix-free quantities computable from a single run, in the spirit of adjoint-based influence computation \citep{koh2017understanding}. Throughout this subsection, fix an observed window $0\le s<t$ along the training run and write $\bm A(\tau)\triangleq \bm B(\tau)\Kss(\tau)$.

\begin{proposition}[Dual Transfer]\label{thm:forward_backward}
\textit{Forward solves.}
For $\bm h\in\mathbb R^{np}$, the forward solution $\bm z_h:[s,t]\to\mathbb R^{np}$ to
$\partial_\tau \bm z_h(\tau)=-\bm A(\tau)\bm z_h(\tau)$, $\bm z_h(s)=\bm h$, satisfies
\begin{align}
\bm D\bm h &= \int_s^t \Kss(\tau)\bm z_h(\tau) d\tau, \label{eq:causal_dual_D}\\
\bm G\bm h &= \int_s^t \Kqs(\tau)\bm z_h(\tau) d\tau. \label{eq:causal_dual_G}
\end{align}

\textit{Backward solves.}
For $\bm h\in\mathbb R^{np}$, $\bm\xi\in\mathbb R^{|Q|p}$, $\bm\eta\in\mathbb R^{np}$, the backward solutions
$\bm m_h,\bm q_\xi,\bm d_\eta:[s,t]\to\mathbb R^{np}$ with terminal value zero at $\tau=t$ to
\begin{align}
-\partial_\tau \bm m_h(\tau)&=-\bm A(\tau)^\top \bm m_h(\tau)+\Kss(\tau)\bm z_h(\tau),\\
-\partial_\tau \bm q_\xi(\tau)&=-\bm A(\tau)^\top \bm q_\xi(\tau)+\Ksq(\tau)\bm\xi,\\
-\partial_\tau \bm d_\eta(\tau)&=-\bm A(\tau)^\top \bm d_\eta(\tau)+\Kss(\tau)\bm\eta,
\intertext{recover the operator actions at $\tau=s$,}
\bm m_h(s)&=\bm W\bm h, \label{eq:causal_dual_W}\\
\bm q_\xi(s)&=\bm G^\top \bm\xi, \label{eq:causal_dual_Gadj}\\
\bm d_\eta(s)&=\bm D^\top \bm\eta. \label{eq:causal_dual_Dadj}
\end{align}
\end{proposition}

\begin{corollary}[Matrix-Free Pathwise Quantities]\label{cor:matrix_free_constants}
On the observed window, the pathwise quantities $\|\Gamma_Q(s,t)\|_{\op}$ and $\|\bm R_\perp\|_{\op}$ satisfy
\begin{align}
\|\Gamma_Q(s,t)\|_{\op}
&=
\sup_{\bm h\notin\ker \bm W}\frac{\|\bm G\bm h\|_2^2}{\bm h^\top \bm W \bm h},\\
\|\bm R_\perp\|_{\op}^2
&=
\inf_{\bm A}\sup_{\bm h\notin\ker \bm W}\frac{\|(\bm G-\bm A\bm D)\bm h\|_2^2}{\bm h^\top \bm W \bm h},
\end{align}
and are one-run matrix-free quantities: every application of $\bm W$, $\bm D$, $\bm D^\top$, $\bm G$, and $\bm G^\top$ reduces to forward/backward linear ODE solves on the realized feature-learning path via \autoref{thm:forward_backward}, valid for arbitrary kernel drift.
\end{corollary}

\begin{proof}
Since $\bm z_h(\tau)=\mathcal P_g(\tau,s)\bm h$, \autoref{eq:causal_dual_D} and \autoref{eq:causal_dual_G} are the definitions of $\bm D$ and $\bm G$. Variation of constants on the three backward equations gives
\begin{align}
\bm m_h(s)
&=
\int_s^t \mathcal P_g(\tau,s)^\top \Kss(\tau)\bm z_h(\tau) d\tau\\
&=
\int_s^t \mathcal P_g(\tau,s)^\top \Kss(\tau)\mathcal P_g(\tau,s)\bm h d\tau
=
\bm W\bm h,\\
\bm q_\xi(s)
&=
\int_s^t \mathcal P_g(\tau,s)^\top \Ksq(\tau)\bm\xi d\tau
=
\bm G^\top\bm\xi,\\
\bm d_\eta(s)
&=
\int_s^t \mathcal P_g(\tau,s)^\top \Kss(\tau)\bm\eta d\tau
=
\bm D^\top\bm\eta.
\end{align}
The remaining claim follows because $\|\Gamma_Q\|_{\op}$ and $\|\bm R_\perp\|_{\op}$ depend only on applications of $\bm W$, $\bm D$, $\bm D^\top$, $\bm G$, and $\bm G^\top$.
\end{proof}

\section{Population Risk Training: Proofs and Algorithm}\label{app:population_risk_proofs}

\subsection{Exchangeability}\label{app:main_text_proofs}

\paragraph{Intuition.} Each leave-one-out evaluation places a held-out training point against a model that did not see it, which is distributionally indistinguishable from a fresh draw against a model trained on $n-1$ independent samples. Averaging over $i$ converts the empirical average of leave-one-out losses to a population risk.

\begin{proof}[Proof of \autoref{thm:exchangeability_identity}]
For any fixed $i$, the i.i.d.\ assumption gives
\begin{align}
(S_{-i},Z_i)\stackrel{d}{=}(S_{n-1},Z),\qquad S_{n-1}\sim\mathcal D^{n-1},\quad Z\sim\mathcal D\text{ independent of }S_{n-1}.
\end{align}
Averaging $\mathbb E[\ell(\bm w_T(S_{-i}),Z_i)]=\mathbb E[\mathcal L_{\mathcal D}(\bm w_T(S_{n-1}))]$ over $i$ yields the claim.
\end{proof}

\subsection{Population Risk from Kernel-Block Agreement}\label{sec:kernel_block_derivation}\label{app:direct_population_risk_details}

\paragraph{Per-example triple, kernel block, and leave-one-out expansion.}
Specializing \autoref{thm:loo_test_transfer} to a one-step window from the current iterate, fix the optimizer history $\mathcal F_t$, let $\bm w_t$ and a base preconditioner $\bm P_t\succeq 0$ be $\mathcal F_t$-measurable, and let $B = (z_1,\dots,z_b)$ be an exchangeable batch independent of $\mathcal F_t$. For each example record the output residual, the per-example Jacobian, and the parameter-space gradient,
\begin{align}\label{eq:per_example_triple}
\bm r_a = \nabla_{\bm u}\ell\bigl(F(\bm w_t,z_a), z_a\bigr),
\qquad
\bm J_a = D_{\bm w} F(\bm w_t, z_a),
\qquad
\bm g_a = \bm J_a^\top \bm r_a.
\end{align}
A preconditioner $M\succeq 0$ induces the pairwise kernel block $K_{ac}^M = \bm J_a M \bm J_c^\top$, which is the $(a,c)$ block of $\Kss^M$. If example $a$ is evaluated on a step trained without it, the update is $\bm w_{-a}^+ = \bm w_t - \eta M \bar{\bm g}_{-a}$ with $\bar{\bm g}_{-a} = (b-1)^{-1}\sum_{c\ne a}\bm g_c$, and a first-order expansion lifts the parameter inner product to a kernel block:
\begin{align}\label{eq:loo_kernel_expansion}
\ell_a(\bm w_{-a}^+)
=
\ell_a(\bm w_t)
- \frac{\eta}{b-1}\sum_{c\ne a} \bm r_a^\top K_{ac}^M \bm r_c
+ O(\eta^2).
\end{align}
Averaging over $a$, the first-order population-safe improvement of the kernel increment induced by $M$ is the off-diagonal agreement
\begin{align}\label{eq:off_diagonal_kernel_rate}
\Omega_B(M)
=
\frac{1}{b(b-1)}\sum_{a\ne c} \bm r_a^\top K_{ac}^M \bm r_c.
\end{align}
The diagonal terms $a=c$ are absent: those are the self-use terms removed by leave-one-out. A kernel increment improves population risk to the extent that distinct examples agree through it.

\paragraph{Collapse to parameter space.}
Substituting $\bm g_a = \bm J_a^\top\bm r_a$ rewrites the kernel-block agreement as an inner product on parameter gradients,
\begin{align}\label{eq:off_diagonal_rate}
\Omega_B(M)
=
\frac{1}{b(b-1)}\sum_{a\ne c} \bm g_a^\top M \bm g_c
=
\bar{\bm g}_B^\top M \bar{\bm g}_B - \frac{1}{b-1}\operatorname{tr}(M\bm\Sigma_B)
=
\operatorname{tr}(M \bm A_B),
\end{align}
with $\bar{\bm g}_B = b^{-1}\sum_a \bm g_a$, centered residuals $\bm c_a = \bm g_a - \bar{\bm g}_B$, minibatch covariance $\bm\Sigma_B = b^{-1}\sum_a \bm c_a \bm c_a^\top$, and the off-diagonal rate matrix
\begin{align}\label{eq:one_step_loo}
\bm A_B \triangleq \bar{\bm g}_B \bar{\bm g}_B^\top - \frac{1}{b-1}\bm\Sigma_B.
\end{align}

\begin{theorem}[Kernel-increment population risk]\label{thm:kernel_increment}
Let $\mathcal R_{\mathrm{1ex},B}^\eta = b^{-1}\sum_a \ell_a(\bm w_{-a}^+)$ be the average leave-one-out risk on the current batch. Conditional on $\mathcal F_t$, exchangeability gives $(B_{-a}, Z_a)\stackrel{d}{=}(S_{b-1},Z)$, so $\mathbb E[\mathcal R_{\mathrm{1ex},B}^\eta\mid\mathcal F_t]$ is the population risk of the one-step learner trained on an independent $(b-1)$-sample. To first order in $\eta$,
\begin{align}\label{eq:one_step_expansion_main}
\mathcal R_{\mathrm{1ex},B}^\eta
=
\widehat L_B(\bm w_t)
- \eta\operatorname{tr}(M \bm A_B)
+ O(\eta^2),
\end{align}
so maximizing $\operatorname{tr}(M\bm A_B)$ over $M\succeq 0$ is the first-order population-risk rule for selecting the next kernel increment $\bm J_S M \bm J_S^\top$ added to $\mathcal W_S^M$. Given a base preconditioner $\bm P_t\succeq 0$, the unique $M$ with $0\preceq M\preceq\bm P_t$ attaining this maximum is the spectral projector through the optimizer's metric,
\begin{align}\label{eq:full_psd_projector}
M^\star = \bm P_t^{1/2}\mathbf 1_{(0,\infty)}\!\bigl(\bm P_t^{1/2}\bm A_B\bm P_t^{1/2}\bigr)\bm P_t^{1/2},
\end{align}
which keeps the positive eigenspace of $\bm A_B$ as seen through $\bm P_t$ and discards the rest.
\end{theorem}

\begin{corollary}[Diagonal gate]\label{thm:diagonal_gate}
For diagonal $\bm P_t = \operatorname{diag}(p_k)$, $M^\star$ in \eqref{eq:full_psd_projector} updates parameter $k$ exactly when
\begin{align}\label{eq:hard_filter}
\mu_k^2 > \sigma_k^2/(b-1),
\qquad
\mu_k = \bar g_{B,k},
\quad
\sigma_k^2 = (\bm\Sigma_B)_{kk}.
\end{align}
The reverse direction is immediate: if the gate updates a parameter with $\mu_k^2 < \sigma_k^2/(b-1)$, the worst-case loss curvature on that parameter forces a strict first-order increase in population risk.
\end{corollary}

\begin{proof}[Proof of \autoref{thm:loo_test_transfer}]
On the one-step window $[t,t+\eta]$ from $\bm w_t$, the propagator and test transfer of \autoref{def:test_transfer} on training set $S_{-a}$ satisfy
\begin{align}
\partial_\tau\mathcal P_g^M(\tau,t)
&=
-\bm B(\tau)\Kss^M(\tau)\mathcal P_g^M(\tau,t),
\qquad
\mathcal P_g^M(t,t)=\bm I,
\\
\mathcal P_g^M(\tau,t) &= \bm I+O(\tau-t),
\\
\mathsf G_{Q_a,S_{-a}}^M(t+\eta,t)
&=
\int_t^{t+\eta}\bm K^{M}_{Q_a,S_{-a}}(\tau)\mathcal P_g^M(\tau,t)\, d\tau
=
\eta\bm K^{M}_{Q_a,S_{-a}}(t)+O(\eta^2).
\end{align}
Since $a\notin S_{-a}$, the $c$-block of $\bm K^{M}_{Q_a,S_{-a}}$ equals $\bm J_a M\bm J_c^\top=K_{ac}^M$ for each $c\in S_{-a}$, with the diagonal block $K_{aa}^M$ absent. Under the convention $\Phi_{S_{-a}}=(b-1)^{-1}\sum_{c\in S_{-a}}\phi_c$, the stacked output gradient $\bm g_{S_{-a}}(t)$ has $c$-block $(b-1)^{-1}\bm r_c$, and a first-order Taylor expansion of $\ell_a$ at $\bm w_t$ gives
\begin{align}
\ell_a(\bm w_t)-\ell_a(\bm w_{-a}^+)
&=
\bm r_a^\top\bigl(\bm U_{Q_a}(\bm w_t)-\bm U_{Q_a}(\bm w_{-a}^+)\bigr)
+
O\bigl(\|\bm U_{Q_a}(\bm w_{-a}^+)-\bm U_{Q_a}(\bm w_t)\|_2^2\bigr)
\\
&=
\eta\bm r_a^\top\bm K^{M}_{Q_a,S_{-a}}(t)\bm g_{S_{-a}}(t)+O(\eta^2)
\\
&=
\frac{\eta}{b-1}\sum_{c\ne a}\bm r_a^\top K_{ac}^M\bm r_c+O(\eta^2).
\end{align}
Averaging over $a$ yields \eqref{eq:loo_average_omega}. Conditional on $\mathcal F_t$, exchangeability of $B$ gives $(B_{-a},Z_a)\stackrel{d}{=}(S_{b-1},Z)$, so the average leave-one-out improvement is in expectation the population-risk improvement of the one-step learner on $b-1$ independent samples.
\end{proof}

\begin{proof}[Proof of \autoref{thm:kernel_increment}]
Set the centered residual, the batch step, and the leave-one-out step,
\begin{align}
\bm c_a &\triangleq \bm g_a - \bar{\bm g}_B,
&
\bm w^+ &\triangleq \bm w_t - \eta M \bar{\bm g}_B,
&
\bm w_{-a}^+ &= \bm w^+ + \frac{\eta}{b-1} M \bm c_a,
\end{align}
where the last equality uses $\bar{\bm g}_{-a} = \bar{\bm g}_B - \bm c_a/(b-1)$. Taylor-expanding $\ell_a$ and $\widehat L_B$ at $\bm w^+$ and lifting parameter inner products through $\bm g_a = \bm J_a^\top \bm r_a$,
\begin{align}
\bm g_a^\top M \bm g_c
&=
\bm r_a^\top \underbrace{\bm J_a M \bm J_c^\top}_{K_{ac}^M} \bm r_c, \label{eq:parameter_to_kernel}
\\
\ell_a(\bm w_{-a}^+)
&=
\ell_a(\bm w^+)
+ \frac{\eta}{b-1}\bm g_a^\top M \bm c_a
+ O(\eta^2),
\\
\widehat L_B(\bm w^+)
&=
\widehat L_B(\bm w_t)
- \eta\bar{\bm g}_B^\top M \bar{\bm g}_B
+ O(\eta^2).
\end{align}
Averaging over $a$ and using $\sum_a \bm c_a = 0$ to reduce $b^{-1}\sum_a \bm g_a^\top M \bm c_a$ to $\operatorname{tr}(M\bm\Sigma_B)$,
\begin{align}
\frac{1}{b}\sum_a \ell_a(\bm w_{-a}^+)
&=
\widehat L_B(\bm w^+)
+ \frac{\eta}{b-1}\operatorname{tr}(M\bm\Sigma_B)
+ O(\eta^2)
\\
&=
\widehat L_B(\bm w_t)
- \eta\!\left[\bar{\bm g}_B^\top M \bar{\bm g}_B - \frac{1}{b-1}\operatorname{tr}(M\bm\Sigma_B)\right]
+ O(\eta^2)
\\
&=
\widehat L_B(\bm w_t) - \eta\operatorname{tr}(M\bm A_B) + O(\eta^2).
\end{align}
For the population-risk part, condition on $\mathcal F_t$ and use exchangeability of $B$, $(B_{-a},Z_a)\stackrel{d}{=}(S_{b-1},Z)$ for an i.i.d.\ $(b-1)$-sample $S_{b-1}$ and an independent draw $Z$. The conditional law gives
\begin{align}
\mathbb E[\ell_a(\bm w_{-a}^+) \mid \mathcal F_t]
=
\mathbb E_{S_{b-1}, Z}\!\left[\ell\bigl(\bm w_t - \eta M \nabla L_{S_{b-1}}(\bm w_t), Z\bigr) \mid \mathcal F_t\right],
\end{align}
which is the population risk of the one-step learner on $b-1$ independent samples; averaging over $a$ closes the proof.
\end{proof}

\begin{proof}[Proof of \autoref{thm:diagonal_gate}, full-matrix and diagonal forms]
The objective $\operatorname{tr}(M\bm A_B)$ is linear in $M$. Writing $M = \bm P_t^{1/2} N \bm P_t^{1/2}$ with $0\preceq N\preceq I$ and $C = \bm P_t^{1/2}\bm A_B\bm P_t^{1/2}$, the change of variable yields
\begin{align}
\operatorname{tr}(M\bm A_B)
&=
\operatorname{tr}(NC),
\\
N^\star &= \mathbf 1_{(0,\infty)}(C),
\end{align}
where $N^\star$ is the unique maximizer over $\{0\preceq N\preceq I\}$, placing weight $1$ on every positive eigenspace of $C$ and zero on every nonpositive one; conjugating back gives $M^\star$ in \eqref{eq:full_psd_projector}.

For the diagonal form, take $M = \operatorname{diag}(q_k p_k)$ to obtain
\begin{align}
\operatorname{tr}(M\bm A_B)
&=
\sum_k q_k p_k\!\left[\mu_k^2 - \frac{\sigma_k^2}{b-1}\right].
\end{align}
With $p_k\ge 0$, every summand is nonnegative iff $\mu_k^2 > \sigma_k^2/(b-1)$, which matches the diagonal of $M^\star$ when $\bm P_t$ is diagonal. Any parameter with $\mu_k^2 < \sigma_k^2/(b-1)$ admits an adversarial loss curvature that produces a strict first-order increase in population risk, so the threshold is tight.
\end{proof}

For multi-epoch training, replayed batches carry information about $\bm w_t$ and the independence used above is replaced by a total-variation bound; see \autoref{app:proof_minibatch_coherence}.

\paragraph{Reservoir invisibility.}
Per-example parameter gradients factor as $\bm g_a=\bm J_a^\top\bm r_a$. Writing $\bm r_B=b^{-1}\sum_a(\bm e_a\otimes\bm r_a)$ and using $\bm c_a^{\bm w}=\bm J_S^\top\bm c_a^{\bm u}$ for the output-space centering,
\begin{align}\label{eq:reservoir_free_appendix}
\bar{\bm g}_B\bar{\bm g}_B^\top
&=
\bm J_S^\top\bm r_B\bm r_B^\top\bm J_S,
&
\bm\Sigma_B
&=
\bm J_S^\top\bm\Sigma_B^{\bm u}\bm J_S,
&
\bm A_B
&=
\bm J_S^\top\!\left(\bm r_B\bm r_B^\top-\frac{1}{b-1}\bm\Sigma_B^{\bm u}\right)\!\bm J_S.
\end{align}
Any output direction in $\ker\Kss=\ker\bm J_S^\top$ contributes zero to $\operatorname{tr}(M\bm A_B)$ for every $M$, so the inclusion $\ker\mathcal W_S\subseteq\ker\mathsf G$ from \autoref{thm:reservoir_invisibility} makes the reservoir silent for the population-safe rate and for every test prediction through the same factorization.

\paragraph{From a single step to the trajectory.}
Each step adds the increment $\eta\Kss^{M_t}(t)$ to the cumulative dissipation $\mathcal W_S^M$, whose range is the signal channel of the run. The gate \eqref{eq:full_psd_projector} chooses $M_t$ to maximize the $\mathsf G$-rate of that increment, and $\mathcal W_S^M$ at trajectory scale is the integral of these one-step maximizers. Population-risk training stands to the averaged test transfer $\mathsf G_{Q_a,S_{-a}}$ as plain SGD stands to the training transfer $\bm D$.

\subsection{Cross-Validation Risks as Population Risk}\label{sec:exchangeable_boundary}

\paragraph{Intuition.} Cross-validation estimates test loss without a held-out set. We show that every $k$-fold cross-validation risk equals, in expectation, the population risk of a learner trained on $n-k$ samples, and that the first-order correction around uniform sample weights is the same for every $k$. A single centered trace therefore debiases the entire family.

For $I\subset[n]$ with $|I|=k$, let $\bm w^{(-I)}(T)$ be the terminal model after training on the dataset $S_{-I}$ with the $k$ points in $I$ removed and the remaining weights renormalized to $\tfrac{1}{n-k}$. Define the $k$-fold cross-validation risk
\begin{align}\label{eq:k_face_risk}
\mathcal L_{k}^\Psi(T,S_n)
&\triangleq
\binom{n}{k}^{-1}
\sum_{|I|=k}
\frac1k\sum_{i\in I} \Psi_{Z_i}\!\big(F(\bm w^{(-I)}(T),Z_i)-\bm y_i\big).
\end{align}
Setting $k=1$ recovers leave-one-out cross-validation; setting $k=n/K$ recovers standard $K$-fold cross-validation.

\begin{theorem}[Cross-validation as population risk]\label{thm:cv_family}
For every $k=1,\dots,n-1$,
\begin{align}\label{eq:cv_family}
\mathbb E_{S_n\sim\mathcal D^n}\!\left[\mathcal L_{k}^\Psi(T,S_n)\right]
=
\mathbb E_{S_{n-k}\sim\mathcal D^{n-k}}
\!\left[\mathcal L_{\mathcal D}^\Psi(\bm w_T(S_{n-k}))\right].
\end{align}
Thus every $k$-fold cross-validation risk is an exact sample-only population risk for the $(n{-}k)$-sample learner: any architecture, any loss, any optimizer.
\end{theorem}
\begin{proof}
By exchangeability, for any fixed $I$ with $|I|=k$,
\begin{align}
(S_{-I},(Z_i)_{i\in I})\stackrel{d}{=}(S_{n-k},Z_1',\dots,Z_k'),
\end{align}
where $S_{n-k}\sim\mathcal D^{n-k}$ and $Z_j'\sim\mathcal D$ are independent. Averaging over $i\in I$ gives the single-subset version; averaging over all $\binom{n}{k}$ subsets yields the result.
\end{proof}

The next proposition shows that the first-order term is the same for every $k$, so a single trace debiases the entire family.

\begin{proposition}[Common First Variation]\label{prop:common_first_variation}
Let $\alpha^{(-I)}\in\widetilde\Delta_n$ be the weight vector that places mass $\frac{n}{n-k}$ on each point $j\notin I$ and $0$ on each $j\in I$. The average direction from a holdout weight vector back to the uniform weights $\bar\alpha=\bm 1$, over all holdouts containing point $i$, is
\begin{align}\label{eq:common_first_variation}
\frac{1}{\binom{n-1}{k-1}}
\sum_{\substack{I\ni i\\ |I|=k}}
\bigl(\bar\alpha-\alpha^{(-I)}\bigr)
=
\frac{n}{n-1}\bm C_n \bm e_i
=
\bm\nu^{(i)},
\end{align}
the mass-preserving delete-one direction, independent of~$k$.
\end{proposition}
\begin{proof}
Fix $j\neq i$. Among the $\binom{n-1}{k-1}$ holdouts $I\ni i$, the $\binom{n-2}{k-2}$ that contain $j$ contribute $1$, and the $\binom{n-2}{k-1}$ that do not contribute $1-\tfrac{n}{n-k}=-\tfrac{k}{n-k}$. The binomial ratios $\binom{n-2}{k-2}/\binom{n-1}{k-1}=(k-1)/(n-1)$ and $\binom{n-2}{k-1}/\binom{n-1}{k-1}=(n-k)/(n-1)$ then give
\begin{align}
\frac{1}{\binom{n-1}{k-1}}
\sum_{\substack{I\ni i\\ |I|=k}}\bigl(\bar\alpha-\alpha^{(-I)}\bigr)_j
&=
\frac{k-1}{n-1}+\frac{n-k}{n-1}\!\left(-\frac{k}{n-k}\right)
=
-\frac{1}{n-1}.
\end{align}
The $i$-th component equals $1$ for every $I\ni i$, so the averaged vector has entry $1$ at position $i$ and $-\tfrac{1}{n-1}$ elsewhere,
\begin{align}
\frac{1}{\binom{n-1}{k-1}}\sum_{\substack{I\ni i\\ |I|=k}}\bigl(\bar\alpha-\alpha^{(-I)}\bigr)
&=
\frac{n}{n-1}\!\left(\bm e_i-\tfrac{1}{n}\bm 1\right)
=
\frac{n}{n-1}\bm C_n\bm e_i
=
\bm\nu^{(i)},
\end{align}
which is independent of $k$.
\end{proof}

Every $k$-fold cross-validation population risk is expanded around the same uniform weights $\bar\alpha$, and the averaged holdout direction matches for all $k$. The first-order Taylor expansion therefore produces a $k$-independent leading term and a remainder $R_k$,
\begin{align}\label{eq:universal_first_variation}
\mathcal L_{k}^\Psi(T,S)
&=
\mathcal L_{\mathrm{pop}}^\Psi(T,S)
+
R_k(T,S),
\\
\mathcal L_{\mathrm{pop}}^\Psi
&=
\widehat{\mathcal L}_S^\Psi
+
\frac{1}{n-1}\operatorname{tr}(\mathsf J_\Psi\bm C_n).
\end{align}
The centered influence trace $\operatorname{tr}(\mathsf J_\Psi\bm C_n)$ is the common first variation across leave-one-out, $K$-fold, and any exchangeable holdout mask; only the second-order remainder $R_k$ varies with~$k$.

\subsection{Variance Estimator and Algorithm Details}\label{sec:variance_estimator_details}

\paragraph{Intuition.} The kernel-block one-step rule of \autoref{sec:kernel_block_derivation},
\begin{align}
\Omega_B(M)
=
\bar{\bm g}_B^\top M\bar{\bm g}_B-\tfrac{1}{b-1}\operatorname{tr}(M\bm\Sigma_B),
\end{align}
yields a practical algorithm once we supply a streaming estimate of the per-parameter variance, a conversion between fresh-batch and finite-dataset regimes, and a soft relaxation of the gate.

\paragraph{Fresh-batch and finite-dataset regimes.}
The streaming variance estimator $\hat{\bm s}_t$ tracks the per-batch variance $(\bm\Sigma_B)_{kk}/(b-1)$. The translation to the full-dataset variance is the finite-population correction below.

\begin{theorem}[Minibatch covariance]\label{thm:minibatch_covariance}
Let $B \subset [n]$ be a size-$b$ subset sampled uniformly without replacement, $\bm g_B = b^{-1}\sum_{i \in B} \bm g_i$, and
\begin{align}
\bm\Sigma_g
=
\frac{1}{n}\sum_i (\bm g_i - \bar{\bm g})(\bm g_i - \bar{\bm g})^\top.
\end{align}
For any $\bm P \succeq 0$,
\begin{align}\label{eq:finite_pop_cov}
\operatorname{Cov}_B(\bm g_B)
&=
\frac{n-b}{b(n-1)} \bm\Sigma_g,
\\
\frac{1}{n-1}\operatorname{tr}(\bm P \bm\Sigma_g)
&=
\frac{b}{n-b}\operatorname{tr}\bigl(\bm P\operatorname{Cov}(\bm g_B)\bigr).
\end{align}
The streaming LOO coefficient is therefore $\alpha = 1$ in the fresh-batch (online) regime and $\alpha = b/(n-b)$ in the finite-dataset regime.
\end{theorem}

\begin{proof}
Write $\bm g_B - \bar{\bm g} = b^{-1}\sum_{i \in B}\bm c_i$ with $\bm c_i = \bm g_i - \bar{\bm g}$. Using $\Pr(i \in B) = b/n$ and $\Pr(i,j \in B) = b(b-1)/(n(n-1))$ for $i \ne j$, together with $\sum_{i \ne j} \bm c_i \bm c_j^\top = -\sum_i \bm c_i \bm c_i^\top$,
\begin{align}
\operatorname{Cov}(\bm g_B)
=
\frac{1}{b^2}\!\left[\frac{b}{n} - \frac{b(b-1)}{n(n-1)}\right]\sum_i \bm c_i \bm c_i^\top
=
\frac{n-b}{b(n-1)} \bm\Sigma_g.
\end{align}
\end{proof}

\paragraph{Streaming variance estimator.}
The exponential moving estimate
\begin{align}
\bm s_t
=
\rho \bm s_{t-1} + (1-\rho)(\bm g_t - \bm m_{t-1})^{\odot 2}
\end{align}
tracks the diagonal of $\operatorname{Var}(\bar{\bm g}_{B,k}) = (\bm\Sigma_B)_{kk}/(b-1)$ when consecutive minibatches are approximately i.i.d.\ and parameter drift between steps is small. In highly non-stationary settings, $\bm s_t$ can be replaced by the exact per-batch variance computed from per-example gradients via \texttt{vmap}, at the cost of one extra backward pass per step. This style of estimator follows the stochastic trace tradition of \citet{hutchinson1990stochastic}.

\paragraph{Hard, soft, and SNR gates.}
Three forms of the per-parameter gate, paired with the parameter update on top of an Adam preconditioner, are
\begin{align}\label{eq:soft_filter}
q_k^{\mathrm{hard}}
&=
\bm 1\{\hat m_k^2 > \alpha \hat s_k\},
\\
q_k^{\mathrm{soft}}
&=
\frac{(\hat m_k^2 - \alpha \hat s_k)_+}{(\hat m_k^2 - \alpha\hat s_k)_+ + \lambda_{\mathrm{pop}}\hat s_k + \varepsilon},
\\
q_k^{\mathrm{SNR}}
&=
\frac{\hat m_k^2}{\hat m_k^2 + \lambda\hat s_k + \varepsilon},
\\
\bm w_{t+1}
&=
\bm w_t - \eta_t \bm q_t \odot \frac{\hat{\bm m}_t}{\sqrt{\hat{\bm v}_t} + \epsilon} - \eta_t \lambda_{\mathrm{wd}} \bm w_t.
\end{align}
The hard gate is the unique binary rule keeping the first-order improvement nonnegative on every parameter (\autoref{thm:diagonal_gate}). The soft gate preserves that first-order safety while smoothing the cutoff: its numerator vanishes when $\hat m_k^2\le \alpha\hat s_k$, and its denominator is bounded below by $\lambda_{\mathrm{pop}}\hat s_k+\varepsilon$. The SNR shrinker used in prior work assigns positive weight even when $\hat m_k^2 < \alpha\hat s_k$.

The hyperparameter $\lambda_{\mathrm{pop}}$ plays the role of a dimensionless regularization scale. At the finite-dataset boundary,
\begin{align}
\lambda_{\mathrm{pop}} \hat{\bm s}_t
\approx
\tfrac{1}{n-1}\widehat{\bm\Sigma}_{g,t}
=
\tfrac{b}{n-b}\operatorname{Cov}(\bm g_B),
\end{align}
which is typically unnecessary at scale.

\subsection{Leave-One-Out Risk}\label{sec:loo_risk}\label{app:deferred_self_exclusion}

\paragraph{Intuition.} Leave-one-out treats each training point as a one-element held-out set: point $i$ is evaluated against the model trained without it. The same transfer operators $\mathsf G_{Q_i}$ that govern training also describe the delete-one prediction displacement, so the expected generalization gap reduces to a one-step loss difference between the original dataset and the one-point-replaced dataset.

Define the training residual, the delete-one displacement, and the leave-one-out exclusion risk,
\begin{align}
\bm e_i(T)
&\triangleq
\bm U_{Q_i}^{S}(T)-\bm y_i,
\\
\bm\Delta_i^{\mathrm{ex}}(T)
&\triangleq
\bm U_{Q_i}^{S_{-i}}(T)-\bm U_{Q_i}^{S}(T)
=
\mathsf G_{Q_i,S}(T)\bm g_S(0)
-
\mathsf G_{Q_i,S_{-i}}(T)\bm g_{S_{-i}}(0),
\\
\mathcal L_{\mathrm{ex}}^\Psi(T,S)
&=
\frac1n\sum_{i=1}^n
\Psi_{Z_i}\!\big(\bm e_i(T)+\bm\Delta_i^{\mathrm{ex}}(T)\big),
\end{align}
where the second line specializes the dataset-level transfer to $Q=Q_i$ on the two datasets $S$ and $S_{-i}$.

\begin{theorem}[Self-influence]\label{thm:self_influence}
Let $S^{(i)}$ denote the dataset with $Z_i$ replaced by an independent draw $Z_i'$, and define the generalization gap
\begin{align}
\gamma_\Psi(T,S)
&\triangleq
\mathcal L_{\mathcal D}^{\Psi}(\bm w_T(S))
-
\widehat{\mathcal L}_{S}^{\Psi}(\bm w_T(S)).
\end{align}
Then
\begin{align}\label{eq:exact_replace_one_gap}
\mathbb E[\gamma_\Psi(T,S)]
&=
\frac1n\sum_{i=1}^n
\mathbb E\!\left[
\Psi_{Z_i}\!\big(\bm U_{Q_i}^{S^{(i)}}(T)-\bm y_i\big)
-
\Psi_{Z_i}\!\big(\bm U_{Q_i}^{S}(T)-\bm y_i\big)
\right],
\\
\bm\Delta_i(T)
&=
\mathsf G_{Q_i,S}(T)\bm g_S(0)
-
\mathsf G_{Q_i,S^{(i)}}(T)\bm g_{S^{(i)}}(0). \label{eq:delta_i_loo}
\end{align}
The full version is in \autoref{sec:self_influence_full}.
\end{theorem}

\subsection{Algorithm Pseudocode}\label{app:deferred_algorithm}

\paragraph{Intuition.} The algorithm is Adam with one extra parameter-sized state vector tracking a running variance $\bm s$ of per-example gradients. The only deviation from a standard moment optimizer is a per-parameter gate that suppresses the update on parameter $k$ whenever the squared mean gradient $\hat m_k^2$ is below the variance threshold $\alpha \hat s_k$. This is the population-risk safe rule derived in the local one-step theory: parameters whose batch signal is dominated by leave-one-out noise contribute nothing positive to first-order population improvement, so they sit out the step.

\begin{algorithm}[t]
\caption{Population risk training via gradient leave-one-out}\label{alg:pop_risk}

\renewcommand{\arraystretch}{1.4}
\newcommand{\algindent}{\hspace{1.5em}}
\setlength{\tabcolsep}{4pt}

\rowcolors{2}{gray!10}{white}

\begin{tabularx}{\linewidth}{
  >{\tiny\color{gray!80!black}}r
  @{\hspace{1em}}
  >{\raggedright\arraybackslash}X
}
\multicolumn{2}{@{}p{\dimexpr\linewidth-2\tabcolsep\relax}}{
  \textbf{Require} Learning rate $\eta$, moments $(\beta_1,\beta_2)$, covariance decay $\rho$, LOO coefficient $\alpha$, population strength $\lambda_{\mathrm{pop}}$, weight decay $\lambda_{\mathrm{wd}}$.
} \\
\multicolumn{2}{@{}p{\dimexpr\linewidth-2\tabcolsep\relax}}{
  \textbf{Ensure} Updated parameters $\bm w_T$.
} \\
\midrule
1 & Initialize $\bm m, \bm v, \bm s \gets \bm 0, \bm 0, \bm 0$. \\
2 & \textbf{for} $t = 1, 2, \ldots, T$ \textbf{do} \\
3 & \algindent Sample minibatch $B_t$ and compute $\bm g_t \gets |B_t|^{-1}\sum_{i\in B_t} \nabla\ell_i(\bm w_{t-1})$. \\
4 & \algindent $\bm m_{\mathrm{prev}} \gets \bm m$. \\
5 & \algindent Update variance estimator: $\bm s \gets \rho \bm s + (1-\rho)(\bm g_t - \bm m_{\mathrm{prev}})^{\odot 2}$. \\
6 & \algindent Update first moment: $\bm m \gets \beta_1 \bm m + (1-\beta_1)\bm g_t$. \\
7 & \algindent Update second moment: $\bm v \gets \beta_2 \bm v + (1-\beta_2)\bm g_t^{\odot 2}$. \\
8 & \algindent Bias-correct $\hat{\bm m}, \hat{\bm v}, \hat{\bm s}$. \\
9 & \algindent Compute mask: $\bm\delta \gets (\hat{\bm m}^{\odot 2} - \alpha \hat{\bm s})_+$,\quad $\bm q \gets \bm\delta / (\bm\delta + \lambda_{\mathrm{pop}} \hat{\bm s} + \varepsilon)$. \\
10 & \algindent Update parameters: $\bm w_t \gets \bm w_{t-1} - \eta \bm q \odot \hat{\bm m} / (\sqrt{\hat{\bm v}} + \epsilon) - \eta \lambda_{\mathrm{wd}} \bm w_{t-1}$. \\
11 & \textbf{end for} \\
12 & \textbf{return} $\bm w_T$. \\
\end{tabularx}
\end{algorithm}

\section{Complexity Measure and Self-Influence}\label{app:complexity_self_influence}\label{app:full_signal_direction_theorem}

\subsection{Optimal Signal Directions for a Complexity Measure $R$}\label{app:optimal_signal_directions}

\paragraph{Intuition.} Test-invisibility identifies the directions in which training motion has no effect on test predictions. The signal channel contains every direction in which training reduced loss, including those that reflect transferable structure together with those that reflect memorization. A complexity measure $R\succ 0$ on output space resolves this within the signal channel: the top eigenspace of $R^{-1/2}\mathcal W_S R^{-1/2}$ is the worst-case-optimal way to read out signal from a training run, and the split between signal and memorization is determined by the choice of $R$ together with the cumulative dissipation Gramian $\mathcal W_S$, with the test set entering only through a scalar visibility constant. The graph-Laplacian metric $\Rsg = I+\gamma \widehat L_S$ is one valid choice; the self-influence metric $R_{S,\gamma,\beta}^a(s,T)$ defined below is another, and the theorem holds for any $R\succ 0$.

\begin{theorem}[Optimal signal directions: hard projectors]
\label{thm:sharp_semantic_phase}
\label{thm:train_only_minimax_signal}
Fix $0\le s\le T$, a test set $Q$, and a complexity measure $R\succ 0$ on
$\mathbb R^{np}$. Write $\bm W \triangleq \mathcal W_S(s,T)$ and
$C_R(s,T) \triangleq R^{-1/2} \bm W R^{-1/2}\succeq 0$, with spectral decomposition
$C_R(s,T)=\sum_{j=1}^{\rho}\lambda_j^R\phi_j^R(\phi_j^R)^\top$,
$\lambda_1^R\ge\cdots\ge\lambda_\rho^R>0$.
Let $\bm Q_r^R\triangleq \sum_{j\le r}\phi_j^R(\phi_j^R)^\top$ be the top-$r$
spectral projector, and recall the visibility Gramian $\Gamma_Q(s,T)$ from \autoref{def:test_transfer}.

\emph{For hard rank-$r$ subspaces}, define the projector class
\begin{align}
\mathcal P_r
\triangleq
\{\bm Q\in\mathbb R^{np\times np}:\ \bm Q^2=\bm Q,\ \bm Q^\top=\bm Q,\ \rank(\bm Q)=r\}
\end{align}
and, for $\bm Q\in\mathcal P_r$, the worst-case lost test motion
\begin{align}
\mathcal E_r(\bm Q)
\triangleq
\sup_{\substack{0\preceq \mathsf A\preceq \|\Gamma_Q(s,T)\|_{\op}C_R(s,T)}}
\|(I-\bm Q)\mathsf A(I-\bm Q)\|_{\op}.
\end{align}
Then
\begin{align}
\inf_{\bm Q\in\mathcal P_r}\mathcal E_r(\bm Q)
=
\|\Gamma_Q(s,T)\|_{\op}\lambda_{r+1}^R,
\label{eq:signal_filter_opt}
\end{align}
and the infimum is attained by the hard spectral projector $\bm Q=\bm Q_r^R$.
If $\lambda_r^R>\lambda_{r+1}^R$, the minimizer is \emph{unique among orthogonal projectors}.

\emph{Remark (Contractions admit multiple optima).}
For the relaxed class $\mathcal M_r = \{\bm M: 0\preceq \bm M\preceq I,\ \rank(\bm M)\le r\}$, several minimizers can attain the optimal value. For example, with $C_R = \operatorname{diag}(100,1)$ and $r=1$, both $\operatorname{diag}(1,0)$ and $\operatorname{diag}(0.9,0)$ achieve the same tail bound $\lambda_2 = 1$. The corresponding statement for contractions is the spectral relaxation in \autoref{thm:continuous_signal_filter}.
\end{theorem}

Here signal denotes directions where training dissipated the most loss per unit of complexity. The realized test-visible operator satisfies the spectral bound
\begin{align}
R^{-1/2}\bm G^\top \bm G R^{-1/2}
\preceq
\|\Gamma_Q(s,T)\|_{\op} C_R(s,T),
\end{align}
so that choosing $R=\Rsg$ recovers the graph-Laplacian case and the later choice $R=R_{S,\gamma,\beta}^a(s,T)$ incorporates self-influence within the same theorem. The signal basis is fixed before observing the test predictor, and the realized test operator refines that basis.

The hard projector restricts the controller class to idempotent operators. Allowing spectral contractions with a trace budget yields a spectral filter of $C_R(s,T)$ that attenuates small eigenvalues more strongly than large ones (\autoref{thm:continuous_signal_filter} in \autoref{sec:deferred_signal_parts}).

Three successive filters identify the directions that drive test error. The reservoir of low-dissipation directions is test-invisible. Among the remaining test-visible directions, leakage is captured by $\|R_\perp\|_{\op}$, the deviation of test motion from what training motion predicts. Among directions that are mobile and test-visible, those of high complexity rank low in $C_R(s,T)$ and are deprioritized by the signal theorem. Transferable signal remains after all three. The signal-channel partition records what training moved, and the complexity filter separates structure from nuisance within the signal channel.

\begin{figure}[t]
    \centering
    \safeincludegraphics[width=\textwidth]{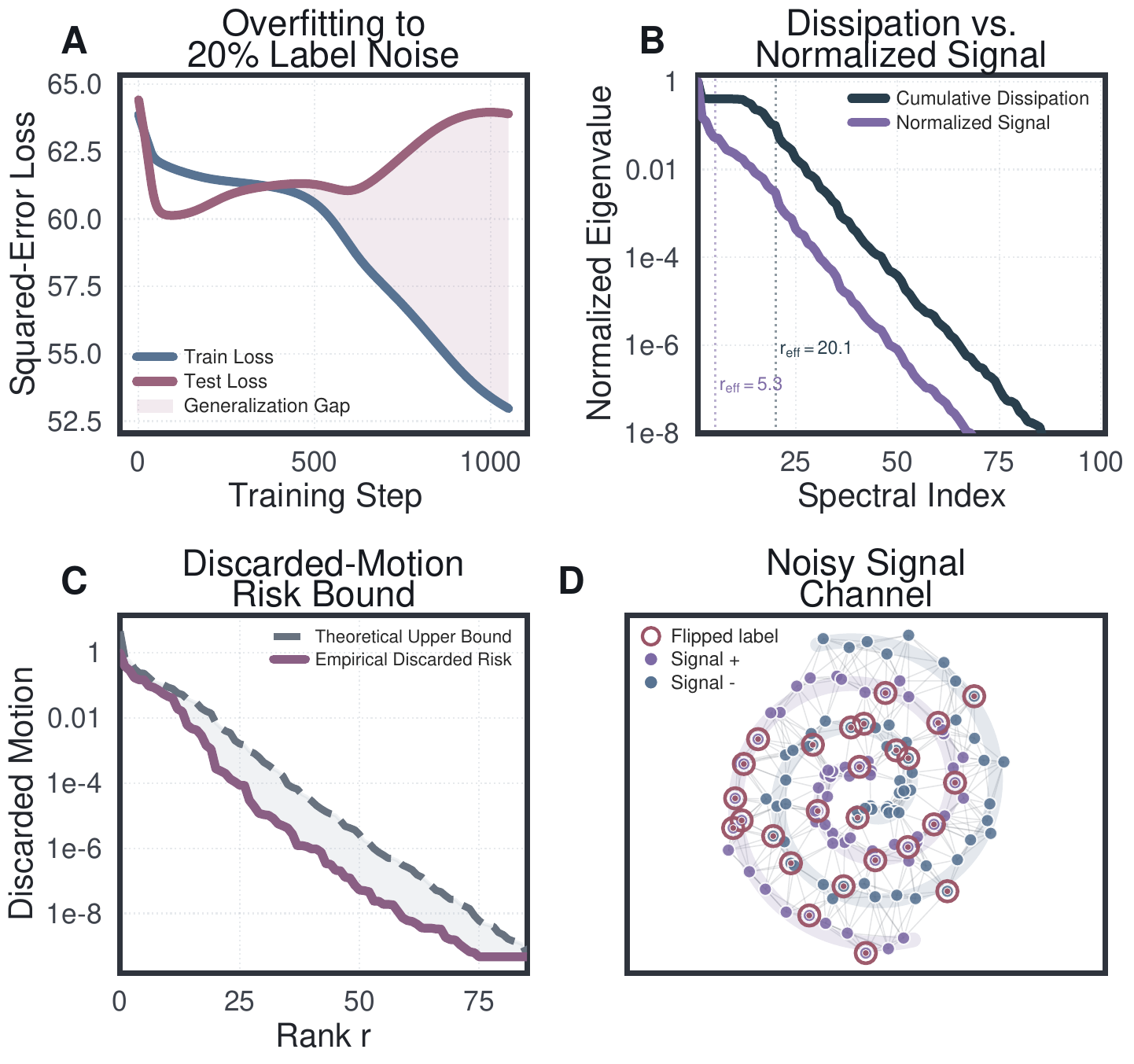}
    \caption{\textbf{Isolating the signal channel (\autoref{thm:train_only_minimax_signal}).}
    Evaluated on a dataset with $20\%$ label noise.
    \textbf{(A)}~Standard overfitting: training loss vanishes while test loss diverges.
    \textbf{(B)}~The raw cumulative dissipation $\lambda(\mathcal{W}_S)$ decays smoothly. Normalizing by manifold structure yields $C_R$ (purple), which drops by six orders of magnitude at effective rank $r_{\mathrm{eff}}=5.3$.
    \textbf{(C)}~The empirical worst-case lost test motion tightly tracks the theoretical bound $\|\Gamma_Q\|_{\op}\lambda_{r+1}^R$.
    \textbf{(D)}~The signal channel contains both clean structure and corrupted labels (circled), motivating the complexity measure $R$.}
    \label{fig:signal_spectrum}
\end{figure}

\subsection{Proofs for the Optimal Signal Directions Theorem}\label{sec:proofs_signal}

\paragraph{Intuition.}
Among rank-$r$ summaries of a training run, the one that retains the most test-relevant motion per unit of training cost is identified by an Eckart--Young-style argument: the optimal projector is given by the leading eigenvectors of $C_R(s,T)$, with the metric $R$ replacing the Euclidean inner product used in the classical statement. This appendix proves the optimal-signal-direction theorem of \autoref{app:full_signal_direction_theorem}.

\subsubsection{Deferred Parts of the Signal Principle}\label{sec:deferred_signal_parts}

\paragraph{Intuition.}
The optimal-signal-direction theorem combines two ingredients: a ceiling on how strongly the test predictor reacts along any direction (admissibility), and a tail bound that controls the test motion lost when only the top-$r$ directions are retained. We state both pieces here and then combine them in the proof. With the notation of \autoref{thm:train_only_minimax_signal}, let $\bm G\triangleq \mathsf G_Q(T,s)$ and define the $R$-orthogonal projector $\bm\Pi_r^R \triangleq R^{-1/2}\bm Q_r^R R^{1/2}$.

\paragraph{Admissible class.}
The realized test predictor satisfies
\begin{align}
R^{-1/2}\bm G^\top \bm G R^{-1/2}
\preceq
\|\Gamma_Q(s,T)\|_{\op} C_R(s,T).
\label{eq:minimax_signal_admissible}
\end{align}

\paragraph{Tail bound.}
For every $\bm h\in\mathbb R^{np}$ and every $r\ge 0$,
\begin{align}
\|\bm G(I-\bm\Pi_r^R)\bm h\|_2^2
\le
\|\Gamma_Q(s,T)\|_{\op}\lambda_{r+1}^R
\bm h^\top R(I-\bm\Pi_r^R)\bm h,
\label{eq:minimax_signal_tail}
\end{align}
with the convention $\lambda_{r+1}^R=0$ if $r\ge \rho$. Projecting into the tail eigenspace of $C_R(s,T)$ and applying the admissibility bound yields \eqref{eq:minimax_signal_tail} as a corollary of the admissible class.

\begin{remark}[Continuous relaxation]\label{thm:continuous_signal_filter}
Relaxing from rank-$r$ projectors to trace-bounded contractions $\{V:0\preceq V\preceq I,\ \tr(V)\le \tau\}$ yields a spectral filter of $C_R(s,T)$ that attenuates each eigendirection in proportion to its eigenvalue, diagonal in the eigenbasis of $C_R$, and fully suppressing small eigenvalues while preserving large ones. The hard rank-$r$ result above is the idempotent special case.
\end{remark}

\begin{proof}[Proof of \autoref{thm:train_only_minimax_signal}]
\textit{Admissible class.} \autoref{thm:reservoir_invisibility} gives $\bm G^\top \bm G \preceq \|\Gamma_Q(s,T)\|_{\op}\bm W$. Conjugating by $R^{-1/2}$ gives
\begin{align}
R^{-1/2}\bm G^\top \bm G R^{-1/2}
\preceq
\|\Gamma_Q(s,T)\|_{\op}R^{-1/2}\bm W R^{-1/2}
=
\|\Gamma_Q(s,T)\|_{\op}C_R(s,T),
\end{align}
which is \eqref{eq:minimax_signal_admissible}.

\textit{Tail bound.} Write $P^2\triangleq \|\Gamma_Q(s,T)\|_{\op}$, $C\triangleq C_R(s,T)$, $\bm Q\triangleq \bm Q_r^R$, $\bm\Pi\triangleq \bm\Pi_r^R$, and $\bm z\triangleq R^{1/2}\bm h$. Then $\bm G(I-\bm\Pi)\bm h=\bm G R^{-1/2}(I-\bm Q) \bm z$ and
\begin{align}
\|\bm G(I-\bm\Pi)\bm h\|_2^2
&=
\bm z^\top (I-\bm Q)R^{-1/2}\bm G^\top \bm G R^{-1/2}(I-\bm Q)\bm z \\
&\le
P^2 \bm z^\top (I-\bm Q)C(I-\bm Q)\bm z
\qquad\text{by \eqref{eq:minimax_signal_admissible}} \\
&\le
P^2\lambda_{r+1}^R \bm z^\top (I-\bm Q)\bm z
\qquad\text{since $\bm Q$ is the top-$r$ spectral projector of $C$} \\
&=
P^2\lambda_{r+1}^R \bm h^\top R(I-\bm\Pi)\bm h,
\end{align}
which proves \eqref{eq:minimax_signal_tail}.

\textit{Uniqueness.} Every rank-$r$ $R$-orthogonal projector can be written
\begin{align}
\bm\Pi = R^{-1/2} \bm Q R^{1/2},
\end{align}
where $\bm Q$ is an orthogonal rank-$r$ projector on the whitened space.
For a candidate transfer operator $\mathsf A$ from invisible to visible directions with
$0\preceq \mathsf A\preceq P^2 C$, the lost-motion ratio is
\begin{align}
\sup_{\bm h\neq 0}
\frac{\bm h^\top (I-\bm\Pi)^\top R^{1/2} \mathsf A R^{1/2}(I-\bm\Pi)\bm h}
     {\bm h^\top R \bm h}
=
\lambda_{\max}\!\big((I-\bm Q)\mathsf A(I-\bm Q)\big).
\end{align}
Therefore
\begin{align}
\sup_{\substack{0\preceq \mathsf A\preceq P^2 C}}
\sup_{\bm h\neq 0}
\frac{\bm h^\top (I-\bm\Pi)^\top R^{1/2} \mathsf A R^{1/2}(I-\bm\Pi)\bm h}
     {\bm h^\top R \bm h}
\le
P^2\lambda_{\max}\!\big((I-\bm Q)C(I-\bm Q)\big).
\label{eq:minimax_upper}
\end{align}

The bound \eqref{eq:minimax_upper} is attained. Pick a unit top eigenvector $\bm v$ of $(I-\bm Q)C(I-\bm Q)$ with eigenvalue
\begin{align}
\mu_{\bm Q} \triangleq \lambda_{\max}\!\big((I-\bm Q)C(I-\bm Q)\big),
\end{align}
and set
\begin{align}
\bm u
\triangleq
\frac{C^{1/2}\bm v}{\|C^{1/2}\bm v\|_2},
\qquad
\mathsf A_{\bm Q}
\triangleq
P^2 C^{1/2}\bm u\bm u^\top C^{1/2}.
\end{align}
Then $\bm u\bm u^\top\preceq I$ gives $0\preceq \mathsf A_{\bm Q}\preceq P^2 C$, and $\bm v\in \range(I-\bm Q)$ together with the definition of $\mathsf A_{\bm Q}$ yields
\begin{align}
\bm v^\top (I-\bm Q)\mathsf A_{\bm Q}(I-\bm Q)\bm v
&=
P^2\bm v^\top C^{1/2}\bm u\bm u^\top C^{1/2}\bm v \\
&=
P^2\|C^{1/2}\bm v\|_2^2 \\
&=
P^2\bm v^\top C \bm v \\
&=
P^2\mu_{\bm Q},
\end{align}
so the supremum in \eqref{eq:minimax_upper} equals $P^2\lambda_{\max}\!\big((I-\bm Q)C(I-\bm Q)\big)$.

Minimizing this expression over rank-$r$ projectors $\bm Q$, the Courant--Fischer theorem gives
\begin{align}
\inf_{\rank(\bm Q)=r}\lambda_{\max}\!\big((I-\bm Q)C(I-\bm Q)\big)&=\lambda_{r+1}^R,
\end{align}
with minimizer $\bm Q_r^R$, the top-$r$ spectral projector of $C$, so
\begin{align}
\bm\Pi_r^R &= R^{-1/2}\bm Q_r^R R^{1/2}
\end{align}
is the minimizing original-space projector and the optimal value is $P^2\lambda_{r+1}^R$. Uniqueness under the strict gap $\lambda_r^R>\lambda_{r+1}^R$ is the strict-gap case of Courant--Fischer.

\medskip
\noindent\textit{Extension to universal rank-$r$ linear channels
(\autoref{eq:signal_filter_opt}).}
Let $\bm M\in\mathcal M_r$, i.e.\ $0\preceq \bm M\preceq I$ with $\rank(\bm M)\le r$.
Write $C\triangleq C_R(s,T)$ and $P^2\triangleq \|\Gamma_Q(s,T)\|_{\op}$.
Since $0\preceq \mathsf A\preceq P^2 C$,
\begin{align}
\mathcal E_r(\bm M)
=
P^2\|(I-\bm M)C^{1/2}\|_{\op}^2.
\end{align}
Indeed,
\begin{align}
\sup_{\mathsf A}\sup_{\bm z}\frac{\bm z^\top(I-\bm M)\mathsf A(I-\bm M)\bm z}{\|\bm z\|^2}
&=\lambda_{\max}\bigl((I-\bm M)P^2 C(I-\bm M)\bigr)\\
&=P^2\|C^{1/2}(I-\bm M)\|_{\op}^2.
\end{align}
Now $C^{1/2}\bm M$ has rank at most $r$, so by the Eckart--Young--Mirsky
theorem for operator norm,
\begin{align}
\inf_{\rank(\bm M)\le r}\|C^{1/2}-C^{1/2}\bm M\|_{\op}
=
\inf_{\rank(\bm N)\le r}\|C^{1/2}-\bm N\|_{\op}
=
\sigma_{r+1}(C^{1/2})
=
\sqrt{\lambda_{r+1}^R},
\end{align}
where the second equality uses the substitution $\bm N=C^{1/2}\bm M$ (every rank-$r$ operator in $\range(C^{1/2})$ is attainable because the constraint $\bm M\preceq I$ is slack at the optimum). Hence
\begin{align}
\inf_{\bm M\in\mathcal M_r}\mathcal E_r(\bm M)
=
P^2\lambda_{r+1}^R.
\end{align}

For the matching lower bound, fix any $\bm M$ with $0\preceq \bm M\preceq I$ and $\rank(\bm M)\le r$. Courant--Fischer gives $\dim(\ker \bm M)\ge np-r$, so $\ker \bm M$ intersects the bottom $(np-r)$-dimensional eigenspace of $C$ nontrivially. Pick
\begin{align}
\bm z
\in
\ker \bm M\cap\operatorname{span}\{\phi_{r+1}^R,\dots,\phi_\rho^R\},
\qquad
\|\bm z\| = 1.
\end{align}
Then $(I-\bm M)\bm z=\bm z$ and
$\bm z^\top C \bm z\ge\lambda_{r+1}^R$, giving
$\mathcal E_r(\bm M)\ge P^2\lambda_{r+1}^R$.

The infimum is attained by $\bm M=\bm Q_r^R$: since $\bm Q_r^R$ is an orthogonal
projector with $0\preceq \bm Q_r^R\preceq I$ and $\rank(\bm Q_r^R)=r$, it belongs
to $\mathcal M_r$, and the upper bound is achieved with equality.
\end{proof}

\subsection{Self-Influence and Adjoint Reweighting}\label{sec:proofs_self_influence}

\paragraph{Intuition.}
The population-risk gap can be read off from a single training run by examining how each training point would have responded to its own removal: each example carries a self-influence whose average over the dataset equals the generalization gap. This appendix proves \autoref{thm:self_influence_full} and the reweighting sensitivity formula, extending classical influence functions \citep{cook1982residuals,koh2017understanding} to feature-learning trajectories and connecting with the leave-one-out stability tradition of \citet{bousquet2002stability,hardt2016train}. The full statement contains three pieces: the population-risk gap equals an average over training points of a one-shot loss difference; when the loss is differentiable, that loss difference becomes an integral over a prediction-displacement vector; and when the training objective is separable per example, the displacement vector decomposes into a content-change part and an observation-channel-change part. We state these as a base theorem and two corollaries so that each piece can be cited on its own.

\subsubsection{Full Statement of the Self-Influence Theorem}\label{sec:self_influence_full}

\begin{theorem}[Self-Influence, base form]\label{thm:self_influence_full}
Under the notation of \autoref{sec:population_risk},
\begin{align}
\mathbb E[\gamma_\Psi(T,S)]
=
\frac1n\sum_{i=1}^n
\mathbb E\!\left[
\Psi_{Z_i}\!\big(\bm U_{Q_i}^{S^{(i)}}(T)-\bm y_i\big)
-
\Psi_{Z_i}\!\big(\bm U_{Q_i}^{S}(T)-\bm y_i\big)
\right].
\end{align}
\end{theorem}

\begin{corollary}[Integral representation of the gap]\label{cor:self_influence_integral}
If each $\Psi_z$ is differentiable, define the prediction displacement
\begin{align}
\bm\Delta_i(T)
&=
\mathsf G_{Q_i,S}(T)\bm g_S(0)
-
\mathsf G_{Q_i,S^{(i)}}(T)\bm g_{S^{(i)}}(0).
\end{align}
Then the gap admits the integral representation
\begin{align}
\mathbb E[\gamma_\Psi(T,S)]
&=
\frac1n\sum_{i=1}^n
\mathbb E\!\left[
\int_0^1
\left\langle
\nabla\Psi_{Z_i}\!\big(\bm e_i(T)+\theta\bm\Delta_i(T)\big),
\bm\Delta_i(T)
\right\rangle
d\theta
\right].\label{eq:exact_self_influence_integral}
\end{align}
\end{corollary}

\begin{corollary}[Displacement under a separable objective]\label{cor:self_influence_separable}
Under the separable training objective of \autoref{thm:replace_one_stability},
\begin{align}
\bm\Delta_i(T)
&=
\mathsf G_{Q_i,S}^{(i)}(T)\bm\delta_i
+
\bigl(\mathsf G_{Q_i,S}(T)-\mathsf G_{Q_i,S^{(i)}}(T)\bigr)\bm g_{S^{(i)}}(0), \label{eq:delta_i_loo_separable}
\end{align}
where $\bm\delta_i$ is the initial gradient difference of the replaced example.
\end{corollary}

\begin{remark}[Content vs.\ channel split]\label{rem:self_influence_content_channel}
Each dataset induces its own train-test decomposition, so the displacement splits into a content-change term and a channel-change term. Writing
\begin{align}
\bm z_S &= \bm W_S^{1/2}\bm g_S(0),
&
\bm C_{i,S} &= \mathsf G_{Q_i,S}(T)\bm W_S^{\dagger/2},
\intertext{with the same definitions for $S^{(i)}$,}
\bm\Delta_i(T) &= \bm C_{i,S}(\bm z_S - \bm z_{S^{(i)}}) + (\bm C_{i,S} - \bm C_{i,S^{(i)}}) \bm z_{S^{(i)}}.
\end{align}
The first term measures how much the learned content changed; the second measures how much the observation channel for point $i$ changed. The generalization gap is the average loss response to both.
\end{remark}

\begin{proof}[Proof of \autoref{thm:self_influence}]
By the law of the unconscious statistician and exchangeability $(S^{(i)},Z_i)\stackrel{d}{=}(S,Z_i')$,
\begin{align}
\mathbb E\big[\mathcal L_{\mathcal D}^{\Psi}(\bm w_T(S))\big]
&=
\frac1n\sum_{i=1}^n
\mathbb E\!\left[
\Psi_{Z_i}\!\big(F(\bm w_T(S^{(i)}),Z_i)-\bm y_i\big)
\right],
\\
\mathbb E\!\left[\widehat{\mathcal L}_{S}^{\Psi}(\bm w_T(S))\right]
&=
\frac1n\sum_{i=1}^n
\mathbb E\!\left[
\Psi_{Z_i}\!\big(F(\bm w_T(S),Z_i)-\bm y_i\big)
\right],
\end{align}
and subtracting gives \eqref{eq:exact_replace_one_gap}. For \eqref{eq:exact_self_influence_integral}, apply the fundamental theorem of calculus to $\theta\mapsto \Psi_{Z_i}\!\big(\bm e_i(T)+\theta\bm\Delta_i(T)\big)$. Finally, \eqref{eq:delta_i_loo} follows from $\bm U_Q^S(T)=\bm U_Q(\bm w_0;Q)-\mathsf G_{Q,S}(T)\bm g_S(0)$ applied to $Q=Q_i$ and the two datasets $S$ and $S^{(i)}$; \eqref{eq:delta_i_loo_separable} is \autoref{thm:replace_one_stability} with $Q=Q_i$.
\end{proof}

\subsubsection{Self-Influence from the Training Displacement}\label{sec:self_influence_decomp}

\paragraph{Intuition.}
Self-influence (\autoref{thm:self_influence}) tracks how each training point's loss responds to its own removal, while the signal spectrum (\autoref{app:full_signal_direction_theorem}) tracks how training motion translates to test-visible displacement per unit data roughness. Both quantities arise from the same train-displacement operator, so a single forward pass and a single backward pass produce both at once.

\begin{proposition}[Self-influence decomposition]\label{prop:self_influence_decomp}
Fix an observed window $0\le s\le T$ and write
\begin{align}
\bm D\triangleq \mathsf D_S(T,s)
=\int_s^T \Kss(\tau)\mathcal P_g(\tau,s) d\tau.
\end{align}
For covectors $\bm a=(\bm a_1,\dots,\bm a_n)$ with $\bm a_i\in\mathbb R^p$, define the block contraction
\begin{align}
\bm A_{\bm a}:\mathbb R^{np}\to\mathbb R^n,
\\
\bm A_{\bm a}(\bm x_1;\dots;\bm x_n)
\triangleq
(\langle \bm a_1,\bm x_1\rangle,\dots,\langle \bm a_n,\bm x_n\rangle).
\end{align}
Let $Q_i=\{z_i\}$ denote the singleton containing the $i$-th training point. Then for every $\bm h\in\mathbb R^{np}$,
\begin{align}
(\bm A_{\bm a}\bm D\bm h)_i
=
\langle \bm a_i,\mathsf G_{Q_i}(T,s)\bm h\rangle,
\qquad i=1,\dots,n.
\end{align}
If $\bm a_i=\nabla\Psi_{z_i}(\bm e_i(s))$ with $\bm e_i(s)=\bm U_{Q_i}(s)-\bm y_i$, then $\bm D^\top \bm A_{\bm a}^\top \bm A_{\bm a}\bm D$ is the exact quadratic expression for self-influence through $\bm D$ on the observed window.
\end{proposition}

\autoref{prop:self_influence_decomp} shows that the quadratic $\bm D^\top \bm A_{\bm a}^\top \bm A_{\bm a}\bm D$ shares its forward and backward solves with the transfer operators of \autoref{thm:forward_backward}, so the same ODE solves suffice. The covectors $\bm a_i=\nabla\Psi_{z_i}(\bm e_i(s))$ are available at the window start $s$, so the construction uses only quantities observed before the terminal state.

\begin{definition}[Self-influence transfer]\label{def:self_influence_metric}
For $\beta\ge 0$, define the centered self-influence metric and the associated self-influence-weighted signal operator,
\begin{align}
R_{S,\gamma,\beta}^a(s,T) &\triangleq\Rsg+\beta \bm D^\top \bm A_{\bm a}^\top \bm C_n \bm A_{\bm a}\bm D,\\
T_{Q,\gamma,\beta}^a(s,T) &\triangleq R_{S,\gamma,\beta}^a(s,T)^{-1/2}\mathsf G_Q(T,s)^\top\mathsf G_Q(T,s)R_{S,\gamma,\beta}^a(s,T)^{-1/2},
\end{align}
where $\bm C_n=\bm I-\frac{1}{n}\bm 1\bm 1^\top$ is the centering projector. The centering penalizes per-example deviations from the batch mean (the centered directions $\bm C_n \bm e_i$) and preserves the common mode of uniform shifts across all examples. Since $R_{S,\gamma,\beta}^a\succ 0$, \autoref{thm:train_only_minimax_signal} applies with $R=R_{S,\gamma,\beta}^a$: the top-$r$ eigenspaces of $C_{R_{S,\gamma,\beta}^a}(s,T)$ are the unique rank-$r$ subspaces minimizing worst-case lost test motion, and the Eckart--Young--Mirsky and Courant--Fischer bounds carry over with $\Rsg$ replaced by $R_{S,\gamma,\beta}^a$. The case $\beta=0$ recovers the graph-Laplacian setup of \autoref{app:full_signal_direction_theorem}.
\end{definition}

Together these objects describe the training process at four levels: the dynamics determine which directions training moves; the transfer operator determines what reaches test outputs; the signal spectrum determines which fast directions carry transferable structure; and self-influence supplies the observable scores that make all three measurable. The same centering projector $\bm C_n$ appears in both the self-influence metric and the population-risk objective
\begin{align}
\mathcal L_{\mathrm{pop}}^\Psi = \widehat{\mathcal L}_S^\Psi + \frac{1}{n-1}\operatorname{tr}(\mathsf J_\Psi \bm C_n)
\end{align}
of \autoref{sec:population_risk}, so the signal spectrum and the population-risk correction share a common data-weight centering.

\begin{proof}[Proof of \autoref{prop:self_influence_decomp}]
Since $\bm D=\mathsf G_S(T,s)$ is the transfer operator with test set equal to the full training set, the $i$-th $p$-block of $\bm D\bm h$ is $\mathsf G_{Q_i}(T,s)\bm h$. Contracting that block against $\bm a_i$ gives the first display. The quadratic expression follows from the chain rule.
\end{proof}

\subsubsection{Stability Under Example Replacement}\label{sec:replace_one}

\paragraph{Intuition.}
Resampling one training point and rerunning training from the same initialization shifts the test prediction by the sum of two transparent quantities: the initial-gradient effect of the replaced example, and a drift in the transfer operator between the two neighboring datasets. This decomposition matches the algorithmic stability framework of \citet{bousquet2002stability,hardt2016train} and turns those uniform stability bounds into computable, run-specific quantities.

For a dataset $S$ and test set $Q$, define the transfer operator
\begin{align}\label{eq:transfer_operator_dataset}
\mathsf G_{Q,S}(T)\triangleq \int_0^T \Kqs^S(\tau)\mathcal P_{g,S}(\tau,0) d\tau,
\end{align}
so that \autoref{thm:output_dynamics} gives
\begin{align}\label{eq:transfer_dataset_repn}
\bm U_Q^S(T)=\bm U_Q(\bm w_0;Q)-\mathsf G_{Q,S}(T)\bm g_S(0).
\end{align}
Write $\mathsf G_{Q,S}^{(i)}(T):\mathbb R^p\to\mathbb R^{|Q|p}$ for the $i$-th $p$-column block of $\mathsf G_{Q,S}(T)$.

\begin{lemma}[Stability Under Example Replacement]\label{thm:replace_one_stability}
Fix $T\ge 0$ and a test set $Q$. Let $S' = S^{(i\leftarrow z')}$ and train both datasets from the same initialization $\bm w_0$.

For a general training objective $\Phi_S$,
\begin{align}\label{eq:replace_one_exact_general}
\bm U_Q^S(T)-\bm U_Q^{S'}(T)
=
-\mathsf G_{Q,S}(T)\bigl(\bm g_S(0)-\bm g_{S'}(0)\bigr)
-
\bigl(\mathsf G_{Q,S}(T)-\mathsf G_{Q,S'}(T)\bigr)\bm g_{S'}(0).
\end{align}

If, in addition, the training objective is the empirical average of a per-example loss with $\phi(\cdot;z)$ convex and $C^2$ for every $z$, then
\begin{align}
\Phi_S(\bm u)&=\frac1n\sum_{j=1}^n \phi(u_j;z_j),
\\
\bm g_S(0)-\bm g_{S'}(0)&=\bm e_i\otimes \bm\delta_i,
\\
\bm\delta_i&\triangleq
\frac1n\Bigl(\nabla_{\bm u}\phi(F(\bm w_0,z_i);z_i)-\nabla_{\bm u}\phi(F(\bm w_0,z');z')\Bigr),
\intertext{and so}
\bm U_Q^S(T)-\bm U_Q^{S'}(T)
&=
-\mathsf G_{Q,S}^{(i)}(T)\bm\delta_i
-
\bigl(\mathsf G_{Q,S}(T)-\mathsf G_{Q,S'}(T)\bigr)\bm g_{S'}(0),\label{eq:replace_one_exact_sparse}
\\
\|\bm U_Q^S(T)-\bm U_Q^{S'}(T)\|_2
&\le
\|\mathsf G_{Q,S}^{(i)}(T)\|_{\op}\|\bm\delta_i\|_2
+
\|\mathsf G_{Q,S}(T)-\mathsf G_{Q,S'}(T)\|_{\op}\|\bm g_{S'}(0)\|_2.\label{eq:replace_one_bound}
\end{align}
\end{lemma}

\begin{proof}
By \autoref{thm:output_dynamics},
\begin{align}
\bm U_Q^S(T)&=\bm U_Q(\bm w_0;Q)-\mathsf G_{Q,S}(T)\bm g_S(0),
\\
\bm U_Q^{S'}(T)&=\bm U_Q(\bm w_0;Q)-\mathsf G_{Q,S'}(T)\bm g_{S'}(0),
\intertext{and subtracting yields \eqref{eq:replace_one_exact_general}. Under the separable empirical loss only the replaced block changes at time $0$, so}
\bm g_S(0)-\bm g_{S'}(0)&=\bm e_i\otimes \bm\delta_i.
\end{align}
Substituting into \eqref{eq:replace_one_exact_general} gives \eqref{eq:replace_one_exact_sparse}, and \eqref{eq:replace_one_bound} follows from the triangle inequality.
\end{proof}

The two terms in \eqref{eq:replace_one_bound} read off directly. The first measures the effect of the replaced example's initial gradient, and the second measures the transfer-operator drift between the two neighboring datasets. For any test loss with Lipschitz constant $L_Q$ in prediction space, multiplying the right-hand side of \eqref{eq:replace_one_bound} by $L_Q$ yields a replace-one stability bound \citep{bousquet2002stability}. The quantity $\|\mathsf G_{Q,S}^{(i)}(T)\|_{\op}$ is an example-specific, test-specific influence score controlled by test-invisibility (\autoref{sec:third_law}): directions in $\ker\Kqs$ leave $\mathsf G_{Q,S}$ unchanged, so noise-trapped gradient stays clear of test predictions.

\subsubsection{Reweighting Sensitivity and Generalization Estimation}\label{sec:self_influence_details}

\paragraph{Intuition.}
The integral expression for the gap involves the transfer-operator difference $\bm\Delta_i(T)$, which would require retraining for every $i$. The first variation of the test functional under an infinitesimal reweighting of a training point, captured by a single backward sensitivity solve along the realized run, gives the same information without retraining, extending classical influence functions \citep{cook1982residuals,koh2017understanding} to the full feature-learning trajectory. From that one solve the entire self-influence matrix is read off, and its centered trace is an unbiased estimate of the population-risk gap.

\begin{theorem}[One-Run Reweighting Sensitivity Formula]\label{thm:adjoint_reweighting}
Assume the training objective is separable:
\begin{align}
L_S(\bm w)=\frac1n\sum_{j=1}^n \ell_j(\bm w),
\\
\ell_j(\bm w)\triangleq\phi(F(\bm w,z_j);z_j),
\end{align}
with each $\ell_j\in C^2$. Fix a test set $Q$ and a $C^1$ test functional
$\psi:\mathbb R^{|Q|p}\to\mathbb R$. For a reweighting direction $\bm\nu\in\mathbb R^n$,
let $\bm w^\lambda$ solve gradient flow for
\begin{align}
L_S^\lambda(\bm w)
\triangleq
\frac1n\sum_{j=1}^n (1-\lambda \nu_j)\ell_j(\bm w),
\\
\bm w^\lambda(0)=\bm w_0.
\end{align}
Assume $\lambda\mapsto \bm w^\lambda(T)$ is differentiable at $\lambda=0$, and write
$\bm w(t)\triangleq\bm w^0(t)$. Then
\begin{align}
\left.\frac{d}{d\lambda}\right|_{\lambda=0}\psi(\bm U_Q(\bm w^\lambda(T)))
=
\frac1n\int_0^T
\left\langle
\bm p_{\psi,Q}(\tau),
\sum_{j=1}^n \nu_j \nabla_{\bm w} \ell_j(\bm w(\tau))
\right\rangle
d\tau,
\end{align}
where $\bm p_{\psi,Q}$ is the backward sensitivity along the realized run:
\begin{align}
\partial_\tau \bm p_{\psi,Q}(\tau)
=
\nabla_{\bm w}^2 L_S(\bm w(\tau))\bm p_{\psi,Q}(\tau),
\\
\bm p_{\psi,Q}(T)=\bm J_Q(\bm w(T))^\top \nabla\psi(\bm U_Q(\bm w(T))).
\end{align}

Equivalently, if $\bm v_{\bm\nu}$ solves the forward sensitivity equation
\begin{align}
\partial_\tau \bm v_{\bm\nu}(\tau)
=
-\nabla_{\bm w}^2 L_S(\bm w(\tau))\bm v_{\bm\nu}(\tau)
+
\frac1n\sum_{j=1}^n \nu_j \nabla_{\bm w} \ell_j(\bm w(\tau)),
\\
\bm v_{\bm\nu}(0)=0,
\end{align}
then
\begin{align}
\left.\frac{d}{d\lambda}\right|_{\lambda=0}\psi(\bm U_Q(\bm w^\lambda(T)))
=
\bigl\langle
\nabla\psi(\bm U_Q(\bm w(T))),
\bm J_Q(\bm w(T))\bm v_{\bm\nu}(T)
\bigr\rangle.
\end{align}
\end{theorem}

\begin{proof}
Differentiating the weighted flow at $\lambda=0$ gives the forward sensitivity equation, and the chain rule yields the terminal formula
\begin{align}
\left.\frac{d}{d\lambda}\right|_{\lambda=0}\psi(\bm U_Q(\bm w^\lambda(T)))
=
\bigl\langle
\nabla\psi(\bm U_Q(\bm w(T))),
\bm J_Q(\bm w(T))\bm v_{\bm\nu}(T)
\bigr\rangle.
\end{align}
Set $\bm H(\tau)\triangleq \nabla_{\bm w}^2L_S(\bm w(\tau))$. Combining
\begin{align}
\partial_\tau \bm v_{\bm\nu}(\tau)&=-\bm H(\tau)\bm v_{\bm\nu}(\tau)
+\frac1n\sum_{j=1}^n \nu_j \nabla_{\bm w}\ell_j(\bm w(\tau)),
\\
\partial_\tau \bm p_{\psi,Q}(\tau)&=\bm H(\tau)\bm p_{\psi,Q}(\tau),
\intertext{the cross-term cancels and}
\frac{d}{d\tau}\langle \bm p_{\psi,Q}(\tau),\bm v_{\bm\nu}(\tau)\rangle
&=
\frac1n
\left\langle
\bm p_{\psi,Q}(\tau),
\sum_{j=1}^n \nu_j \nabla_{\bm w}\ell_j(\bm w(\tau))
\right\rangle.
\end{align}
Integrating over $[0,T]$, using $\bm v_{\bm\nu}(0)=0$, and inserting the terminal condition for
$\bm p_{\psi,Q}$ gives the sensitivity formula.
\end{proof}

Entry $\mathsf J_{ij}$ measures how much the test loss at point $i$ would move under a small downweighting of training point $j$. The centered trace $\operatorname{tr}(\mathsf J\bm C_n)$ keeps only relative changes among examples.

\begin{definition}[Self-Influence Matrix]\label{def:self_influence_matrix}
Under the hypotheses of \autoref{thm:adjoint_reweighting}, fix $i,j\in\{1,\dots,n\}$, write $Q_i=\{z_i\}$, and let $\bm w^{(j,\lambda)}$ denote the weighted flow corresponding to $\bm\nu=\bm e_j$. Set
\begin{align}
\psi_i(\bm u)&=\Psi_{Z_i}(\bm u-\bm y_i),
\\
\bm e_i(T)&=\bm U_{Q_i}(\bm w(T))-\bm y_i.
\intertext{The \emph{self-influence matrix} is}
\mathsf J_\Psi(T,S)
&=
\bigl[\mathsf J_{ij}(T,S)\bigr]_{i,j=1}^n,
\\
\mathsf J_{ij}(T,S)
&\triangleq
\left.\frac{d}{d\lambda}\right|_{\lambda=0}
\psi_i\!\bigl(\bm U_{Q_i}(\bm w^{(j,\lambda)}(T))\bigr).
\end{align}
\end{definition}

A single backward sensitivity solve produces every entry $\mathsf J_{ij}$ along the realized run, and its centered trace estimates population risk from one training pass.

\begin{theorem}[Self-influence matrix from a backward solve]
\label{thm:self_influence_backward}
Under the notation of \autoref{def:self_influence_matrix}, with $\bm J_i(\bm w(T))\triangleq D_wF(\bm w(T),z_i)$ and $\bm p_i$ solving the backward sensitivity equation
\begin{align}
\partial_\tau \bm p_i(\tau)
&=
\nabla_{\bm w}^2L_S(\bm w(\tau))\bm p_i(\tau),
\\
\bm p_i(T)&=\bm J_i(\bm w(T))^\top \nabla\Psi_{Z_i}(\bm e_i(T)),
\intertext{the self-influence matrix entry is}
\mathsf J_{ij}(T,S)
&=
\frac1n\int_0^T
\langle \bm p_i(\tau),\nabla_{\bm w}\ell_j(\bm w(\tau))\rangle d\tau.
\end{align}
\end{theorem}

\begin{corollary}[Generalization estimator from the diagonal]\label{cor:self_influence_diagonal}
Under the same hypotheses, the centered trace estimator
\begin{align}
\widehat\Gamma_{\mathrm{adj}}(T,S)
&\triangleq
\frac1n\sum_{i=1}^n \mathsf J_{ii}(T,S)
=
\frac1n\tr\mathsf J_\Psi(T,S)
\end{align}
collects the diagonal entries of the self-influence matrix. Set
\begin{align}
q_i(\lambda) \triangleq \psi_i\!\bigl(\bm U_{Q_i}(\bm w^{(i,\lambda)}(T))\bigr).
\end{align}
If $q_i$ is twice differentiable on $[0,1]$, Taylor with integral remainder gives
\begin{align}
q_i(1)-q_i(0)
=
\mathsf J_{ii}(T,S)
+
\int_0^1 (1-\lambda)q_i''(\lambda) d\lambda.
\end{align}
\end{corollary}

\section{Frozen-Kernel Limit and Classical Phenomena}\label{sec:classical_phenomena}

\paragraph{Intuition.}
Classical generalization phenomena (benign overfitting, double descent, implicit bias, grokking, ridge shrinkage) are usually told as separate stories. We show here that, in the frozen-kernel limit, each is a different choice of the spectral filter $\bm M$ in \autoref{thm:unified_bias_variance}. Picking $\bm M$ selects the phenomenon; the underlying decomposition is unchanged.

In the frozen-kernel limit \citep{jacot2018neural}, the signal-channel-and-reservoir picture recovers the classical bias-variance split. Benign overfitting, double descent, implicit bias, and grokking then appear as four choices of the spectral filter $\bm M$ in \autoref{thm:unified_bias_variance}. This appendix collects those instances and shows how each classical phenomenon falls out of a single decomposition.

\subsection{Unified Bias--Variance Decomposition}\label{sec:applications}

\paragraph{Intuition.}
\autoref{thm:unified_bias_variance} below states an exact bias--variance split for any self-adjoint contraction $\bm M$: the bias is a quadratic form in the unfit signal direction and the variance is a trace against the noise covariance. Specializing $\bm M$ to a gradient-flow filter recovers implicit bias and grokking; to a ridge filter, ridge regression; to a hard threshold, the bias--variance trade-off; to a rank truncation, benign overfitting and double descent.

In the frozen-kernel, squared-loss regime, every choice of spectral filter $\bm M$ produces a single bias-variance decomposition. The bias is a quadratic form in the unfit signal direction, and the variance is a trace against the noise covariance. Picking $\bm M$ to be a gradient-flow filter, a ridge filter, a hard threshold, or a rank truncation reproduces, respectively, implicit bias and grokking, ridge regression, the bias-variance trade-off, and benign overfitting plus double descent.

Fix the initial configuration $\bm w_0$ and let $\bm\Sigma=\diag(\sigma_1,\dots,\sigma_r)$ collect the positive singular values of $\bm J_S(\bm w_0)$. Bundle the SVD, kernels, the test map on the mobile singular space, and the normalized visibility Gramian:
\begin{align}
\bm J_S(\bm w_0) &= \bm U\bm\Sigma \bm V^\top,
\\
\bm K_0 &= \Kss(\bm w_0)=\bm U\bm\Sigma^2\bm U^\top,
\\
\bm H_0 &= \Kqs(\bm w_0),
\\
\bm C_0 &\triangleq \bm J_Q(\bm w_0)\bm V\bm\Sigma^{-1}\in\mathbb R^{|Q|p\times r},
\\
\bm\Gamma_0 &\triangleq \bm C_0^\top \bm C_0
= \bm\Sigma^{-1}\bm V^\top \bm J_Q(\bm w_0)^\top \bm J_Q(\bm w_0)\bm V\bm\Sigma^{-1}\succeq 0.
\end{align}
The Gramian $\bm\Gamma_0$ is the frozen-kernel limit of the dynamic visibility operator $\Gamma_Q(s,T)\triangleq \psinv{W}G^\top G \psinv{W}$ from \autoref{def:test_transfer}. For any self-adjoint contraction $0\preceq \bm M\preceq \bm I_r$, the filtered output and the corresponding parameter displacement read
\begin{align}
\bm U_Q^{\bm M} &\triangleq \bm U_Q(0) + \bm C_0 \bm M \bm U^\top(\bm y-\bm U_S(0)),
\\
\bm\delta_{\bm M} &\triangleq -\bm V\bm\Sigma^{-1}\bm M\bm U^\top(\bm U_S(0)-\bm y),
\\
\intertext{and the five filters of interest, gradient flow at time $t$, ridge, hard thresholding, predictive rank truncation, and full interpolation, are}
\bm M_t &= \bm I-e^{-t\bm\Sigma^2/n},
\\
\bm M_\eta &= \bm\Sigma^2(\bm\Sigma^2+n\eta \bm I)^{-1},
\\
\bm M_\tau &= \mathbf 1_{[\tau,\infty)}(\bm\Sigma^2),
\\
\bm M &= \bm P_r,
\\
\bm M &= \bm I_r.
\end{align}

A single bias--variance decomposition governs every $\bm M$.

\begin{theorem}[Unified Bias--Variance Decomposition]\label{thm:unified_bias_variance}
Assume $\bm y=\bar{\bm y}+\bm\xi$ with $\mathbb E[\bm\xi\mid S]=0$, and write
\begin{align}
\bar{\bm a} &\triangleq \bm U^\top(\bar{\bm y}-\bm U_S(0)),
&
\bm\zeta &\triangleq \bm U^\top \bm\xi,
&
\bm\Sigma_{\bm\zeta} &\triangleq \operatorname{Cov}(\bm\zeta\mid S),
&
\bar{\bm U}_Q &\triangleq \bm U_Q(0)+ \bm C_0 \bar{\bm a}.
\end{align}
For every self-adjoint contraction $0\preceq \bm M\preceq \bm I_r$,
\begin{align}
\bm U_Q^{\bm M}-\bar{\bm U}_Q &= -\bm C_0(\bm I-\bm M)\bar{\bm a} + \bm C_0\bm M\bm\zeta,
\\
\mathbb E\!\left[\|\bm U_Q^{\bm M}-\bar{\bm U}_Q\|_2^2\mid S\right]
&=
\bar{\bm a}^\top(\bm I-\bm M)\bm\Gamma_0(\bm I-\bm M)\bar{\bm a}
+
\tr(\bm M\bm\Gamma_0 \bm M\bm\Sigma_{\bm\zeta}).
\end{align}
\end{theorem}

\begin{corollary}[Gradient-flow filter limit]\label{cor:grad_flow_filter_limit}
With $\bm M=\bm M_t=\bm I-e^{-t\bm\Sigma^2/n}$, larger training singular values are fit first, and the $t\to\infty$ limit recovers the unique minimum-Euclidean-norm interpolant:
\begin{align}
\bm\delta_{\bm M_t} &= -\bm V\bm\Sigma^{-1}(\bm I-e^{-t\bm\Sigma^2/n})\bm U^\top(\bm U_S(0)-\bm y),
\\
\bm\delta_{\bm M_t} &\xrightarrow{t\to\infty} -\bm V\bm\Sigma^{-1}\bm U^\top(\bm U_S(0)-\bm y) = -\bm J_S(\bm w_0)^\dagger(\bm U_S(0)-\bm y).
\end{align}
\end{corollary}

\begin{proof}
From $\bm H_0 = \bm J_Q(\bm w_0)\bm V\bm\Sigma \bm U^\top$ and $\bm K_0^\dagger = \bm U\bm\Sigma^{-2}\bm U^\top$ we obtain
\begin{align}
\bm H_0\bm K_0^\dagger \bm M = \bm J_Q(\bm w_0)\bm V\bm\Sigma^{-1}\bm M\bm U^\top = \bm C_0\bm M\bm U^\top,
\end{align}
so $\bm U_Q^{\bm M} = \bm U_Q(0)+\bm C_0\bm M(\bar{\bm a}+\bm\zeta)$, and subtracting $\bar{\bm U}_Q = \bm U_Q(0)+\bm C_0\bar{\bm a}$ gives the first claim. Taking conditional expectations and using $\mathbb E[\bm\zeta\mid S]=0$,
\begin{align*}
\mathbb E\!\left[\|\bm U_Q^{\bm M}-\bar{\bm U}_Q\|_2^2\mid S\right]
&=
\|\bm C_0(\bm I-\bm M)\bar{\bm a}\|_2^2
+
\mathbb E\|\bm C_0\bm M\bm\zeta\|_2^2 \\
&=
 \bar{\bm a}^\top(\bm I-\bm M)\bm C_0^\top \bm C_0(\bm I-\bm M)\bar{\bm a}
+
\tr(\bm M\bm C_0^\top \bm C_0\bm M\bm\Sigma_{\bm\zeta}),
\end{align*}
which is the displayed formula because $\bm C_0^\top \bm C_0=\bm\Gamma_0$. The parameter formula follows from the standard frozen-kernel filter calculus, and the limit $\bm M_t\to \bm I_r$ yields the Moore--Penrose solution.
\end{proof}

\begin{corollary}[Predictive rank decomposition]\label{cor:predictive_rank_law}
Diagonalize $\bm\Gamma_0 = \sum_{j=1}^{\rho}\lambda_j\psi_j\psi_j^\top$ with $\lambda_1\ge\cdots\ge\lambda_\rho>0$, and let $\bm P_r=\sum_{j\le r}\psi_j\psi_j^\top$ be the top-$r$ predictive projector. The risk along the predictive-rank path and its increment are
\begin{align}
\mathcal R_r
&\triangleq
\mathbb E\!\left[\|\bm U_Q^{\bm P_r}-\bar{\bm U}_Q\|_2^2\mid S\right]
=
\sum_{j>r}\lambda_j|\langle \psi_j,\bar{\bm a}\rangle|^2
+
\sum_{j\le r}\lambda_j\langle \psi_j,\bm\Sigma_{\bm\zeta}\psi_j\rangle,
\\
\mathcal R_{r+1}-\mathcal R_r
&=
\lambda_{r+1}
\Bigl(
\langle \psi_{r+1},\bm\Sigma_{\bm\zeta}\psi_{r+1}\rangle
-
|\langle \psi_{r+1},\bar{\bm a}\rangle|^2
\Bigr).
\end{align}
\end{corollary}

\begin{corollary}[Benign overfitting]\label{cor:benign_overfitting}
Under the hypotheses of \autoref{cor:predictive_rank_law}, the interpolation variance is finite, and under isotropic noise $\bm\Sigma_{\bm\zeta}=\sigma_\xi^2\bm I_r$ it reduces to the Gramian trace:
\begin{align}
\mathcal R_\rho
&=
\sum_{j\le \rho}\lambda_j\langle \psi_j,\bm\Sigma_{\bm\zeta}\psi_j\rangle,
\\
\mathcal R_\rho
&=
\sigma_\xi^2\tr(\bm\Gamma_0)
=
\sigma_\xi^2\sum_{j=1}^r \frac{\|\bm J_Q(\bm w_0)v_j\|_2^2}{\sigma_j^2}.
\end{align}
\end{corollary}

\begin{remark}[Double descent]\label{rem:double_descent}
Along the predictive-rank path of \autoref{cor:predictive_rank_law}, adding mode $r+1$ lowers risk iff its signal exceeds its noise and raises risk iff the reverse inequality holds, so the U-shape is the sequence of sign changes of $\mathcal R_{r+1}-\mathcal R_r$.
\end{remark}

\begin{proof}
Apply \autoref{thm:unified_bias_variance} with $\bm M=\bm P_r$. Since $\bm P_r\bm\Gamma_0=\bm\Gamma_0\bm P_r$, the bias and variance terms split spectrally and the isotropic-noise variance reduces to a Gramian trace:
\begin{align}
(\bm I-\bm P_r)\bm\Gamma_0(\bm I-\bm P_r)
&=
\sum_{j>r}\lambda_j\psi_j\psi_j^\top,
\\
\bm P_r\bm\Gamma_0\bm P_r
&=
\sum_{j\le r}\lambda_j\psi_j\psi_j^\top,
\\
\tr(\bm\Gamma_0)
&=
\tr\!\bigl(\bm\Sigma^{-1}\bm V^\top \bm J_Q(\bm w_0)^\top \bm J_Q(\bm w_0)\bm V\bm\Sigma^{-1}\bigr)
=
\sum_{j=1}^r \frac{\|\bm J_Q(\bm w_0)v_j\|_2^2}{\sigma_j^2}.
\end{align}
Substituting the first two into the risk formula gives $\mathcal R_r$, and subtracting consecutive values gives the increment formula.
\end{proof}

\begin{figure}[t]
    \centering
    \safeincludegraphics[width=\textwidth]{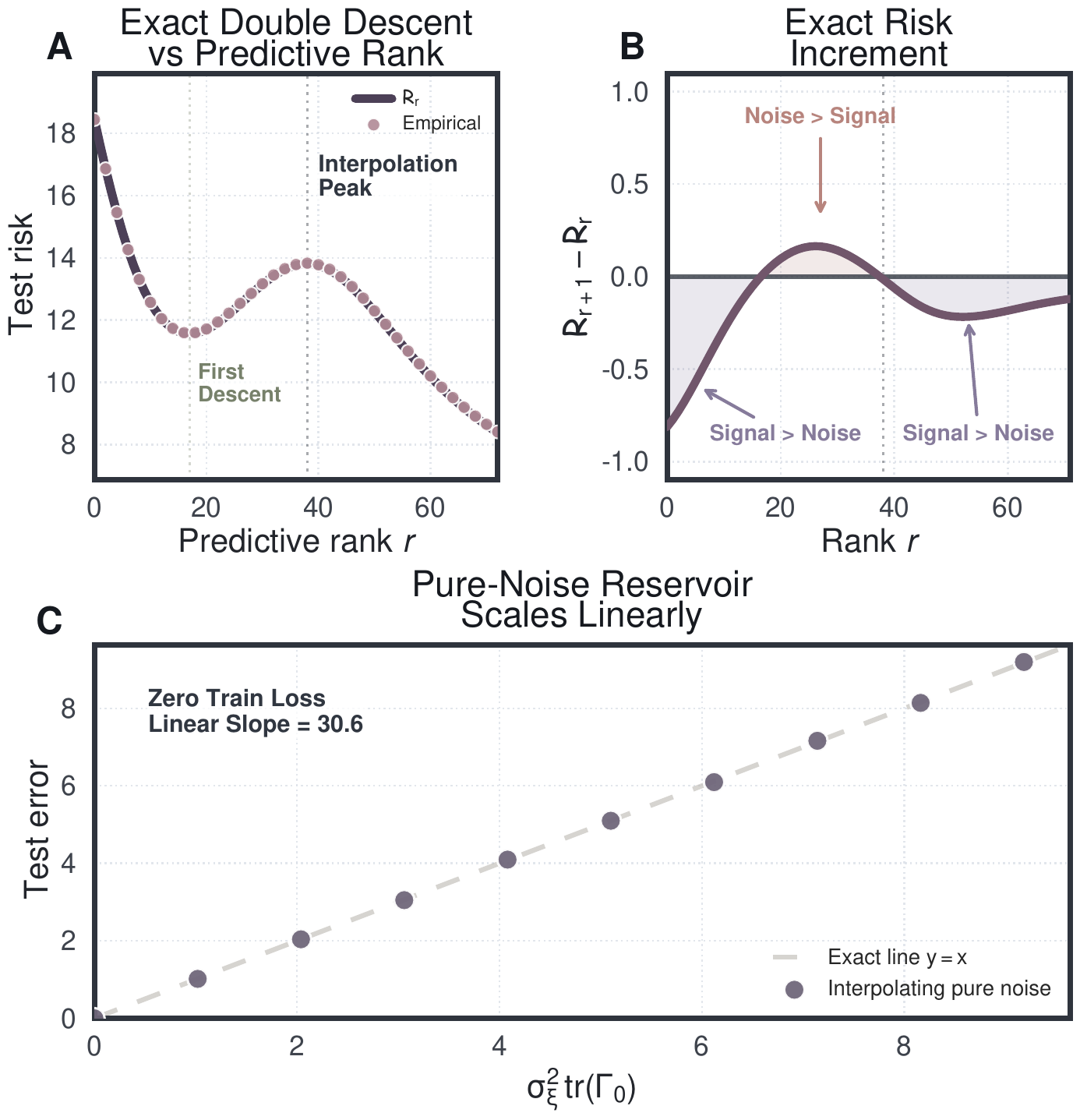}
    \caption{\textbf{Unified Bias--Variance: Capacity Axis.} Validation of \autoref{cor:predictive_rank_law}.
    \textbf{(a)} Empirical test risk (scatter) at $t \to \infty$ perfectly aligns with the theoretical risk $\mathcal{R}_r$ (solid line), explicitly predicting the double-descent peak without approximations.
    \textbf{(b)} The risk increment $\mathcal{R}_{r+1} - \mathcal{R}_r$. The peak of the double descent curve in (a) occurs exactly where this increment crosses below zero, proving that risk increases if and only if noise strictly dominates signal.
    \textbf{(c)} In the overparameterized limit, despite interpolating pure noise ($\widehat{\mathcal{L}}_S = 0$), test error scales strictly linearly with $\sigma_\xi^2 \tr(\bm\Gamma_0)$, proving the residual noise is physically trapped in the inert, test-invisible reservoir.}
    \label{fig:capacity_axis_unification}
\end{figure}

\begin{remark}[Implicit bias and grokking are the same decomposition in two regimes]
Implicit bias \citep{soudry2018implicit,gunasekar2017implicit} is the $\bm M=\bm M_t$ branch of \autoref{thm:unified_bias_variance}: gradient flow activates high-mobility modes first because $\bm M_t=\bm I-e^{-t\bm\Sigma^2/n}$ is monotone in $\sigma_j^2$, and the full interpolating limit is the minimum-norm solution. Grokking \citep{power2022grokking} is the nonstationary continuation of the same selection principle: in the frozen-kernel, no-manifold limit $(\gamma=0)$ the test transfer mass on the mobile singular space is exactly the spectrum of $\bm M_t\bm\Gamma_0\bm M_t$,
\begin{align}
\mathsf G_Q(t,0) &= n\bm H_0\bm K_0^\dagger(\bm I-e^{-t\bm K_0/n}),
\\
T_{Q,0}(0,t)
&=
n^2(\bm I-e^{-t\bm K_0/n})\bm K_0^\dagger \bm H_0^\top \bm H_0 \bm K_0^\dagger(\bm I-e^{-t\bm K_0/n}),
\\
\bm U^\top T_{Q,0}(0,t)\bm U &= n^2 \bm M_t\bm\Gamma_0 \bm M_t.
\end{align}
The optimal signal directions theorem (\autoref{thm:train_only_minimax_signal}) is the dynamic extension of the same spectral decomposition: delayed generalization occurs when the leading filtered predictive modes of $\bm M_t\bm\Gamma_0 \bm M_t$ (or, in the full theory, $C_R(s,T)$) finally dominate the tail.
\end{remark}

\begin{figure}[t]
    \centering
    \safeincludegraphics[width=\textwidth]{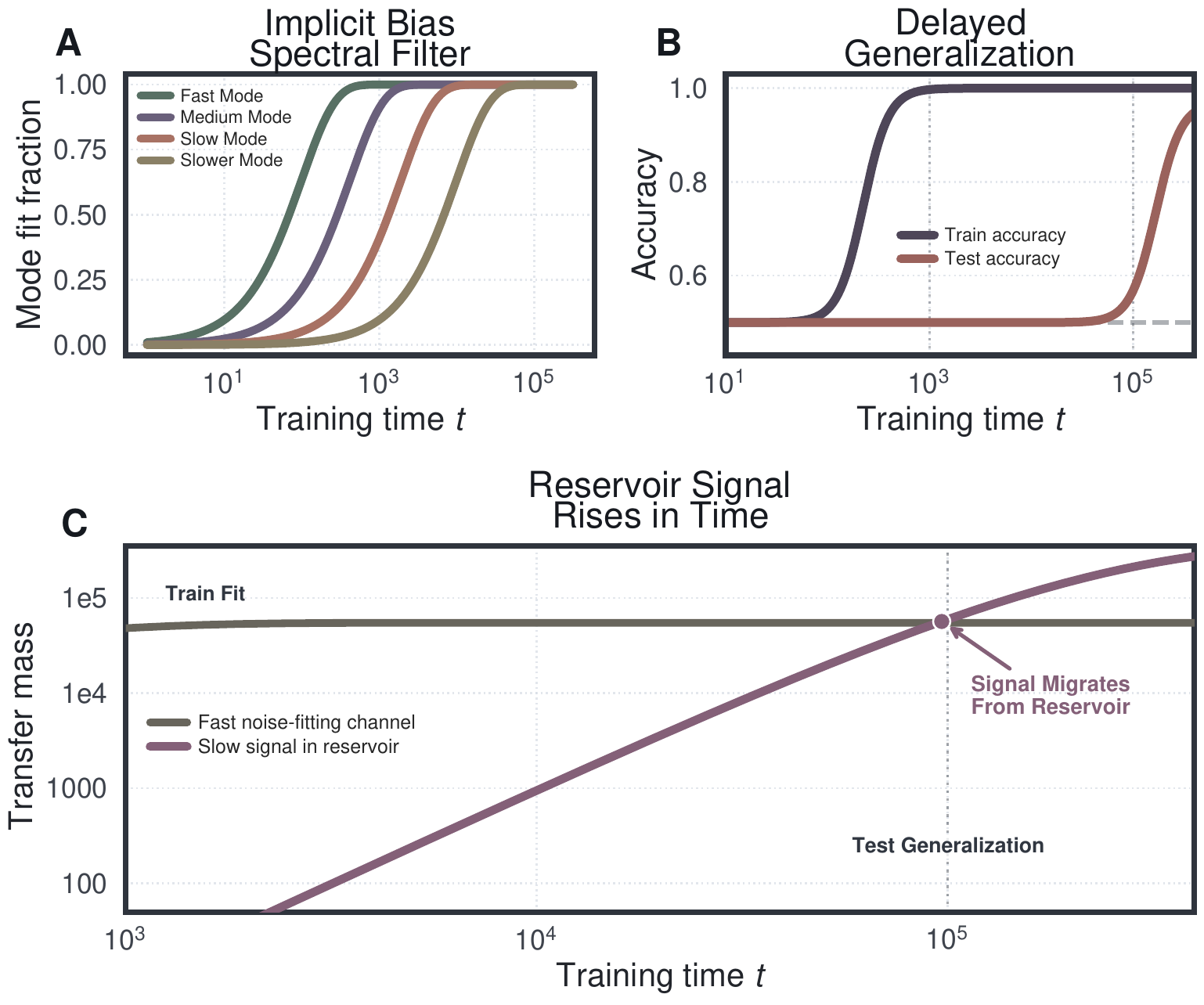}
    \caption{\textbf{Unified Bias--Variance: Time Axis.}
    \textbf{(A) Implicit Bias:} Target fit decomposed over the eigenvectors of $\bm\Gamma_0$. The theoretical filter $1 - e^{-t\sigma_j^2/n}$ derived in \autoref{thm:unified_bias_variance} shows high-mobility modes being learned exponentially faster than low-mobility modes.
    \textbf{(B) Grokking:} Standard delayed generalization. The network interpolates the training set at $t=10^3$, but test accuracy remains at random chance until $t=10^5$.
    \textbf{(C) The Mechanism:} Aggregate transfer mass concentrates in the high-mobility modes first (these fit noise) and saturates, while the slower low-mobility signal modes rise over time and overtake them at the same $t\approx 10^5$ scale where test accuracy rises. Grokking is the delayed resolution of the dynamic spectral filter, not a sudden change in optimization phase.}
    \label{fig:time_axis_unification}
\end{figure}

\begin{figure}[p]
    \centering
    \safeincludegraphics[height=0.68\textheight,keepaspectratio]{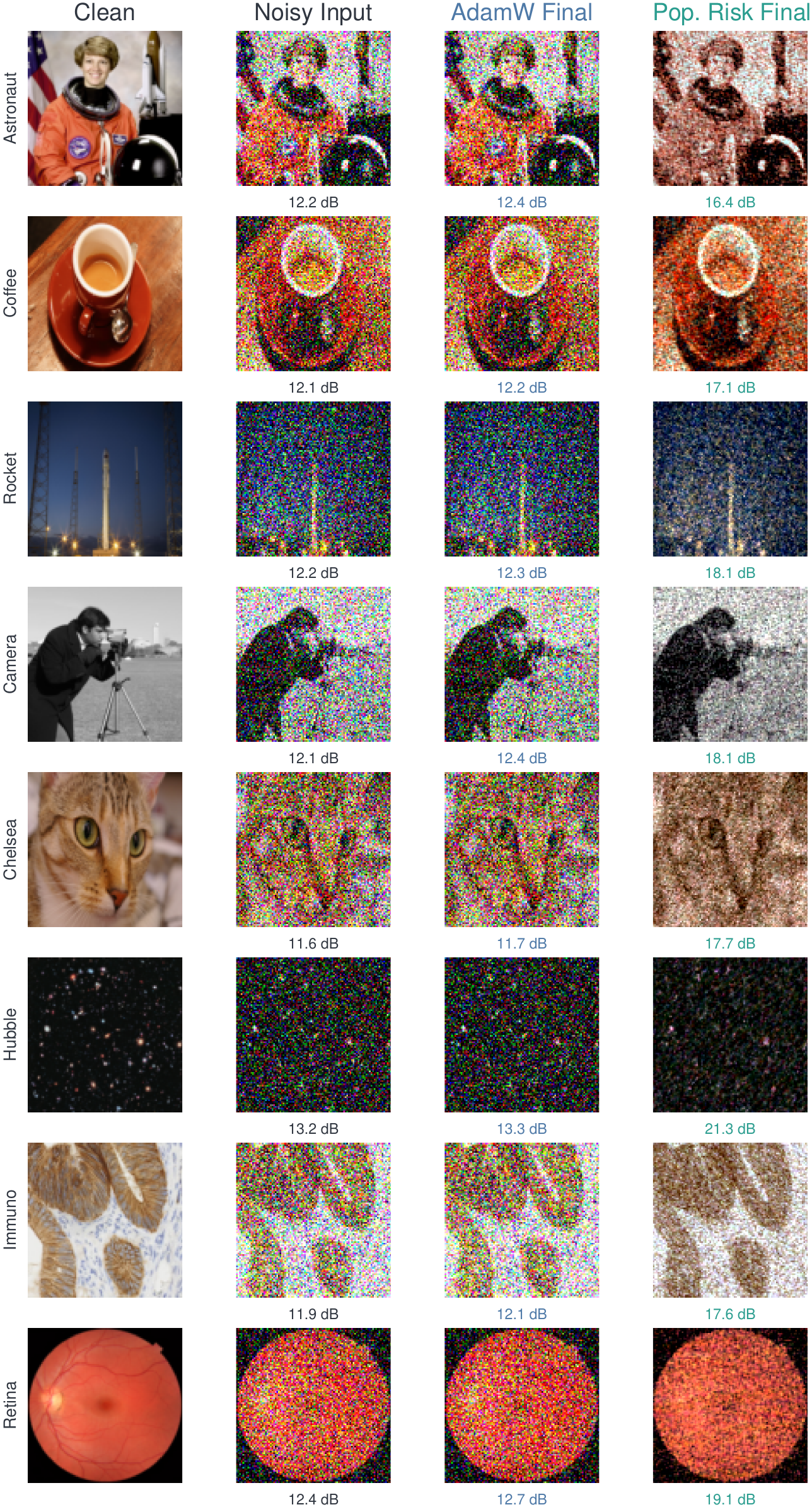}
    \caption{\textbf{Final-step INR reconstructions across images.}
    Each row trains the same coordinate-MLP denoising setup on a different noisy image, using the same optimizer settings and the same final training budget.
    The first two columns show the clean target and the corrupted input; the last two columns show the final AdamW and population-risk reconstructions.
    This gallery complements \autoref{fig:inr_denoising}: the line plots and Fourier spectra of the main text quantify the noise-fitting mechanism on one representative image, while the reconstructions show that the benefit is visually consistent across diverse image structure. Population-risk training reduces the need for image-by-image early stopping because the optimizer suppresses incoherent pixel-noise directions during training.}
    \label{fig:inr_denoising_gallery}
\end{figure}

\subsection{Linear Models: from Gradient Flow to Ridge Regression}\label{sec:linear_model}

\paragraph{Intuition.}
As a sanity check, we specialize the general theory to a linear model: every operator becomes a closed-form expression in the SVD of the data matrix, and weight-decayed gradient flow interpolates smoothly between two classical limits. Sending the weight decay $\lambda\to 0$ recovers minimum-norm interpolation; keeping $\lambda>0$ recovers kernel ridge regression. The same spectral filter calculus survives, with the saturation filter replaced by a shrinkage filter.

In the linear-model setting every object in the general theory is available in closed form. Weight-decayed gradient flow then recovers two classical estimators as limits of a single trajectory: minimum-norm interpolation at $\lambda=0$, kernel ridge regression at $\lambda>0$. We use the linear case to give explicit expressions for the propagator, the cumulative dissipation Gramian, the training displacement, and the transfer operator, and to show that the spectral filter calculus survives the addition of weight decay.

Let $F(\bm w,z)=\bm w^\top \bm x$ with $p=1$ and squared loss $\Phi_S(\bm u)=\tfrac{1}{2n}\|\bm u-\bm y\|_2^2$, and write the data matrix $\bm X=[\bm x_1\ \cdots\ \bm x_n]^\top\in\mathbb R^{n\times d}$ with compact SVD $\bm X=\bm U\bm\Sigma\bm V^\top$, $\bm\Sigma=\diag(\sigma_1,\dots,\sigma_r)$, $r=\rank(\bm X)$. The Jacobian $\bm J_S=\bm X$, the kernel $\Kss=\bm X\bm X^\top=\bm U\bm\Sigma^2\bm U^\top$, and the loss Hessian $\bm B=\tfrac1n\bm I$ are constant; write $\bm a\triangleq\bm U^\top(\bm y-\bm X\bm w_0)$ and $\bm P_{\mathrm{mob}}\triangleq\bm U\bm U^\top$.

\begin{proposition}[Explicit operators for the linear model]\label{prop:linear_operators}
The propagator, cumulative dissipation, training displacement, and test transfer operator on $[0,T]$ are
\begin{align}
  \mathcal P_g(t,0) &= e^{-t\Kss/n},
  &
  \mathcal W_S(0,T) &= \tfrac{n}{2}\bigl(\bm I-e^{-2T\Kss/n}\bigr)\bm P_{\mathrm{mob}},
  \\[3pt]
  \bm D &= n\bigl(\bm I-e^{-T\Kss/n}\bigr)\bm P_{\mathrm{mob}},
  &
  \bm G &= n\Kqs\Kss^{-1}\bigl(\bm I-e^{-T\Kss/n}\bigr)\bm P_{\mathrm{mob}},
\end{align}
where $\Kss^{-1}$ acts on $\range(\Kss)$. The optimal predictor is the frozen-kernel map $\bm A_\circ=\Kqs\Kss^\dagger$ and the irreducible remainder vanishes, $\bm R_\perp=\bm 0$, by \autoref{thm:train_test_coupling}.
\end{proposition}

\begin{proof}
Constant $\bm J_S$ reduces the propagator ODE $\partial_t\mathcal P_g=-\tfrac1n\Kss\mathcal P_g$ to a linear matrix equation with solution $\mathcal P_g(t,0)=e^{-t\Kss/n}$. Substituting into the definitions and integrating on each eigenspace of $\Kss$ via
\begin{align}
\int_0^T\sigma_j^2 e^{-2\tau\sigma_j^2/n} d\tau=\tfrac{n}{2}(1-e^{-2T\sigma_j^2/n}),
\end{align}
and likewise for $\bm D$ and $\bm G$, yields the displayed formulas. Since $1-e^{-T\sigma_j^2/n}>0$ for every $\sigma_j>0$, $\ker\bm D=\ker\Kss=\ker\mathcal W_S$, so $\bm R_\perp=\bm 0$. The optimal predictor $\bm A_\circ=\bm G\bm D^\dagger$ simplifies to $\Kqs\Kss^\dagger$ after canceling the common spectral factor $(\bm I-e^{-T\Kss/n})$.
\end{proof}

Under pure gradient flow from $\bm w_0$, the output-space filter is $\bm M_t=\bm I-e^{-t\bm\Sigma^2/n}$ and the output evolves as $\bm u(t)=\bm u(0)+\bm U\bm M_t\bm a$. As $t\to\infty$ the filter saturates ($\bm M_t\to\bm I_r$) and the parameters converge to $\bm w_0+\bm X^\dagger(\bm y-\bm X\bm w_0)$, the minimum-Euclidean-norm interpolant; this is the $\bm M=\bm M_t$ instance of the unified bias--variance decomposition (\autoref{thm:unified_bias_variance}).

The output-space framework of \autoref{sec:second_law} starts from the unpenalized flow $\partial_t\bm w=-\bm J_S^\top\bm g$; weight decay extends the dynamics by adding a $-\lambda\bm w$ term, and in the linear-model setting the extended dynamics remain explicitly solvable.

\begin{theorem}[Weight-decayed gradient flow converges to ridge regression]\label{thm:linear_ridge}
Under weight-decayed gradient flow $\partial_t\bm w=-\nabla L_S(\bm w)-\lambda\bm w$ with $\lambda>0$ from $\bm w_0=\bm 0$, the $j$-th right-singular component $\alpha_j(t)\triangleq\bm V_j^\top\bm w(t)$ obeys a scalar linear ODE whose output-space filter entry is the product of ridge shrinkage and an exponential approach:
\begin{align}\label{eq:lin_wd_ode}
  \partial_t\alpha_j
  &= -\Bigl(\frac{\sigma_j^2}{n}+\lambda\Bigr)\alpha_j
    +\frac{\sigma_ja_j}{n},
  \qquad
  a_j=(\bm U^\top\bm y)_j,
\\
\label{eq:lin_wd_filter}
  M_j^\lambda(t)
  &= \frac{\sigma_j^2}{\sigma_j^2+n\lambda}
    \Bigl(1-e^{-(\sigma_j^2+n\lambda)t/n}\Bigr),
\intertext{with fixed point $\alpha_j^*=\sigma_j a_j/(\sigma_j^2+n\lambda)$. As $t\to\infty$, $M_j^\lambda\to\sigma_j^2/(\sigma_j^2+n\lambda)$ and the parameters and outputs converge to the kernel ridge regression solution with regularization $n\lambda$:}
\label{eq:lin_wd_limit}
  \bm w(t) &\xrightarrow{t\to\infty}
    \bm X^\top(\bm X\bm X^\top+n\lambda\bm I)^{-1}\bm y,
\\
  \bm u(t) &\xrightarrow{t\to\infty}
    \Kss(\Kss+n\lambda\bm I)^{-1}\bm y.
\end{align}
Components of $\bm w$ orthogonal to every right singular vector satisfy $\partial_t \bm w_\perp=-\lambda \bm w_\perp\to \bm 0$, so weight decay drives every direction outside the data subspace to zero.
\end{theorem}

\begin{proof}
The per-example loss gradient gives the total gradient and, in the right-singular basis, the scalar linear ODE for $\alpha_j$ with rate $\gamma_j\triangleq\sigma_j^2/n+\lambda>0$:
\begin{align}
\nabla_{\bm w}\ell_j(\bm w) &= \frac{1}{n}\bm x_j(\bm x_j^\top\bm w-y_j),
\\
\nabla L_S(\bm w) &= \frac{1}{n}\bm X^\top(\bm X\bm w-\bm y),
\\
\partial_t\alpha_j
  &= -\frac{\sigma_j}{n}(\sigma_j\alpha_j-a_j)-\lambda\alpha_j
  = -\Bigl(\frac{\sigma_j^2}{n}+\lambda\Bigr)\alpha_j+\frac{\sigma_j a_j}{n}.
\end{align}
From $\alpha_j(0)=0$ we obtain $\alpha_j(t)$ and therefore the $j$-th output component:
\begin{align}
\alpha_j(t)
  &= \frac{\sigma_j a_j}{\sigma_j^2+n\lambda}(1-e^{-\gamma_j t}),
\\
(\bm U^\top\bm u)_j &= \sigma_j\alpha_j(t)=\frac{\sigma_j^2 a_j}{\sigma_j^2+n\lambda}(1-e^{-\gamma_j t}),
\end{align}
which is the filter entry \eqref{eq:lin_wd_filter}. Sending $t\to\infty$ and using the SVD identity
\begin{align}
\bm V\bm\Sigma(\bm\Sigma^2+n\lambda\bm I)^{-1}\bm U^\top
=
\bm X^\top(\bm X\bm X^\top+n\lambda\bm I)^{-1},
\end{align}
the parameter and output limits read
\begin{align}
  \bm w(\infty)
  &= \sum_{j=1}^r\frac{\sigma_j a_j}{\sigma_j^2+n\lambda}\bm V_j
  = \bm V\bm\Sigma(\bm\Sigma^2+n\lambda\bm I)^{-1}\bm U^\top\bm y
  = \bm X^\top(\bm X\bm X^\top+n\lambda\bm I)^{-1}\bm y,
\\
\bm u(\infty) &= \bm X\bm w(\infty)=\Kss(\Kss+n\lambda\bm I)^{-1}\bm y.
\end{align}
\end{proof}

\begin{proposition}[Bias--variance trade-off along the ridge path]\label{prop:linear_ridge_tradeoff}
Both the gradient-flow filter $\bm M_t=\bm I-e^{-t\bm\Sigma^2/n}$ and the ridge filter $\bm M_\lambda=\bm\Sigma^2(\bm\Sigma^2+n\lambda\bm I)^{-1}$ are self-adjoint contractions $0\preceq\bm M\preceq\bm I_r$, so the unified bias--variance decomposition (\autoref{thm:unified_bias_variance}) applies to each. Gradient flow ($\bm M=\bm M_t$) drives the bias to zero as $t\to\infty$ at the cost of full interpolation variance $\tr(\bm\Gamma_0\bm\Sigma_{\bm\zeta})$. Ridge regression ($\bm M=\bm M_\lambda$) trades nonzero bias for reduced variance, and the optimal $\lambda$ minimizes
\begin{align}\label{eq:lin_ridge_objective}
\bar{\bm a}^\top(\bm I-\bm M_\lambda)\bm\Gamma_0(\bm I-\bm M_\lambda)\bar{\bm a}
+\tr(\bm M_\lambda\bm\Gamma_0\bm M_\lambda\bm\Sigma_{\bm\zeta}).
\end{align}
\end{proposition}

\begin{remark}[Two limits of one dynamics]
The filters in \autoref{thm:unified_bias_variance} and \autoref{thm:linear_ridge} are the $\lambda=0$ and $\lambda>0$ limits of the same weight-decayed flow, whose output-space filter has $j$-th diagonal entry
\begin{align}\label{eq:lin_unified_filter}
  M_j^{\lambda}(t)
  = \frac{\sigma_j^2}{\sigma_j^2+n\lambda}
    \bigl(1-e^{-(\sigma_j^2+n\lambda)t/n}\bigr).
\end{align}
Setting $\lambda=0$ and $t\to\infty$ recovers minimum-norm interpolation; setting $\lambda>0$ and $t\to\infty$ recovers kernel ridge regression. Weight decay preserves the spectral filter calculus and replaces the saturation filter with a shrinkage filter.
\end{remark}

\section{Additional Experiments}\label{app:additional_experiments}

This section reports three additional experiments where empirical-risk training has a documented failure mode (chaotic-dynamics rollout from noisy observations, INR denoising, noisy-preference DPO) and the population-risk update prevents it.

\paragraph{Chaotic dynamics from noisy state observations.}
A neural one-step predictor for Lorenz '63 is trained from sensor-noisy state observations and evaluated against clean held-out dynamics over $3$ seeds. The smooth vector field has a coherent gradient mean across minibatches; point-specific sensor noise produces zero-mean, high-variance fluctuations and is suppressed by the population-risk update.

\begin{figure}[h]
    \centering
    \safeincludegraphics[width=\textwidth]{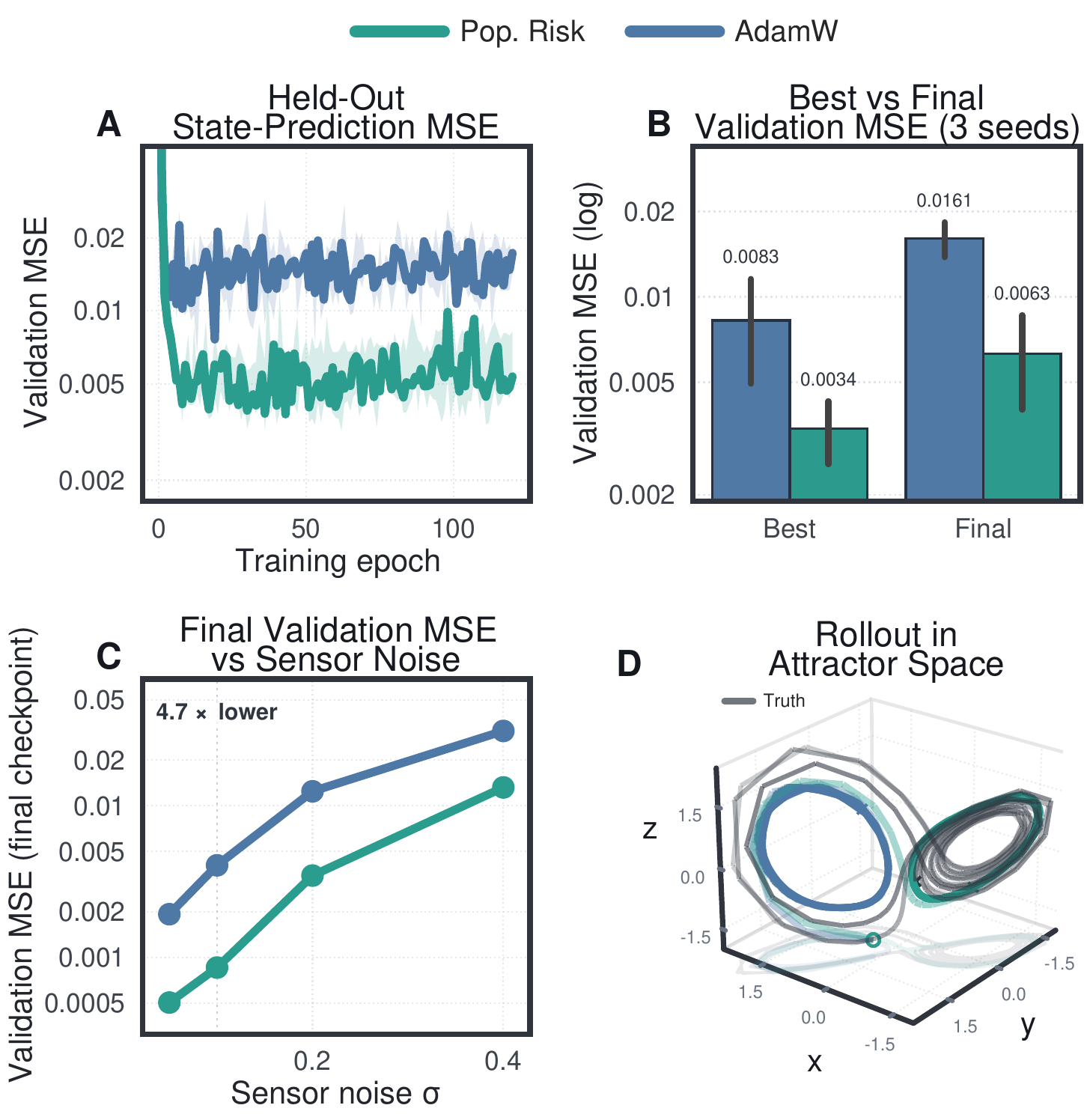}
    \caption{\textbf{Population-risk training on chaotic dynamics.}
    \textbf{(A)}~Held-out state-prediction MSE: AdamW initially improves, then fits sensor noise and its validation error rises, while population-risk training maintains a lower plateau.
    \textbf{(B)}~Best versus final validation MSE: population-risk training finishes below AdamW's best checkpoint.
    \textbf{(C)}~Final validation MSE across sensor-noise levels.
    \textbf{(D)}~Rollouts in attractor space: AdamW drifts away from the true Lorenz manifold, while population-risk training tracks the ground-truth attractor.}
    \label{fig:noisy_dynamics}
\end{figure}

\paragraph{Physics-informed networks with noisy initial conditions.}
A PINN solves the linear advection equation $u_t+\beta u_x=0$ with periodic boundary at $\beta=5$, trained from initial-condition observations corrupted by Gaussian sensor noise and evaluated against the clean analytical solution. Empirical-risk fitting drives the network to interpolate the noisy IC and damages the physical prediction; population-risk training suppresses the noise-fitting channel and reaches the target test error substantially faster than any learning-rate-tuned AdamW baseline. The main-text \autoref{fig:pinn_convection} reports relative $\ell_2$ trajectories, iterations-to-target, and pointwise error fields.

\begin{table}[h]
\centering
\caption{PINN noisy-IC convection benchmark, $\beta=5$, $\sigma_{\mathrm{IC}}=1$, $\approx$1.8M parameters, 3 seeds. $\ell_2$ is relative error against the clean analytical solution on a $101\times101$ grid. Population-risk training reaches $\ell_2\le 0.40$ in $2.4\times$ fewer iterations than the best LR-tuned AdamW.}
\label{tab:pinn_ablation}
\small
\renewcommand{\arraystretch}{1.15}
\rowcolors{2}{white}{gray!10}
\begin{tabular}{lccc}
\toprule
Method & Best $\ell_2$ rel. & Iters.\ to $\ell_2 \leq 0.4$ & Speedup \\
\midrule
  Pop. Risk Training (full) & $0.249$ & $1{,}400$ & $2.36\times$ \\
  Pop. Risk Training (no warmup) & $0.298$ & $2{,}500$ & $1.32\times$ \\
  AdamW (lr=1e-3) & $0.405$ & $>$8k (1/3) & -- \\
  AdamW (lr=1e-4) & $0.309$ & $3{,}300$ & $1.00\times$ \\
  AdamW (lr=5e-5) & $0.336$ & $6{,}100$ & $0.54\times$ \\
  AdamW (lr=1.5e-5) & $0.726$ & $>$8k (never) & -- \\
\bottomrule
\end{tabular}
\end{table}

\paragraph{Implicit neural representation denoising.}
A coordinate MLP is trained on noisy RGB pixel observations and evaluated against the clean image. The smooth image structure is the coherent signal across coordinate minibatches; pixel-level sensor noise is incoherent.

\begin{figure}[h]
    \centering
    \safeincludegraphics[width=0.92\textwidth]{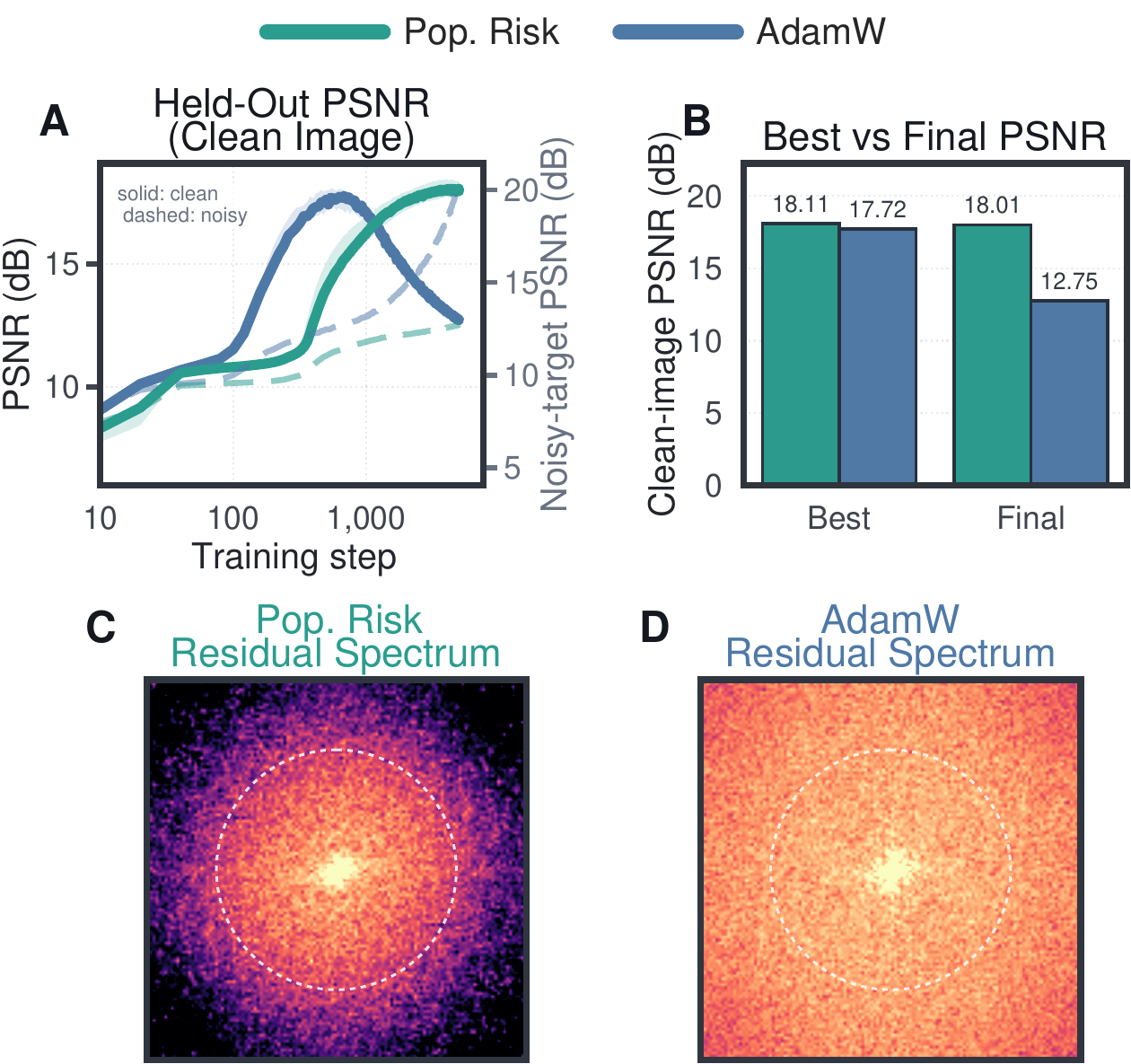}
    \caption{\textbf{Population-risk training removes early stopping in INR denoising.}
    \textbf{(A)}~Held-out clean PSNR: AdamW reaches a transient peak and then degrades, while population-risk training keeps improving the clean image without checkpoint selection.
    \textbf{(B)}~Best versus final clean PSNR.
    \textbf{(C,D)}~Final residual Fourier spectra; outside the dashed high-frequency ring, population-risk training has $8.5\times$ lower residual power.}
    \label{fig:inr_denoising}
\end{figure}

\paragraph{Noisy preference alignment.}
We fine-tune Qwen2.5-0.5B-Instruct with DPO on UltraFeedback preferences where $30\%$ of training pairs have chosen and rejected responses swapped (annotator disagreement). The clean held-out eval set has no label noise; all hyperparameters are shared between AdamW and population-risk training; results are over 3 seeds. The main-text \autoref{fig:noisy_dpo} reports sustained accuracy, reward drift from the reference policy, and the accuracy--drift phase plot.

\begin{table}[h]
\centering
\caption{Noisy-DPO preference alignment benchmark. Qwen2.5-0.5B-Instruct, QLoRA $r\!=\!16$, UltraFeedback, $30\%$ preference noise, 3 seeds. Population-risk training wins on every metric.}
\label{tab:noisy_dpo}
\small
\renewcommand{\arraystretch}{1.15}
\rowcolors{2}{white}{gray!10}
\begin{tabular}{lccr}
\toprule
Metric & AdamW & Pop.\ Risk & Ratio \\
\midrule
Final reward accuracy ($\uparrow$) & $0.566$ & $\mathbf{0.641}$ & $1.13\times$ \\
Mean trajectory accuracy ($\uparrow$) & $0.549$ & $\mathbf{0.625}$ & $1.14\times$ \\
Worst-step accuracy ($\uparrow$) & $0.510$ & $\mathbf{0.589}$ & $1.16\times$ \\
Mean absolute reward drift ($\downarrow$) & $0.41$ & $\mathbf{0.14}$ & $3.05\times$ \\
Steps to sustained $T\!\geq\!0.54$ ($\downarrow$) & $225$ & $\mathbf{75}$ & $3.00\times$ \\
Steps to sustained $T\!\geq\!0.55$ ($\downarrow$) & $225$ & $\mathbf{75}$ & $3.00\times$ \\
Steps to sustained $T\!\geq\!0.56$ ($\downarrow$) & $338$ & $\mathbf{100}$ & $3.38\times$ \\
Steps to sustained $T\!\geq\!0.58$ ($\downarrow$) & N/R & $\mathbf{125}$ & N/R \\
Steps to sustained $T\!\geq\!0.60$ ($\downarrow$) & N/R & $\mathbf{175}$ & N/R \\
\bottomrule
\end{tabular}
\end{table}

\end{document}